\documentclass{article} 
\usepackage{nips15submit_e,times}
\usepackage{hyperref}
\usepackage{url}

\usepackage{amsthm,amsmath}
\usepackage{amssymb}
\usepackage{color}
\usepackage{algorithm2e}
\usepackage{graphicx}
\usepackage{multirow}
\usepackage{caption}
\usepackage{epsf}
\usepackage{wrapfig}
\usepackage{bm}
\usepackage{bbm}

\newtheorem{theorem}{Theorem}[section]
\newtheorem{lemma}{Lemma}[section]
\newtheorem{corollary}{Corollary}[section]

\newcommand{\cbar}[1]{\tilde{#1}}

\usepackage{functan}

\title{Deep learning with Elastic Averaging SGD}

\author{
Sixin Zhang \\
Courant Institute, NYU \\
\texttt{zsx@cims.nyu.edu} \\
\And
Anna Choromanska \\
Courant Institute, NYU \\
\texttt{achoroma@cims.nyu.edu} \\
\And
Yann LeCun \\
Center for Data Science, NYU \&
Facebook AI Research \\
\texttt{yann@cims.nyu.edu } 
}

\newcommand{\redit}{\textcolor{black}}
\newcommand{\blueit}{\textcolor{black}}

\newenvironment{packed_item}{
\begin{itemize}
  \setlength{\itemsep}{1pt}
  \setlength{\parskip}{0pt}
  \setlength{\parsep}{0pt}
}{\end{itemize}}

\nipsfinalcopy 

\begin{document}

\Macro{L2}{}

\maketitle

\begin{abstract}
We study the problem of stochastic optimization for deep learning in the parallel computing environment under communication constraints. A new algorithm is proposed in this setting where the communication and coordination of work among concurrent processes (local workers), is based on an elastic force which links the parameters they compute with a center variable stored by the parameter server (master). The algorithm enables the local workers to perform more exploration, i.e. the algorithm allows the local variables to fluctuate further from the center variable by reducing the amount of communication between local workers and the master. We empirically demonstrate that in the deep learning setting, due to the existence of many local optima, allowing more exploration can lead to the improved performance. We propose synchronous and asynchronous variants of the new algorithm. We provide the stability analysis of the asynchronous variant in the round-robin scheme and compare it with the more common parallelized method \textit{ADMM}. We show that the stability of \textit{EASGD} is guaranteed when a simple stability condition is satisfied, which is not the case for \textit{ADMM}. We additionally propose the momentum-based version of our algorithm that can be applied in both synchronous and asynchronous settings. Asynchronous variant of the algorithm is applied to train convolutional neural networks for image classification on the \textit{CIFAR} and \textit{ImageNet} datasets. Experiments demonstrate that the new algorithm accelerates the training of deep architectures compared to \textit{DOWNPOUR} and other common baseline approaches and furthermore is very communication efficient.
\end{abstract}

\section{Introduction}
\label{sec:Introduction}

One of the most challenging problems in large-scale machine learning is how to parallelize the training of large models that use a form of stochastic gradient descent (\textit{SGD})~\cite{bottou-98x}. There have been attempts to parallelize \textit{SGD}-based training for
large-scale deep learning models on large number of CPUs, including
the Google's Distbelief system~\cite{2012distbelief}. But practical
image recognition systems consist of large-scale convolutional neural networks
trained on few GPU cards sitting in a
single computer~\cite{krizhevsky2012imagenet,2013arXiv1312.6229S}. The main challenge is to devise parallel \textit{SGD} algorithms to train
large-scale deep learning models that yield a significant speedup when run on multiple GPU cards. 

In this paper we introduce the \textit{Elastic Averaging SGD} method (\textit{EASGD}) and its variants. \textit{EASGD} is motivated by quadratic penalty
method~\cite{nocedal2006numerical}, but is re-interpreted as a parallelized extension of the averaging \textit{SGD}
algorithm~\cite{polyak1992acceleration}. The
basic idea is to let each worker maintain its own local 
parameter, and the communication and coordination of work among the local workers is based on an elastic force which links the parameters they compute with a center variable stored by the master. The center variable is updated as a moving average where the average is taken in time and also in space over the parameters computed by local workers. The main contribution of this paper is a new algorithm that provides fast convergent minimization while outperforming \textit{DOWNPOUR} method~\cite{2012distbelief} and other baseline approaches in practice. Simultaneously it reduces the communication overhead between the master and the local workers while at the same time it maintains high-quality performance measured by the test error. The new algorithm applies to deep learning settings such as parallelized training of convolutional neural networks.

The article is organized as follows. Section~\ref{sec:problemset} explains the problem setting, Section~\ref{sec:syncEASGD} presents the synchronous \textit{EASGD} algorithm and its asynchronous and momentum-based variants, Section~\ref{sec:theory} provides stability analysis of \textit{EASGD} and \textit{ADMM} in the round-robin scheme, Section~\ref{sec:Experiments} shows experimental results and Section~\ref{sec:conclusion} concludes. The Supplement contains additional material including additional theoretical analysis. 

\section{Problem setting}
\label{sec:problemset}

Consider minimizing a function $F(x)$ in a parallel computing environment~\cite{Bertsekas1989} with $p \in \mathbb{N}$ workers and a master. In this paper we focus on the stochastic optimization problem of the following form 
\begin{equation}
\vspace{-0.02in}
\min_{x} F(x):=\mathbb{E} [f(x,\xi)],
\label{eq:original}
\end{equation}
where $x$ is the model parameter to be estimated and $\xi$ is a random variable that follows the probability distribution $\mathbb{P}$ over $\Omega$ such that $F(x) = \int_{\Omega}f(x,\xi) \mathbb{P} (d\xi)$. The optimization problem in \blueit{Equation}~\ref{eq:original} can be reformulated as follows
\begin{equation}
\vspace{-0.02in}
\label{eq:penalty}
\min_{x^1,\ldots,x^p,\tilde{x}} \sum_{i=1}^p \mathbb{E} [f(x^i,\xi^i)] + \frac{\rho}{2}\|x^i - \tilde{x}\|^2,
\end{equation}
where each $\xi^i$ follows \blueit{the same distribution} $\mathbb{P}$ (thus we assume each worker can sample the entire dataset). In the paper we refer to $x^i$'s as local variables and we refer to $\tilde{x}$ as a center variable. The problem of the equivalence of these two objectives is studied in the literature and is known as \blueit{the \textit{augmentability} or the \textit{global variable consensus} problem~\cite{hestenes1975optimization,Boyd2011}}. The quadratic penalty term $\rho$ in Equation~\ref{eq:penalty} is expected to ensure that local workers will not fall into different attractors that are far away from the center variable. 
This paper focuses on the problem of reducing the parameter communication overhead between the master and local workers~\cite{NIPS2014_5386,2012distbelief,Yadan2013,Paine2013,onebitparallel}. The problem of data communication when the data is distributed among the workers~\cite{Bertsekas1989,Bekkerman:2011:SUM:2107736.2107740} is a more general problem and is not addressed in this work. We however emphasize that our problem setting is still highly \blueit{non-trivial} under the communication constraints due to the existence of many local optima~\cite{ChoromanskaHMAL15}. 

\section{EASGD update rule}
\label{sec:syncEASGD}

The \textit{EASGD} updates captured in resp. Equation~\ref{eq:update1} and~\ref{eq:update2} are obtained by taking the gradient descent step on the objective in Equation \ref{eq:penalty} with respect to resp. variable $x^i$ and $\tilde{x}$, 
\begin{eqnarray}
x^i_{t+1} &=& x^i_t - \eta (g^i_t(x_t^i) + \rho (x^i_t - \tilde{x}_t) ) \label{eq:update1}\\
\tilde{x}_{t+1} &=& \tilde{x}_t + \eta \sum_{i=1}^p \rho (x^i_t - \tilde{x}_t), \label{eq:update2}
\end{eqnarray}
where $g_t^i(x_t^i)$ denotes the stochastic gradient of $F$ with respect to $x^i$ evaluated at iteration $t$, $x^i_t$ and $\tilde{x}_t$ denote respectively the value of variables $x^{i}$ and $\tilde{x}$ at iteration $t$, and $\eta$ is the learning rate. 

The update rule for the center variable $\tilde{x}$ takes the form of moving average where the average is taken over both space and time. Denote $\alpha = \eta\rho$ and $\beta=p\alpha$, then  Equation~\ref{eq:update1} and~\ref{eq:update2} become
\begin{eqnarray}
x^i_{t+1} &=& x^i_t - \eta g^i_t(x_t^i) - \alpha (x^i_t - \tilde{x}_t) \label{eq:movingaverage1}\\
\tilde{x}_{t+1} &=& (1 - \beta) \tilde{x}_t + \beta \left(\frac{1}{p} \sum_{i=1}^p x^i_t\right).
\label{eq:movingaverage2}
\end{eqnarray}


Note that choosing $\beta=p \alpha$ leads to an elastic symmetry in the update rule, i.e. there exists an symmetric force equal to $\alpha(x_t^i - \tilde{x}_t)$ between the update of each $x^i$ and $\tilde{x}$.
It has a crucial influence on the algorithm's stability as will be explained in Section~\ref{sec:theory}. 
Also in order to minimize the staleness~\cite{NIPS2013_4894} of the difference $x_t^i- \tilde{x}_t$ between the center and the local variable, the update for the master in Equation~\ref{eq:update2} involves $x_t^i$ instead of $x_{t+1}^i$.

Note also that $\alpha = \eta \rho$, where the magnitude of $\rho$  represents the amount of exploration we allow in the model. In particular, small $\rho$ allows for more exploration as it allows $x^i$'s to fluctuate further from the center $\tilde{x}$. 
The distinctive idea of \textit{EASGD} is to allow the local workers to perform more exploration (small $\rho$) and the master to perform exploitation. This approach differs from other settings explored in the literature~\cite{2012distbelief,DBLP:conf/icml/AzadiS14,doi:10.1137/S0363012995282784,Nedic2001381,langford2009slow,NIPS2011_0574,DBLPRechtRWN11,zinkevich2010parallelized}, and focus on how fast the center variable converges. In this paper we show the merits of our approach in the deep learning setting.


\subsection{Asynchronous EASGD}
\label{sec:asyncEASGD}

We discussed the synchronous update of \textit{EASGD} algorithm in the previous section. 
In this section we propose its asynchronous variant. The local workers are still responsible for updating the local variables $x^i$'s, whereas the master is updating the center variable $\tilde{x}$. Each worker maintains its own clock $t^i$, which starts from $0$ and is incremented by $1$ after each stochastic gradient update of $x^{i}$ as shown in Algorithm~\ref{alg:async}. The master performs an update whenever the local workers finished $\tau$ steps of their gradient updates, where we refer to $\tau$ as the \textit{communication period}. As can be seen in Algorithm~\ref{alg:async}, whenever $\tau$ divides the local clock of the $i^{\text{th}}$ worker, the $i^{\text{th}}$ worker communicates with the master and requests the current value of the center variable $\tilde{x}$. The worker then waits until the master sends back the requested parameter value, and computes the elastic difference $\alpha (x - \tilde{x})$ (this entire procedure is captured in step a) in Algorithm~\ref{alg:async}). The elastic difference is then sent back to the master (step b) in Algorithm~\ref{alg:async}) who then updates $\tilde{x}$. 

The communication period $\tau$ controls the frequency of the communication between every local worker and the master, and thus the trade-off between exploration and exploitation.

\begin{minipage}{6.9cm}
\vspace{-0.3in}
\RestyleAlgo{boxruled}
\begin{algorithm}[H]
\caption{Asynchronous EASGD: \newline Processing by worker $i$ and the master}
\label{alg:async}
\begin{tabular}{l}
\textbf{Input:} learning rate $\eta$, moving rate $\alpha$,\\ $\:\:\:\:\:\:\:\:\:\:\:\:\:\:$communication period $\tau \in \mathbb{N}$\\
 \textbf{Initialize:} $\tilde{x}$ is initialized randomly, $x^i = \tilde{x}$,$\:\:\:$ \\$\:\:\:\:\:\:\:\:\:\:\:\:\:\:\:\:\:\:\:\:t^i = 0$\\
\hline
\textbf{Repeat}\\
\:\:\:\:\:$x \leftarrow x^i$\\
\:\:\:\:\:\textbf{if} ($\tau $ divides $t^i $) \textbf{then}\\
\:\:\:\:\:\:\:\:\:\:\:\textbf{a)}\:$x^i \leftarrow x^i - \alpha (x - \tilde{x}) $\\
\:\:\:\:\:\:\:\:\:\:\:\textbf{b)}\:$\tilde{x} \:\:\:\!\!\leftarrow \tilde{x} \:\:\!+ \alpha (x - \tilde{x}) $\\
\:\:\:\:\:\textbf{end}\\
\:\:\:\:\:$ x^i \leftarrow x^i - \eta g^i_{t^i}(x)$\\
\:\:\:\:\:$t^i\:\:\:\:\!\!\!\!\leftarrow \:\!t^i \:\:\!\!+ 1$\\
\textbf{Until forever}
\end{tabular}
\end{algorithm}
\end{minipage}
\hspace{0.05in}
\begin{minipage}{6.9cm}
\RestyleAlgo{boxruled}
\begin{algorithm}[H]
\caption{Asynchronous EAMSGD: \newline Processing by worker $i$ and the master}
\label{alg:asyncM}
\begin{tabular}{l}
\textbf{Input:} learning rate $\eta$, moving rate $\alpha$,\\  $\:\:\:\:\:\:\:\:\:\:\:\:\:\:$communication period $\tau \in \mathbb{N}$,\\ $\:\:\:\:\:\:\:\:\:\:\:\:\:\:$momentum term $\delta$\\
 \textbf{Initialize:} $\tilde{x}$ is initialized randomly, $x^i = \tilde{x}$,$\:\:\:$\\ $\:\:\:\:\:\:\:\:\:\:\:\:\:\:\:\:\:\:\:\:v^i = 0$, $t^i = 0$\\
\hline
\textbf{Repeat}\\
\:\:\:\:\:$x \leftarrow x^i$\\
\:\:\:\:\:\textbf{if} ($\tau $ divides $t^i $) \textbf{then}\\
\:\:\:\:\:\:\:\:\:\:\:\textbf{a)}\:$x^i \leftarrow x^i - \alpha (x - \tilde{x})$\\
\:\:\:\:\:\:\:\:\:\:\:\textbf{b)}\:$\tilde{x} \:\:\:\!\!\leftarrow \tilde{x} \:\:\!+ \alpha (x - \tilde{x}) $\\
\:\:\:\:\:\textbf{end}\\
\:\:\:\:\:$ v^i \leftarrow \delta v^i - \eta g^i_{t^i}(x + \delta v^i)$\\
\:\:\:\:\:$ x^i \leftarrow x^i + v^i$\\
\:\:\:\:\:$t^i\:\:\:\:\!\!\!\!\leftarrow \:\!t^i \:\:\!\!+ 1$\\
\textbf{Until forever}
\end{tabular}
\end{algorithm}
\end{minipage}

\vspace{-0.2in}
\subsection{Momentum EASGD}
\label{sec:momentumEASGD}

The momentum EASGD (\textit{EAMSGD}) is a variant of our Algorithm~\ref{alg:async} and is captured in Algorithm~\ref{alg:asyncM}. It is based on the Nesterov's momentum scheme~\cite{Nesterov:2005:SMN:1058099.1058103, lan2012optimal, icml2013_sutskever13}, where the update of the local worker of the form captured in Equation~\ref{eq:update1} is replaced by the following update
\begin{eqnarray}
v_{t+1}^i &=& \delta v_t^i - \eta g^i_t(x_t^i + \delta v_t^i )\\
x^i_{t+1} &=& x^i_t + v_{t+1}^i  - \eta\rho (x^i_t - \tilde{x}_t) \nonumber,
\end{eqnarray}
where $\delta$ is the momentum term. Note that when $\delta = 0$ we recover the original \textit{EASGD} algorithm.

As we are interested in reducing the communication overhead in the parallel computing environment where the parameter vector is very large, we will be exploring in the experimental section the asynchronous \textit{EASGD} algorithm and its momentum-based variant in the relatively large $\tau$ regime (less frequent communication). 

\section{Stability analysis of EASGD and ADMM in the round-robin scheme}
\label{sec:theory}

In this section we study the stability of the asynchronous \textit{EASGD} and \textit{ADMM} methods in the round-robin scheme~\cite{langford2009slow}.
We first state the updates of both algorithms in this setting, and then we study their stability. We will show that in the one-dimensional quadratic case, \textit{ADMM} algorithm can exhibit chaotic behavior,
leading to exponential divergence. The analytic condition for the \textit{ADMM} algorithm to be stable is still unknown, while for the \textit{EASGD} algorithm it is very simple\footnote{This condition resembles the stability condition for the synchronous \textit{EASGD} algorithm (Condition~\ref{eq:condition} for $p=1$) in the analysis in the Supplement.}.

The analysis of the synchronous \textit{EASGD} algorithm, including its convergence rate, and its averaging property, in the quadratic and strongly convex case, is deferred to the Supplement.

In our setting, the \textit{ADMM} method~\cite{Boyd2011,icml2014c2_zhange14,ouyang2013stochastic} involves solving the following minimax problem\footnote{The convergence analysis in~\cite{icml2014c2_zhange14} is based on the assumption that ``At any master iteration, updates from the workers have the same probability of arriving at the master.", which
is not satisfied in the round-robin scheme.},
\begin{equation}
\max_{\lambda^1,\ldots,\lambda^p} \min_{x^1,\ldots,x^p,\tilde{x}} 
\sum_{i=1}^p F(x^i) - \lambda^i (x^i-\tilde{x}) + \frac{\rho}{2}\|x^i - \tilde{x}\|^2,
\end{equation}
where $\lambda^{i}$'s are the Lagrangian multipliers. The resulting updates of the \textit{ADMM} algorithm in the round-robin scheme are given next. Let $t\ge0$ be a global clock. At each $t$, we linearize the function $F(x^i) $ with $ F(x^i_t)+\scalprod{L2}{\nabla{F}(x^i_t)}{x^i-x^i_t}+\frac{1}{2\eta} \norm{L2}{x^i-x^i_t}^2$ as in~\cite{ouyang2013stochastic}. The updates become 
\begin{eqnarray}
\lambda^{i}_{t+1} &=& 
\left\{ \begin{array}{ll}
\lambda^{i}_t -  (x^{i}_t- \tilde{x}_t) & \mbox{if $\mod(t,p) = i-1$};\\
        \lambda^i_t  & \mbox{if $\mod(t,p) \ne i-1$}.\end{array} \right.  \label{eq:admm_update1} \\
x^{i}_{t+1} &=& 
\left\{ \begin{array}{ll}
\frac{x^i_t - \eta \nabla{F}(x^{i}_t)+ \eta \rho (\lambda^i_{t+1} + \tilde{x}_t)}{1+\eta \rho}
& \mbox{if $\mod(t,p) = i-1$};\\
x^i_t  & \mbox{if $\mod(t,p) \ne i-1$}.\end{array} \right. \label{eq:admm_update2}\\
\tilde{x}_{t+1} &=& \frac{1}{p} \sum_{i=1}^p ( x^i_{t+1} - \lambda^i_{t+1}).  \label{eq:admm_update3}
\end{eqnarray}
Each local variable $x^i$ is periodically updated (with period $p$). First, the Lagrangian multiplier $\lambda^{i}$ is updated with the dual ascent update
as in Equation \ref{eq:admm_update1}. It is followed by the gradient descent update of the local variable 
as given in Equation \ref{eq:admm_update2}. Then the center variable $\tilde{x}$ is updated 
with the most recent values of all the local variables and Lagrangian multipliers as in Equation \ref{eq:admm_update3}.
Note that since the step size for the dual ascent update is chosen to be $\rho$ by convention~\cite{Boyd2011,icml2014c2_zhange14,ouyang2013stochastic}, we have 
re-parametrized the Lagrangian multiplier to be $\lambda^{i}_{t} \leftarrow \lambda^{i}_{t}/\rho$ in
the above updates.

The \textit{EASGD} algorithm in the round-robin scheme is defined similarly and is given below
\begin{eqnarray}
x^{i}_{t+1} &=& 
\left\{ \begin{array}{ll}
x^i_t - \eta \nabla{F}(x^{i}_t) - \alpha (x_t^i - \tilde{x}_t)
& \mbox{if $\mod(t,p) = i-1$};\\
x^i_t  & \mbox{if $\mod(t,p) \ne i-1$}.\end{array} \right. \label{eq:easgd_update1}\\
\tilde{x}_{t+1} &=&  \tilde{x}_{t} + \sum_{i: \mod(t,p) = i-1} \alpha ( x^i_{t}  -   \tilde{x}_{t}).  \label{eq:easgd_update2}
\end{eqnarray}
At time $t$, only the $i$-th local worker (whose index $i-1$ equals $t$ modulo $p$) is activated,
and performs the update in Equations~\ref{eq:easgd_update1} which is followed by the master update given in Equation~\ref{eq:easgd_update2}.

We will now focus on the one-dimensional quadratic case without noise, i.e. $F(x) = \frac{x^2}{2},x\in \mathbb{R}$.

For the \textit{ADMM} algorithm, let the state of the (dynamical) system at time $t$ be $s_t=(\lambda_t^1,x_t^1,\ldots,\lambda_t^p,x_t^p,\tilde{x}_t) \in \mathbb{R}^{2p+1}$. The local worker $i$'s updates in Equations~\ref{eq:admm_update1},~\ref{eq:admm_update2}, and~\ref{eq:admm_update3} are composed of
three linear maps which can be written as $s_{t+1}=(F^i_3 \circ F^i_2 \circ  F^i_1)  (s_t)$. For simplicity,
we will only write them out below for the case when $i=1$ and $p=2$:
\small{\[ 
F_1^1\!\!=\!\! \left( \begin{array}{ccccc}
1 & -1 & 0 & 0 & 1 \\
0 & 1& 0 & 0 & 0  \\
0 & 0 & 1 &0  &0 \\
0 & 0 & 0 &1  &0 \\
0 & 0 & 0 &0  &1\\
\end{array} \right)\!\!,
F_2^1\!\!=\!\! \left( \begin{array}{ccccc}
1 & 0 & 0 & 0 & 0 \\
\frac{\eta \rho}{1+\eta \rho} & \frac{1-\eta}{1+\eta\rho} & 0 & 0 & \frac{\eta\rho}{ 1+\eta\rho} \\
0 & 0 & 1 &0  &0 \\
0 & 0 & 0 &1  &0 \\
0 & 0 & 0 &0  &1\\
\end{array} \right)\!\!,
F_3^1\!\!=\!\! \left( \begin{array}{ccccc}
1 & 0 & 0 & 0 & 0 \\
0 & 1 & 0 & 0 & 0  \\
0 & 0 & 1 &0  &0 \\
0 & 0 & 0 &1  &0 \\
-\frac{1}{p} & \frac{1}{p} & -\frac{1}{p} &\frac{1}{p}  &0\\
\end{array} \right)\!\!.
\]}

\normalsize
For each of the $p$ linear maps, it's possible to find a simple condition such that 
each map, where the $i^{\text{th}}$ map has the form $F^i_3 \circ F^i_2 \circ  F^i_1$, is stable (the absolute value of the eigenvalues of the map are smaller or equal to one). 
However, when these non-symmetric maps are composed one after another as follows $\mathcal{F}=F^p_3 \circ F^p_2 \circ  F^p_1 \circ \ldots \circ F^1_3 \circ F^1_2 \circ  F^1_1$,
the resulting map $\mathcal{F}$ can become unstable! 
(more precisely, some eigenvalues of the map can sit outside
the unit circle in the complex plane).

We now present the numerical conditions for which the \textit{ADMM} algorithm becomes unstable in the round-robin scheme for $p=3$ and $p=8$, by computing the largest absolute eigenvalue of the map $\mathcal{F}$. 
Figure \ref{fig:admm_fig} summarizes the obtained result.
\begin{figure}[h]
\vspace{-0.15in}
  \center
\includegraphics[trim=0in 0in 0in 0in,clip,width = 2.5in]{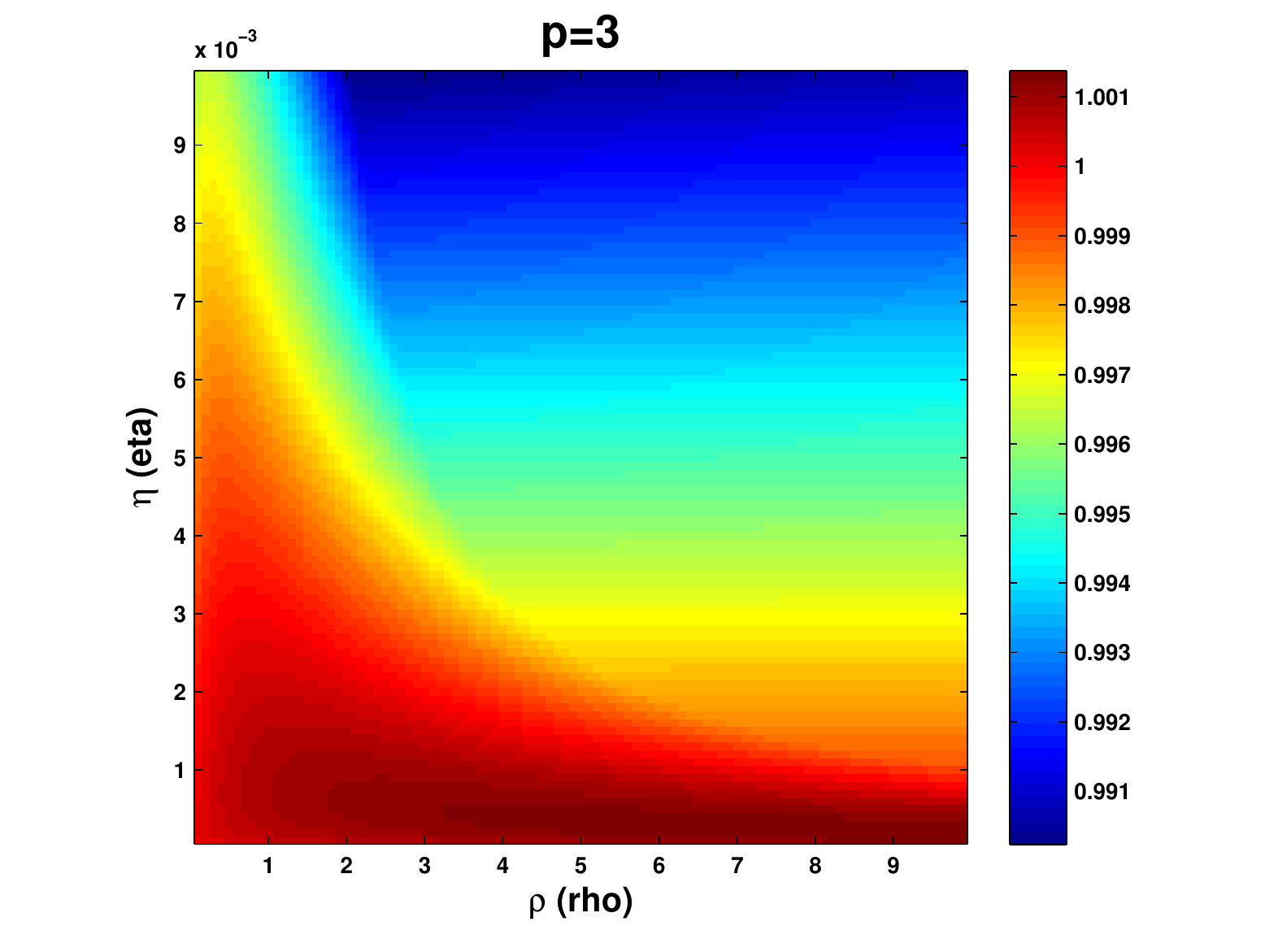}
\includegraphics[trim=0in 0in 0in 0in,clip,width = 2.5in]{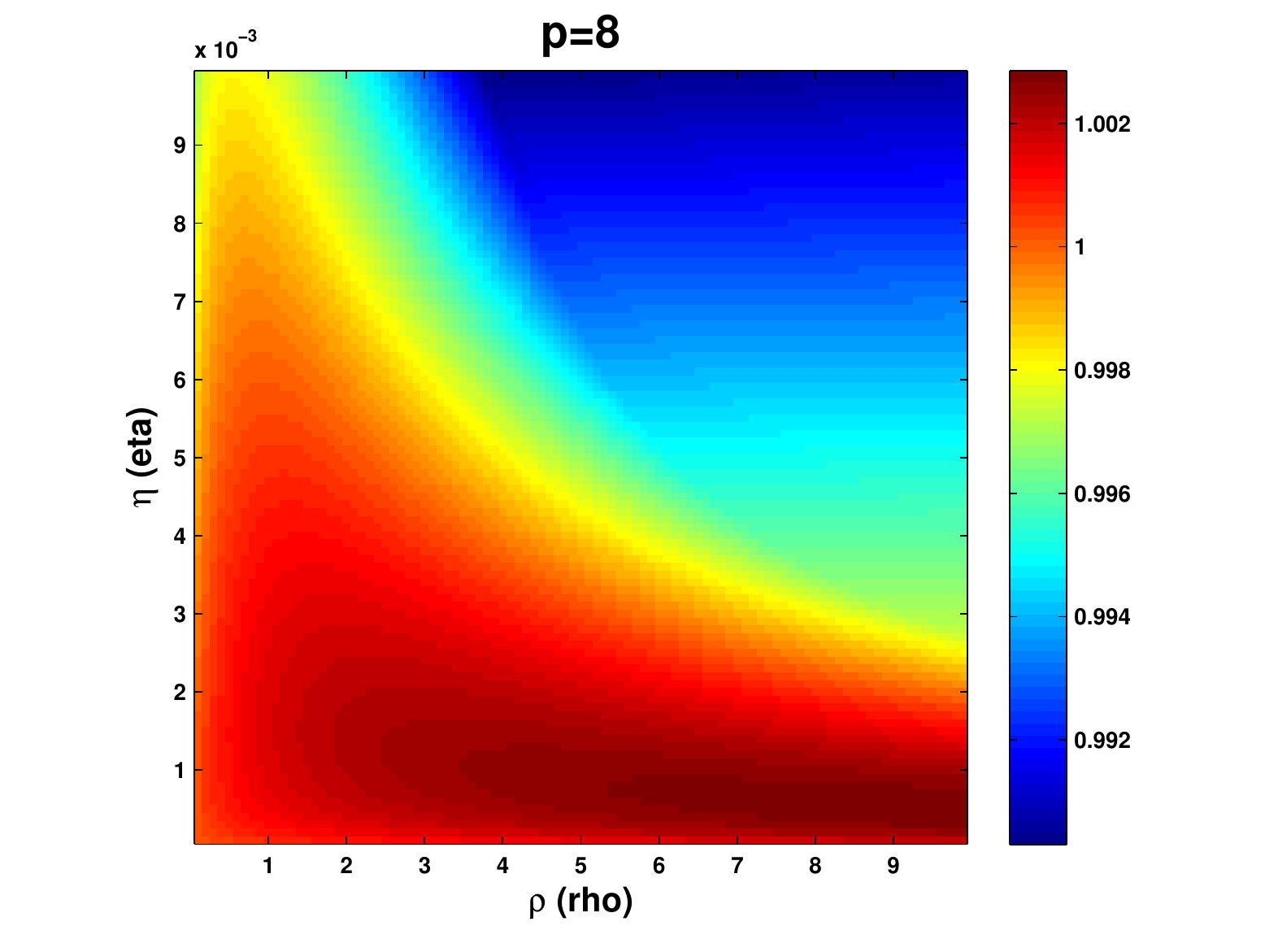}
\vspace{-0.15in}
\caption{The largest absolute eigenvalue of the linear map $\mathcal{F}=F^p_3 \circ F^p_2 \circ  F^p_1 \circ \ldots \circ F^1_3 \circ F^1_2 \circ  F^1_1$ as a function of $\eta \in (0,10^{-2})$ and $\rho \in (0,10)$ when $p=3$ and $p=8$. 
To simulate the chaotic behavior of the \textit{ADMM} algorithm, one may pick $\eta=0.001$ and $\rho=2.5$ 
and initialize the state $s_0$ either randomly or with $\lambda^i_0=0,x^i_0=\tilde{x}_0=1000, \forall i$.
\textit{Figure should be read in color}.}
\label{fig:admm_fig}
\vspace{-0.15in}
\end{figure}

On the other hand, the \textit{EASGD} algorithm involves composing only symmetric linear maps due to
the elasticity. Let the state of the (dynamical) system at time $t$ be $s_t=(x^1_t,\ldots,x^p_t,\tilde{x}_t) \in \mathbb{R}^{p+1}$. The activated local worker $i$'s update in Equation~\ref{eq:easgd_update1} and the master update in Equation~\ref{eq:easgd_update2} can be written as $s_{t+1}=F^i(s_t)$.
In case of $p=2$, the map $F^1$ and $F^2$ are defined as follows
\[
F^1\!\!=\!\! \left( \begin{array}{ccc}
1-\eta-\alpha & 0 & \alpha \\
0 & 1& 0 \\
\alpha & 0 &1-\alpha \\
\end{array} \right)\!\!,
F^2\!\!=\!\! \left( \begin{array}{ccc}
1 & 0 & 0 \\
0 & 1-\eta-\alpha& \alpha \\
0 & \alpha &1-\alpha \\
\end{array} \right)\!\!
\]
For the composite map $F^p \circ \ldots  \circ  F^1 $ to be stable, the condition that needs to be satisfied is actually the same for each $i$, and is furthermore independent of $p$ (since each linear map $F^i$ is symmetric). It essentially involves the stability of the $2\times 2$ matrix $ \small\left( \begin{array}{cc}
1-\eta-\alpha & \alpha \\
\alpha &1-\alpha 
\end{array} \right)$, whose two (real) eigenvalues $\lambda$ satisfy $(1-\eta-\alpha-\lambda)(1-\alpha-\lambda)=\alpha^2$. The resulting stability condition ($|\lambda| \le 1$) is simple and given as $0 \le \eta \le 2, 0 \le \alpha \le \frac{4-2 \eta}{4-\eta}$.

\section{Experiments}
\label{sec:Experiments}

In this section we compare the performance of \textit{EASGD} and \textit{EAMSGD} with the parallel method \textit{DOWNPOUR} and the sequential method \textit{SGD}, as well as their averaging and momentum variants. 

All the parallel comparator methods are listed below\footnote{We have \redit{compared} asynchronous \textit{ADMM}~\cite{icml2014c2_zhange14} with \textit{EASGD} in our setting as well, the performance is nearly the same. However, \textit{ADMM}'s momentum variant is not as stable for large communication periods.}:
\begin{packed_item}
\vspace{-0.05in}
\item \textit{DOWNPOUR}~\cite{2012distbelief}, the pseudo-code of the implementation of \textit{DOWNPOUR} used in this paper is enclosed in the Supplement.
\item \textit{Momentum DOWNPOUR} (\textit{MDOWNPOUR}), where the Nesterov's momentum scheme is applied to the master's update (note it is unclear how to apply it to the local workers or for the case when $\tau>1$). The pseudo-code is in the Supplement.
\item A method that we call \textit{ADOWNPOUR}, where we compute the average over time of the center variable $\tilde{x}$ as follows: $z_{t+1} = (1-\alpha_{t+1}) z_t + \alpha_{t+1} \tilde{x}_t$,
and $\alpha_{t+1} = \frac{1}{t+1}$ is a moving rate, and $z_0 = \tilde{x}_0$. $t$ denotes the master clock, which is initialized to $0$ and incremented every time the center variable $\tilde{x}$ is updated.
\item A method that we call \textit{MVADOWNPOUR}, where we compute the moving average of the center variable $\tilde{x}$ as follows: $z_{t+1} = (1-\alpha) z_t + \alpha \tilde{x}_t$,
and the moving rate $\alpha$ was chosen to be constant, and $z_0 = \tilde{x}_0$. $t$ denotes the master clock and is defined in the same way as for the \textit{ADOWNPOUR} method.
\end{packed_item}
All the sequential comparator methods ($p=1$) are listed below:
\begin{packed_item}
\vspace{-0.05in}
\item \textit{SGD}~\cite{bottou-98x} with constant learning rate $\eta$.
\item \textit{Momentum SGD} (\textit{MSGD})~\cite{icml2013_sutskever13} with constant momentum $\delta$.
\item \textit{ASGD}~\cite{polyak1992acceleration} with moving rate $\alpha_{t+1} = \frac{1}{t+1}$.
\item \textit{MVASGD}~\cite{polyak1992acceleration} with moving rate $\alpha$ set to a constant.
\end{packed_item}

We perform experiments in a deep learning setting on two benchmark datasets: CIFAR-10 (we refer to it as \textit{CIFAR}) \footnote{Downloaded from \url{http://www.cs.toronto.edu/~kriz/cifar.html}.} and ImageNet ILSVRC 2013 (we refer to it as \textit{ImageNet}) \footnote{Downloaded from \url{http://image-net.org/challenges/LSVRC/2013}.}. We focus on the image classification task with deep convolutional neural networks. We next explain the experimental setup. The details of the data preprocessing and prefetching are deferred to the Supplement.

\subsection{Experimental setup}
For all our experiments we use a GPU-cluster interconnected with InfiniBand. Each node has $4$ Titan GPU processors where each local worker corresponds to one GPU processor. The center variable of the master is stored and updated on the centralized parameter server~\cite{2012distbelief}\footnote{Our implementation is available at \url{https://github.com/sixin-zh/mpiT}.}.

To describe the architecture of the convolutional neural network, we will first introduce a notation. Let $(c,y)$ denotes the size of the input image to each layer, where $c$ is the number of color channels and $y$ is both the horizontal and the vertical dimension of the input. Let $C$ denotes the fully-connected convolutional operator and let $P$ denotes the max pooling operator, $D$ denotes the linear operator with dropout rate equal to $0.5$ and $S$ denotes the linear operator with softmax output non-linearity. We use the cross-entropy loss and all inner layers use rectified linear units.
For the \textit{ImageNet} experiment we use the similar approach to~\cite{2013arXiv1312.6229S} with the following $11$-layer convolutional neural network (3,221)C(96,108)P(96,36)C(256,32)P(256,16)C(384,14)
C(384,13)C(256,12)P(256,6)D(4096,1)D(4096,1)S(1000,1). 
For the \textit{CIFAR} experiment we use the similar approach to~\cite{DBLPWanZZLF13} with the following $7$-layer convolutional neural network (3,28)C(64,24)P(64,12)C(128,8)P(128,4)C(64,2)D(256,1)S(10,1).

In our experiments all the methods we run use the same initial parameter chosen randomly, except that we set all the biases to zero for \textit{CIFAR} case and to 0.1 for  \textit{ImageNet} case. This parameter is used to initialize the master and all the local workers\footnote{On the contrary, initializing the local workers and the master with different random seeds \redit{'traps'} the algorithm in the symmetry breaking phase.}. We add $l_2$-regularization $\frac{\lambda}{2}\norm{}{x}^2$ to the loss function $F(x)$. For \textit{ImageNet} we use $\lambda = 10^{-5}$ and for \textit{CIFAR} we use $\lambda = 10^{-4}$. We also compute the stochastic gradient using mini-batches of sample size $128$.

\subsection{Experimental results}

For all experiments in this section we use \textit{EASGD} with $\beta = 0.9$\footnote{Intuitively the 'effective $\beta$' is $\beta/\tau = p \alpha = p \eta \rho$ (thus $\rho = \frac{\beta}{\tau p \eta}$) in the asynchronous setting.}
, for all momentum-based methods we set the momentum term $\delta = 0.99$ and finally for \textit{MVADOWNPOUR} we set the moving rate to $\alpha = 0.001$. We start with the experiment on \textit{CIFAR} dataset with $p=4$ local workers running on a single computing node. For all the methods, we examined the communication periods from the following \redit{set} $\tau = \{1,4,16,64\}$. For comparison we also report the performance of \textit{MSGD} which outperformed \textit{SGD}, \textit{ASGD} and \textit{MVASGD} as shown in Figure~\ref{fig:CIFAR_supp} in the Supplement. For each method we examined a wide range of learning rates (the learning rates explored in all experiments are summarized in Table~\ref{tab:learningrate},~\ref{tab:learningrate2},~\ref{tab:learningrate3} in the Supplement). The \textit{CIFAR} experiment was run $3$ times independently from the same initialization and for each method we report its best performance measured by the smallest achievable test error.
\begin{figure}[h]
\vspace{-0.05in}
  \center
\includegraphics[trim=0cm 0cm 0cm 0cm,clip,width = 1.9in]{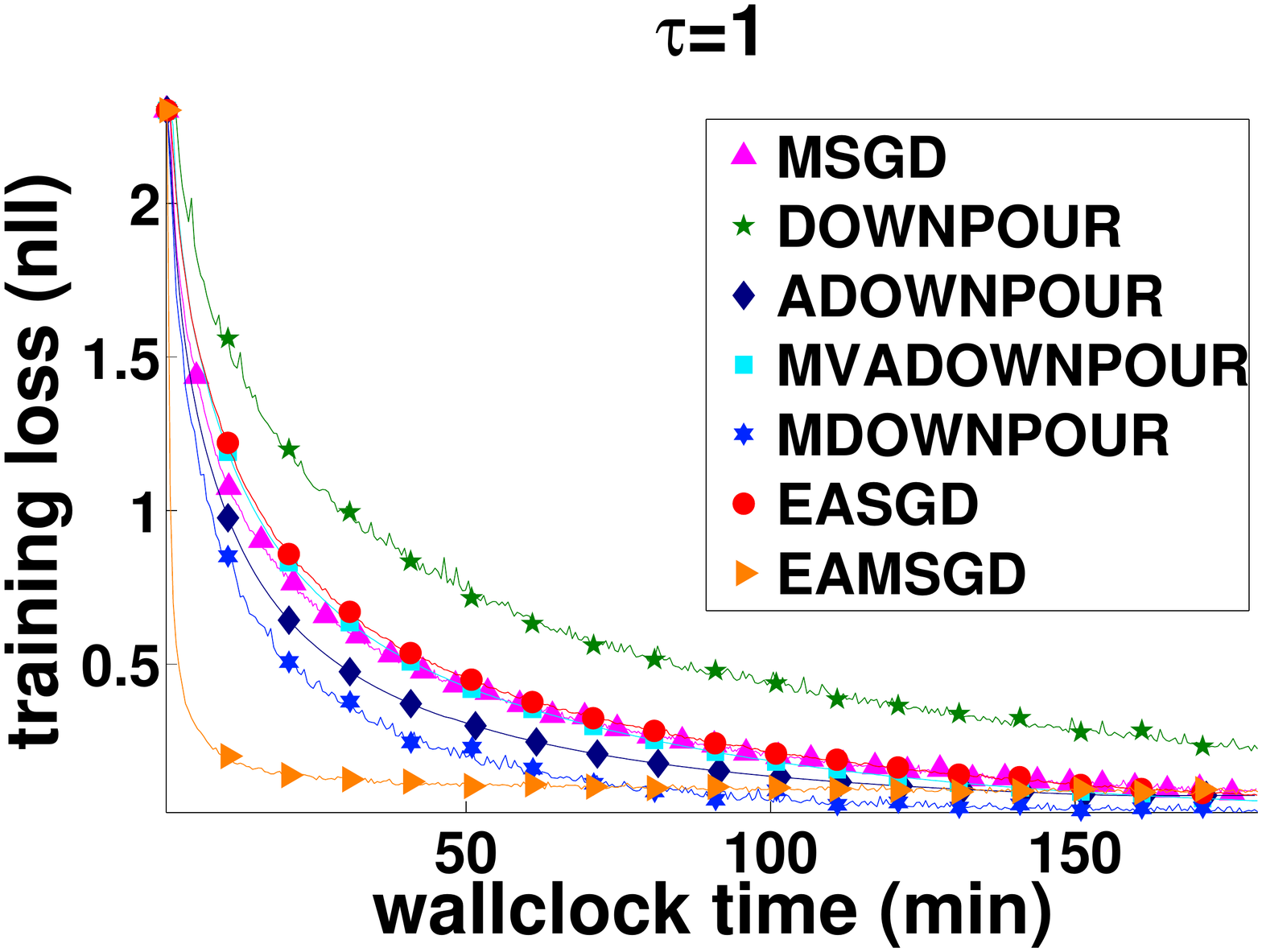}
\hspace{-0.3in}\includegraphics[trim=0cm 0cm 0cm 0cm,clip,width = 1.9in]{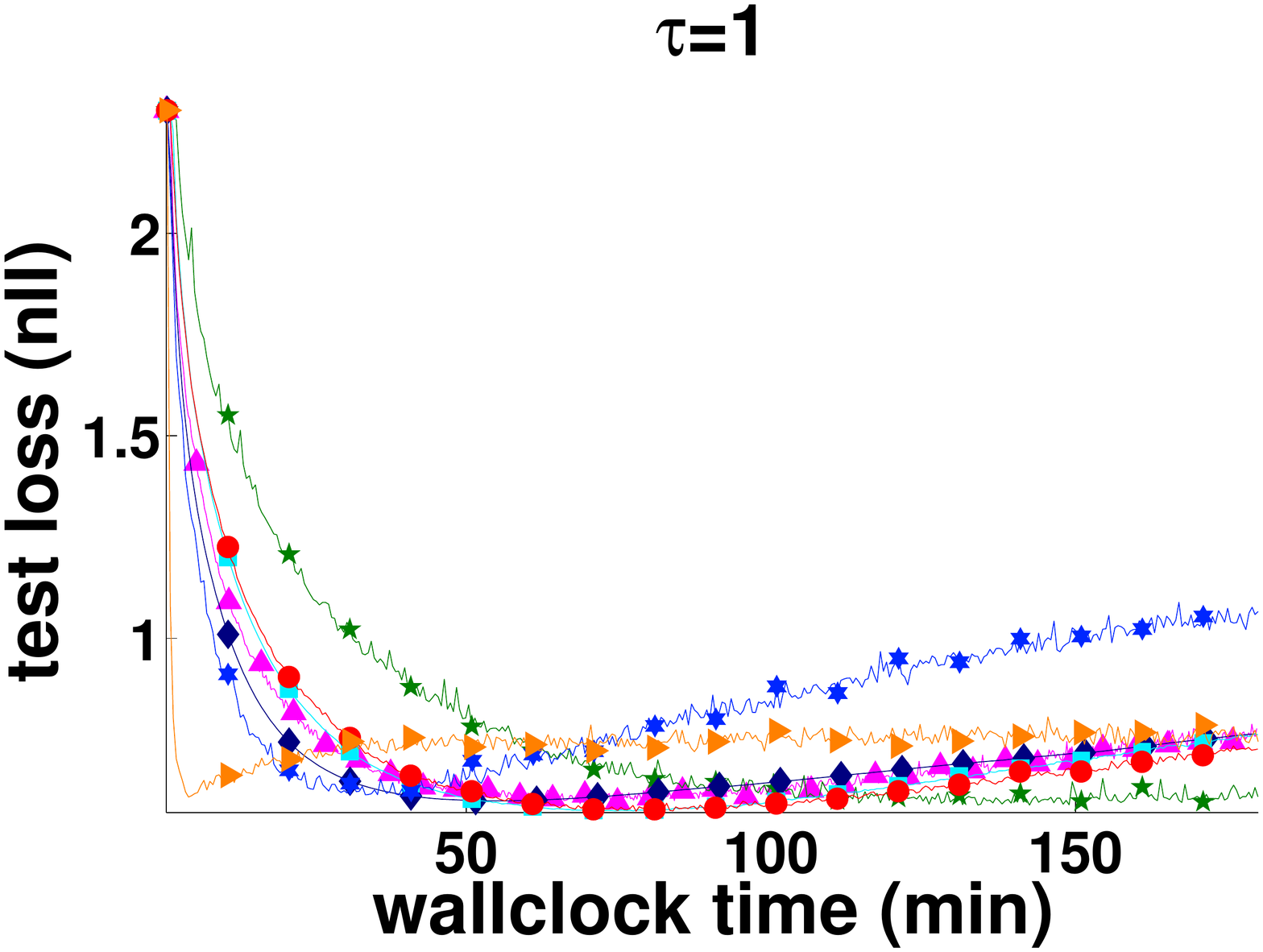} 
\hspace{-0.3in}\includegraphics[trim=0cm 0cm 0cm 0cm,clip,width = 1.9in]{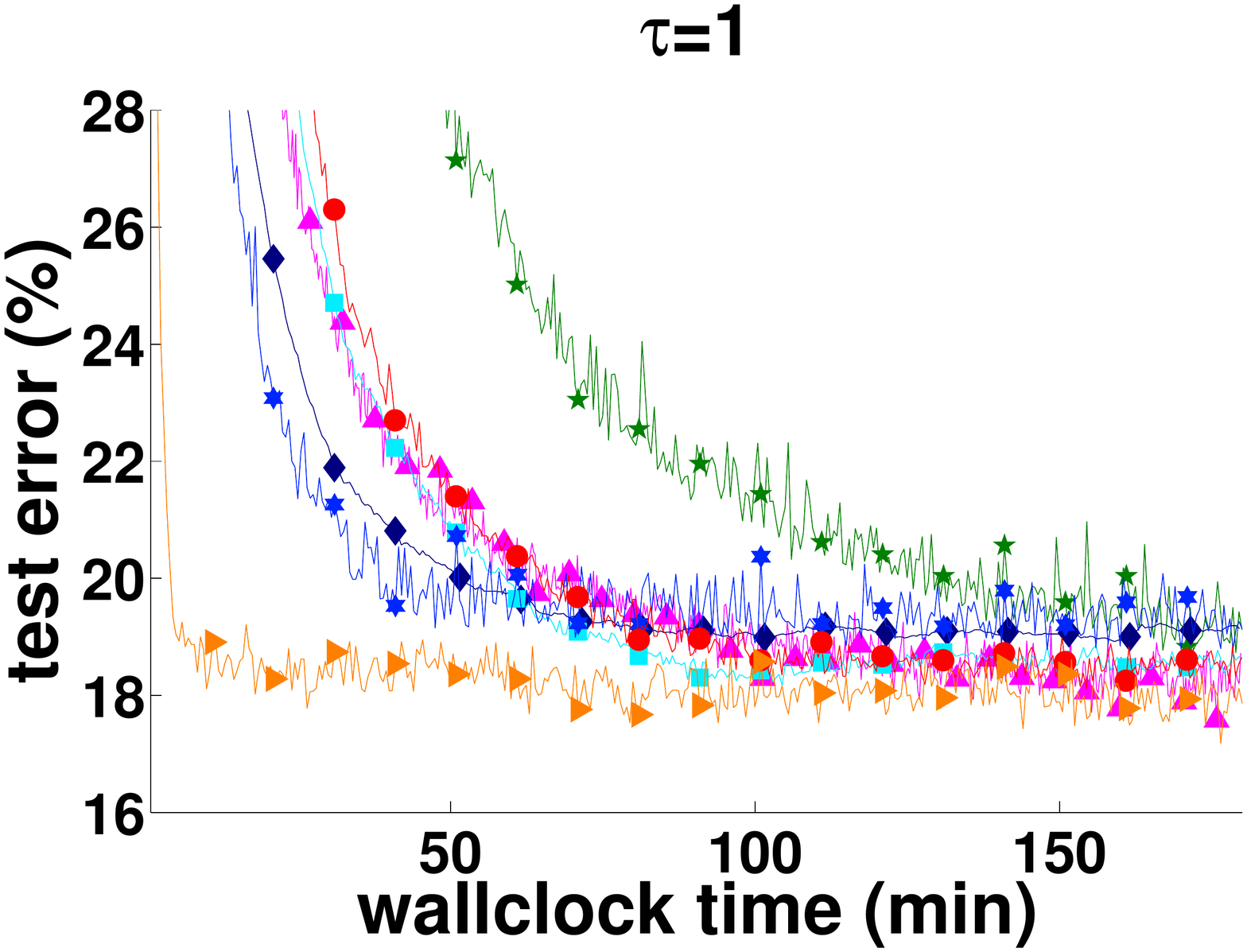}\\
\vspace{-0.05in}
\includegraphics[trim=0cm 0cm 0cm 0cm,clip,width = 1.9in]{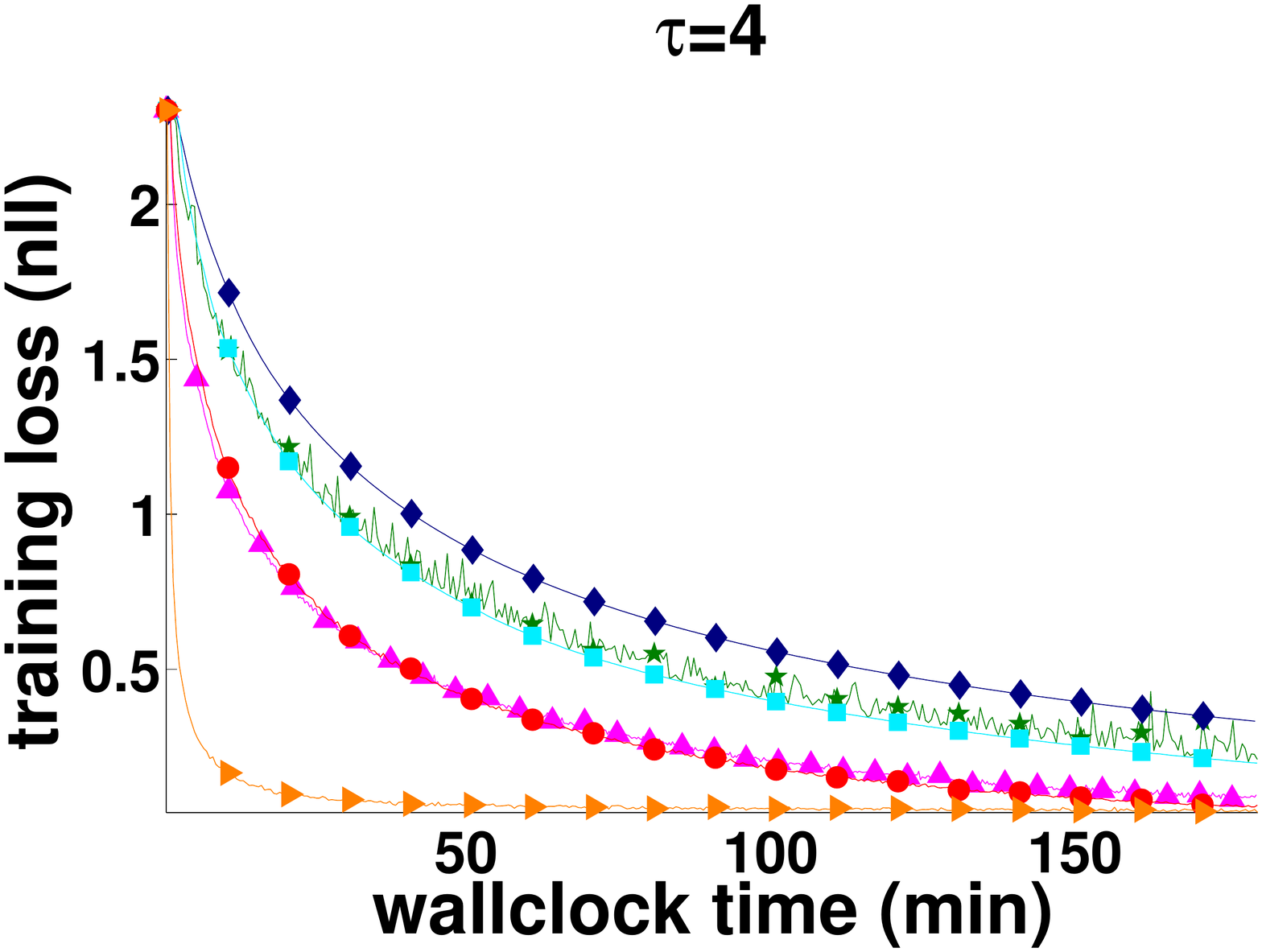}
\hspace{-0.3in}\includegraphics[trim=0cm 0cm 0cm 0cm,clip,width = 1.9in]{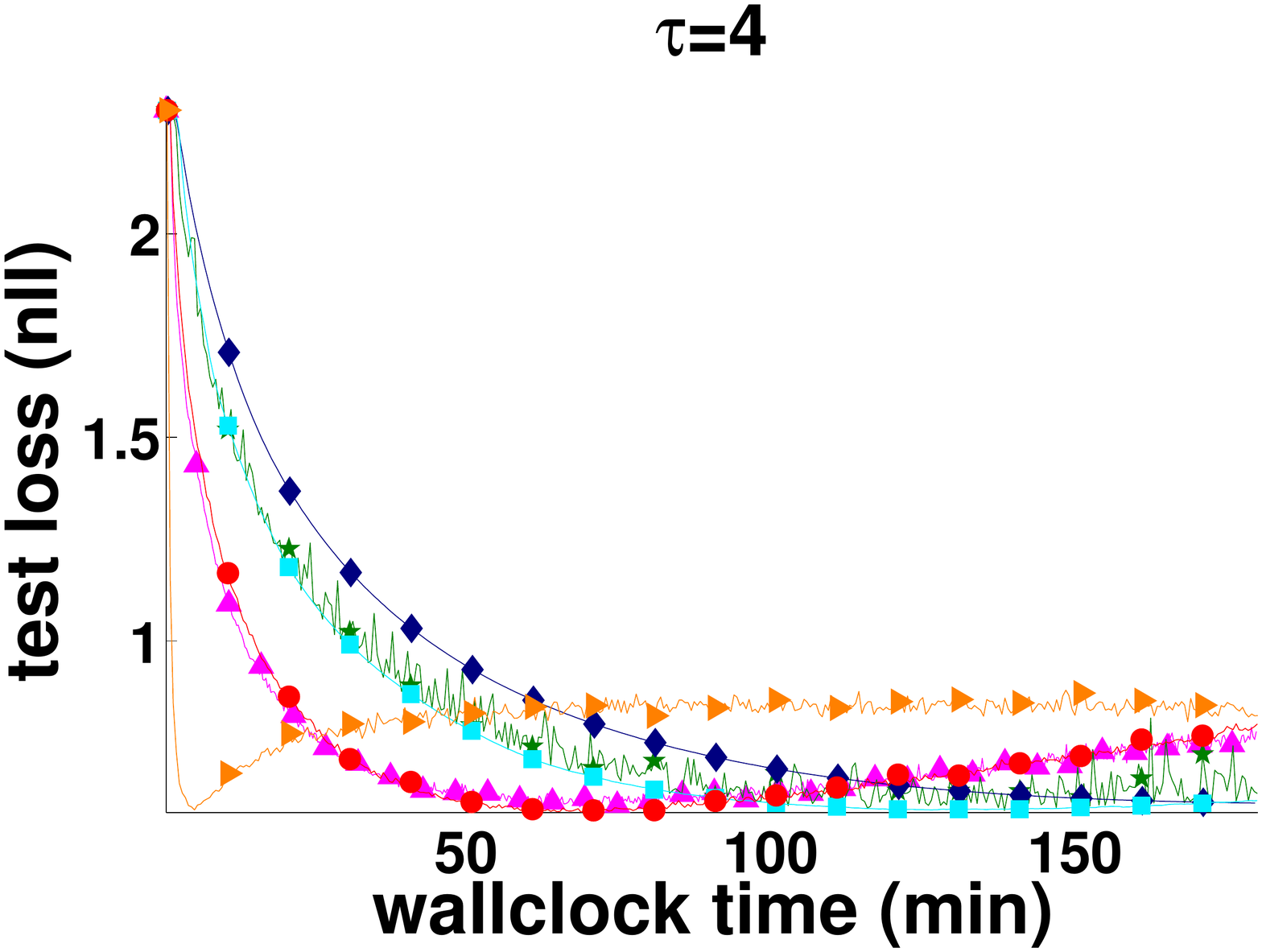} 
\hspace{-0.3in}\includegraphics[trim=0cm 0cm 0cm 0cm,clip,width = 1.9in]{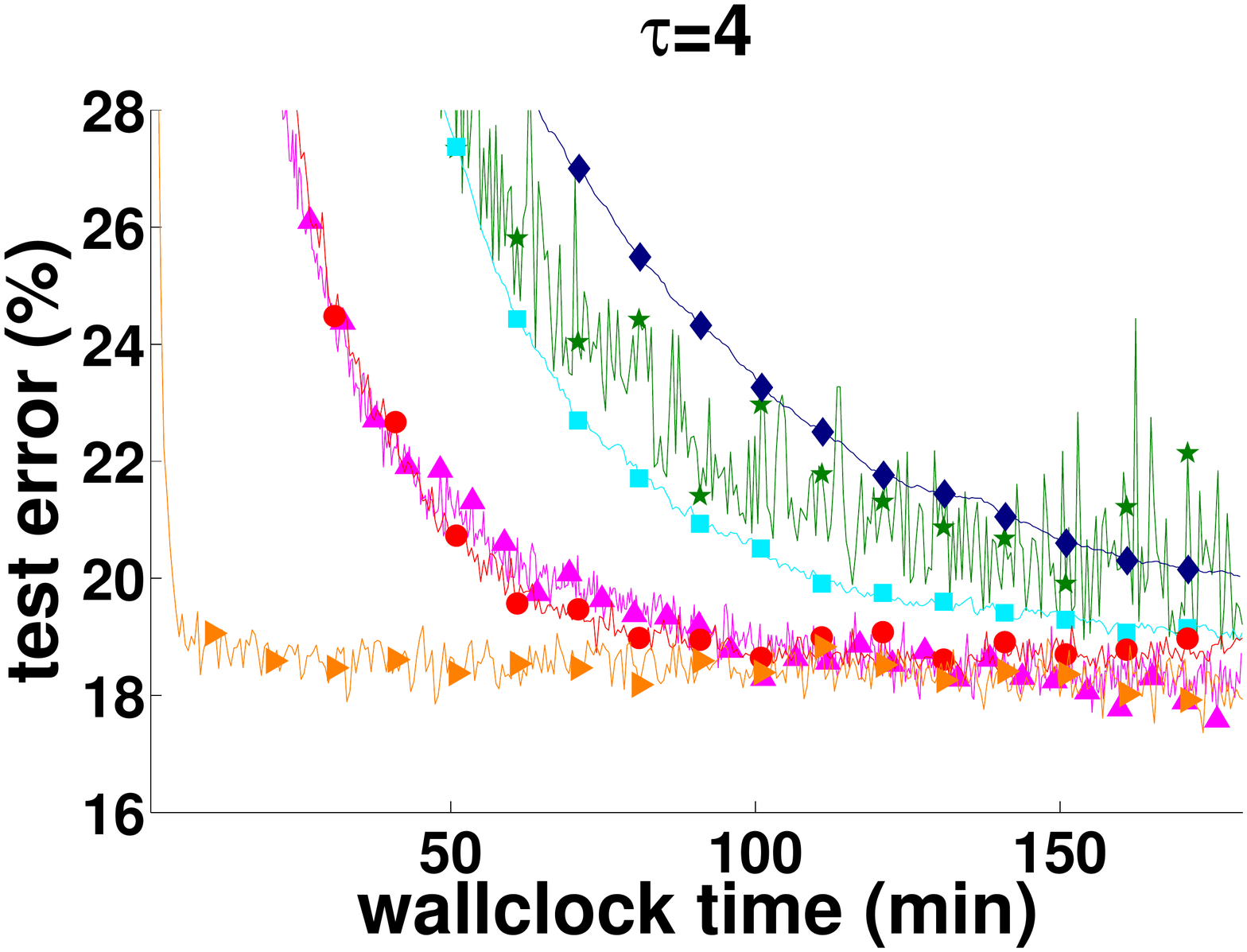}\\
\vspace{-0.05in}
\includegraphics[trim=0cm 0cm 0cm 0cm,clip,width = 1.9in]{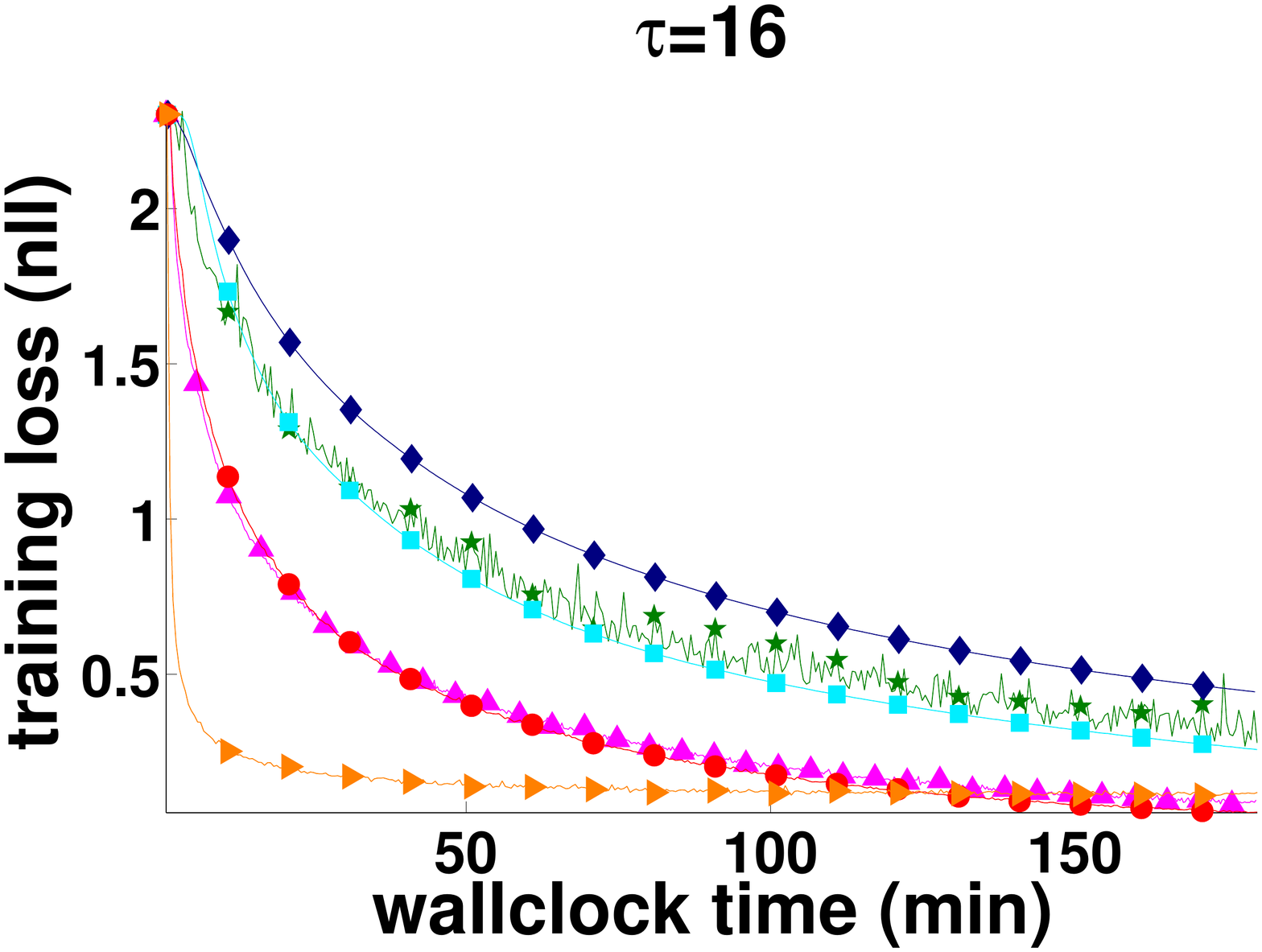}
\hspace{-0.3in}\includegraphics[trim=0cm 0cm 0cm 0cm,clip,width = 1.9in]{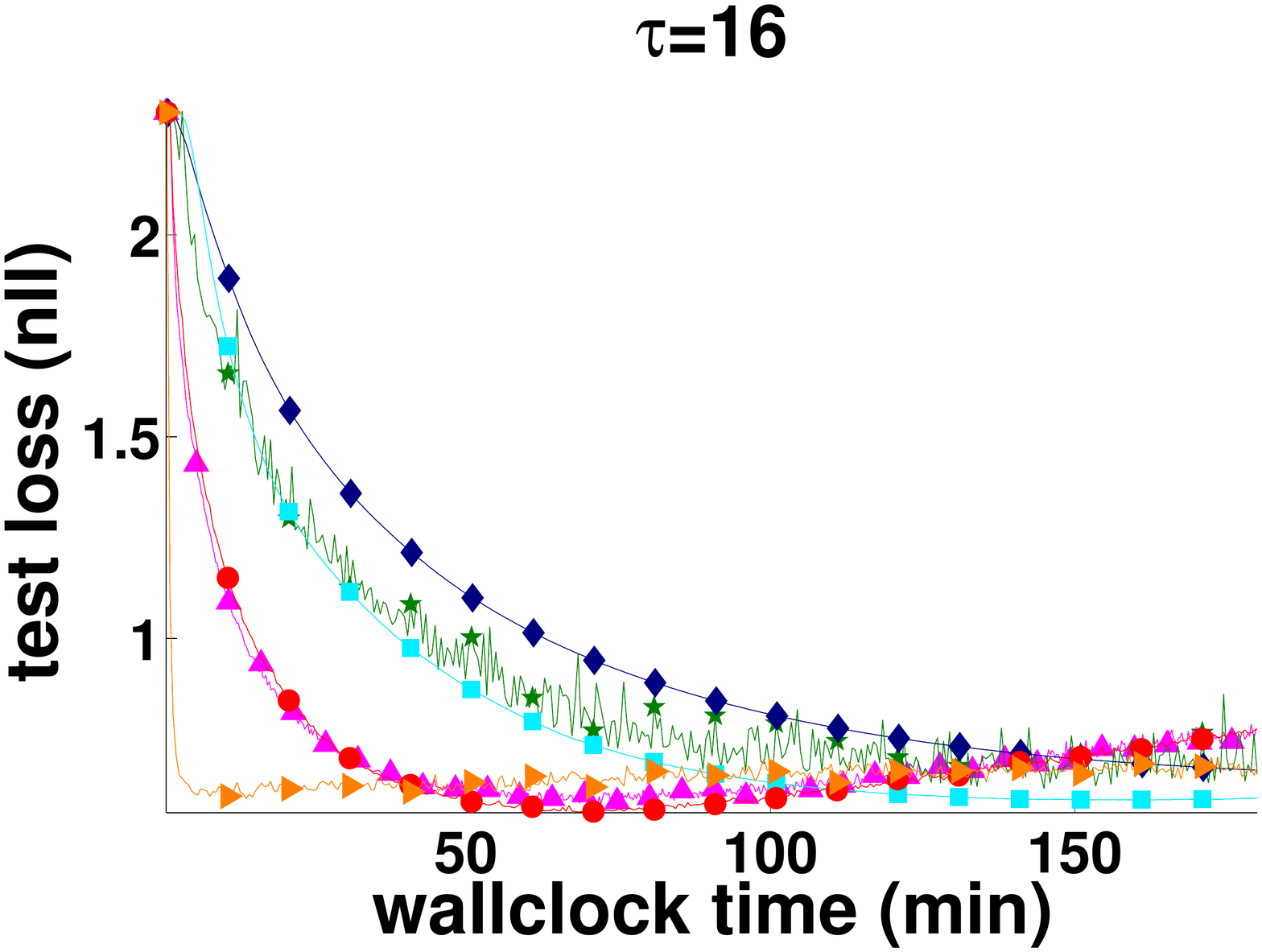} 
\hspace{-0.3in}\includegraphics[trim=0cm 0cm 0cm 0cm,clip,width = 1.9in]{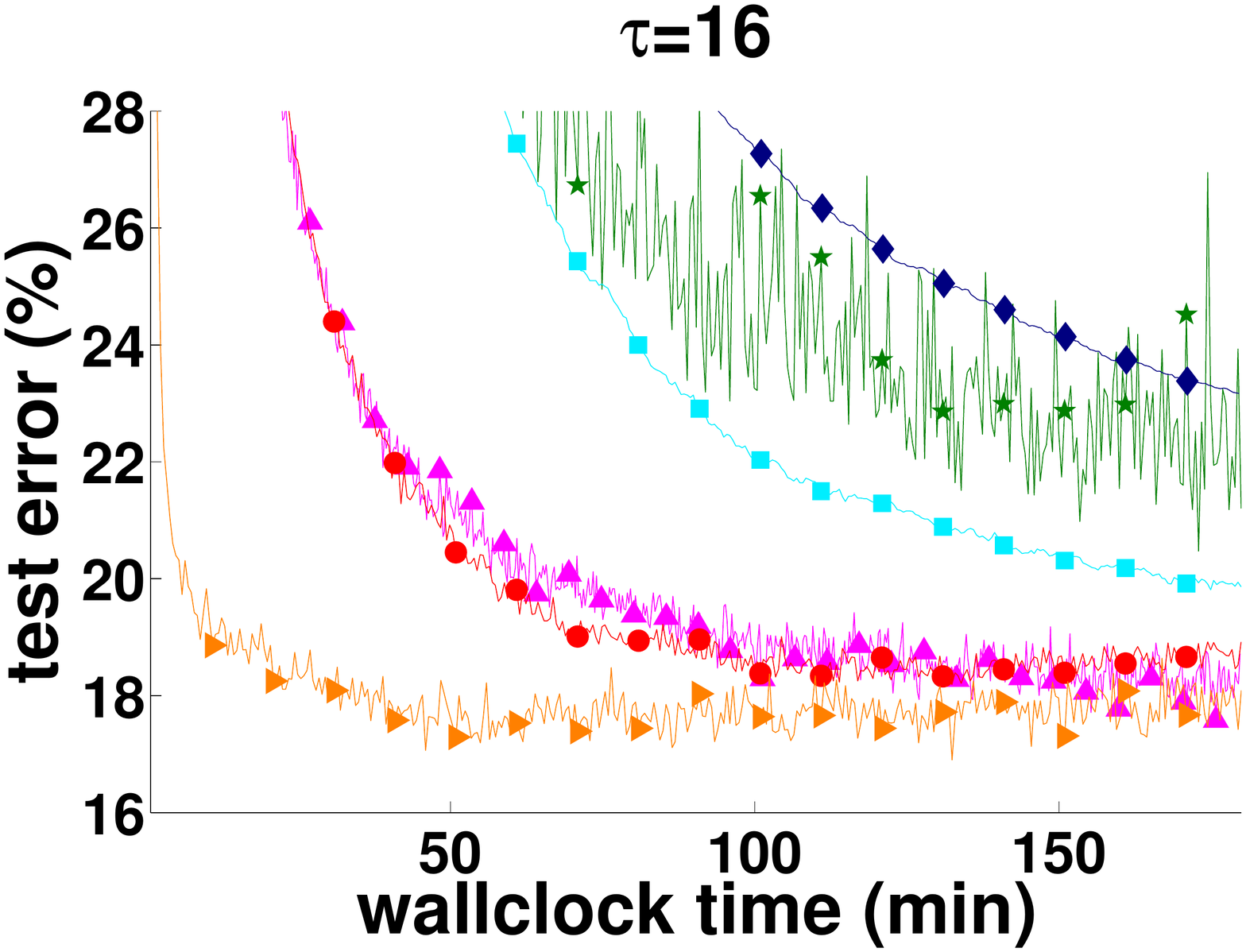}\\
\vspace{-0.05in}
\includegraphics[trim=0cm 0cm 0cm 0cm,clip,width = 1.9in]{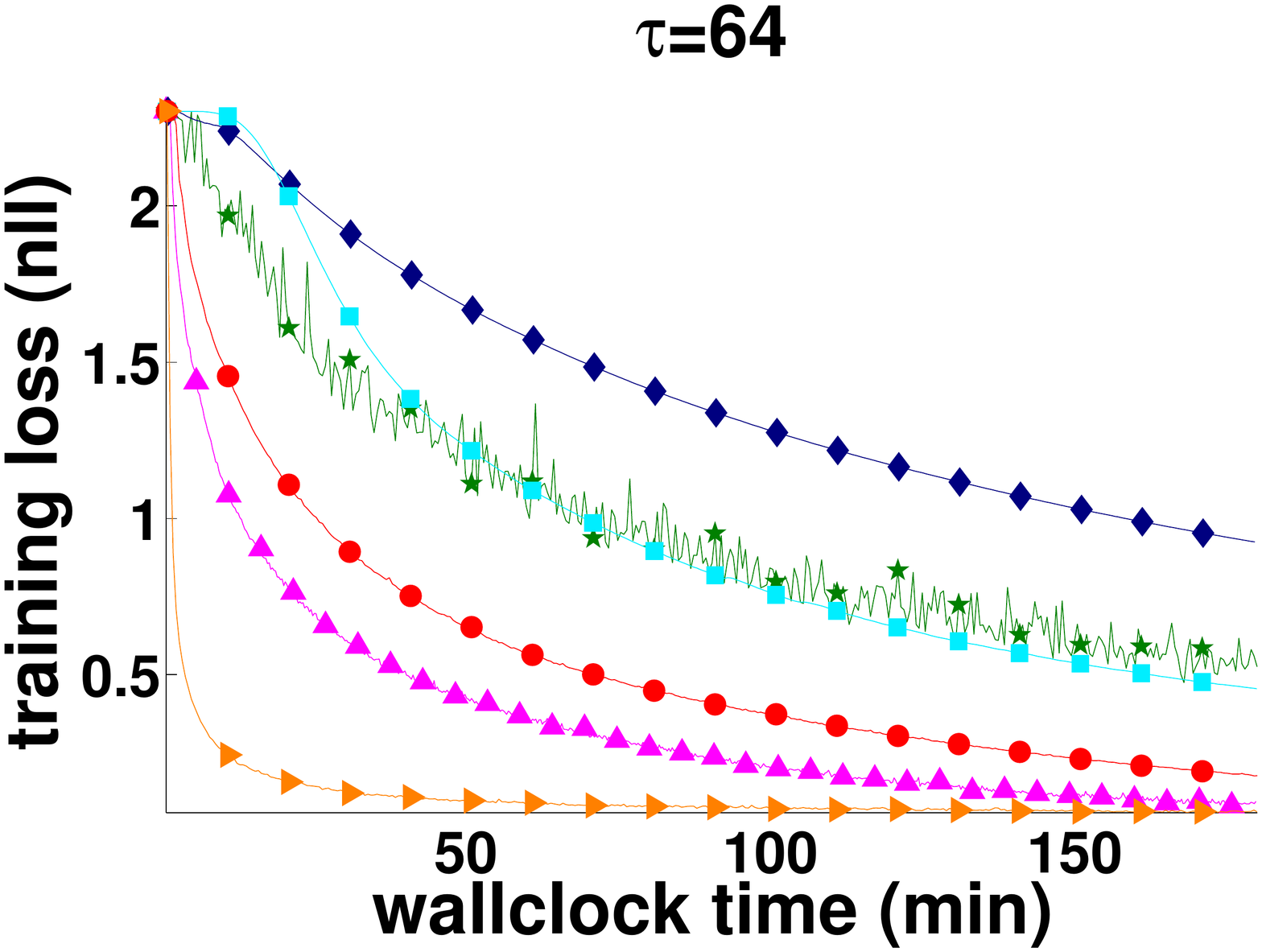}
\hspace{-0.3in}\includegraphics[trim=0cm 0cm 0cm 0cm,clip,width = 1.9in]{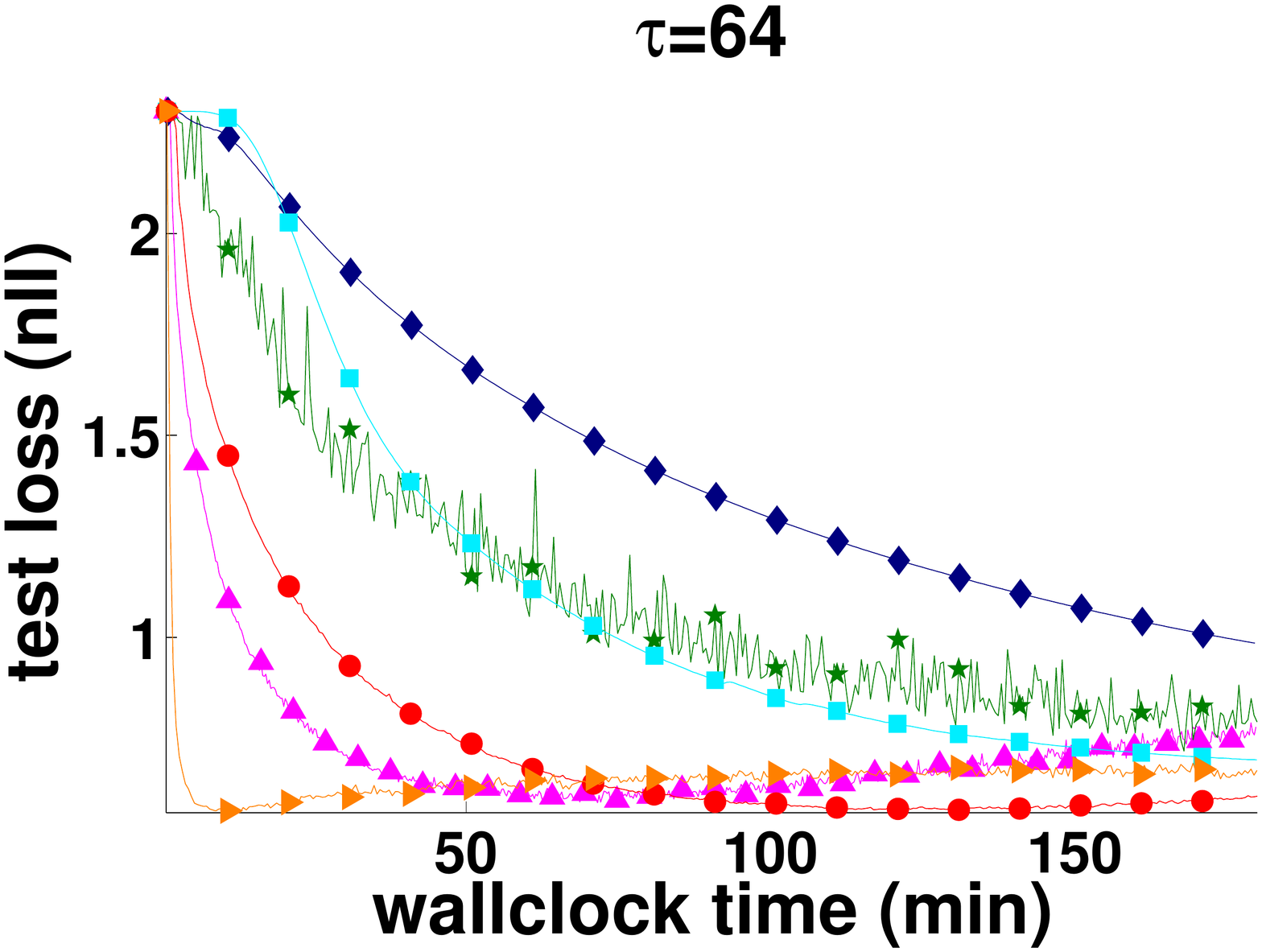} 
\hspace{-0.3in}\includegraphics[trim=0cm 0cm 0cm 0cm,clip,width = 1.9in]{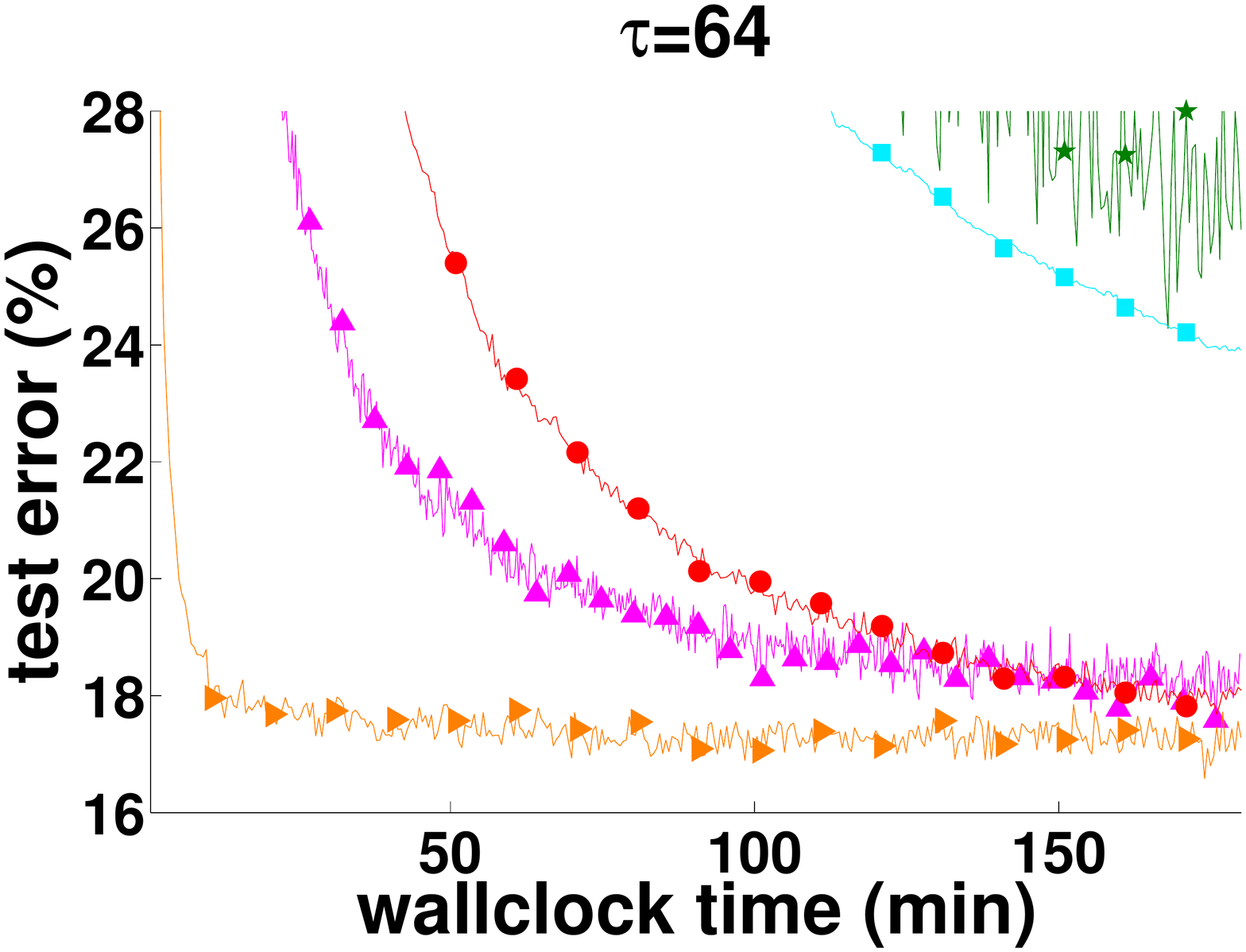}
\vspace{-0.15in}
\caption{Training and test loss and the test error for the center variable versus a wallclock time for different communication periods $\tau$ on \textit{CIFAR} dataset with the $7$-layer convolutional neural network.}
\label{fig:CIFAR}
\vspace{-0.2in}
\end{figure} 
From the results in Figure~\ref{fig:CIFAR}, 
we conclude that all \textit{DOWNPOUR}-based methods achieve their best performance (test error) for small $\tau$ ($\tau \in \{1,4\}$), and become highly unstable for $\tau \in \{16,64\}$. While \textit{EAMSGD} significantly outperforms comparator methods for all values of $\tau$ by having faster convergence. It also finds better-quality solution measured by the test error and this advantage becomes more significant for $\tau \in \{16,64\}$.
Note that the tendency to achieve better test performance with larger $\tau$ is also characteristic for the \textit{EASGD} algorithm. 
\begin{figure}[h!]
  \center
\includegraphics[trim=0cm 0cm 0cm 0cm,clip,width = 1.9in]{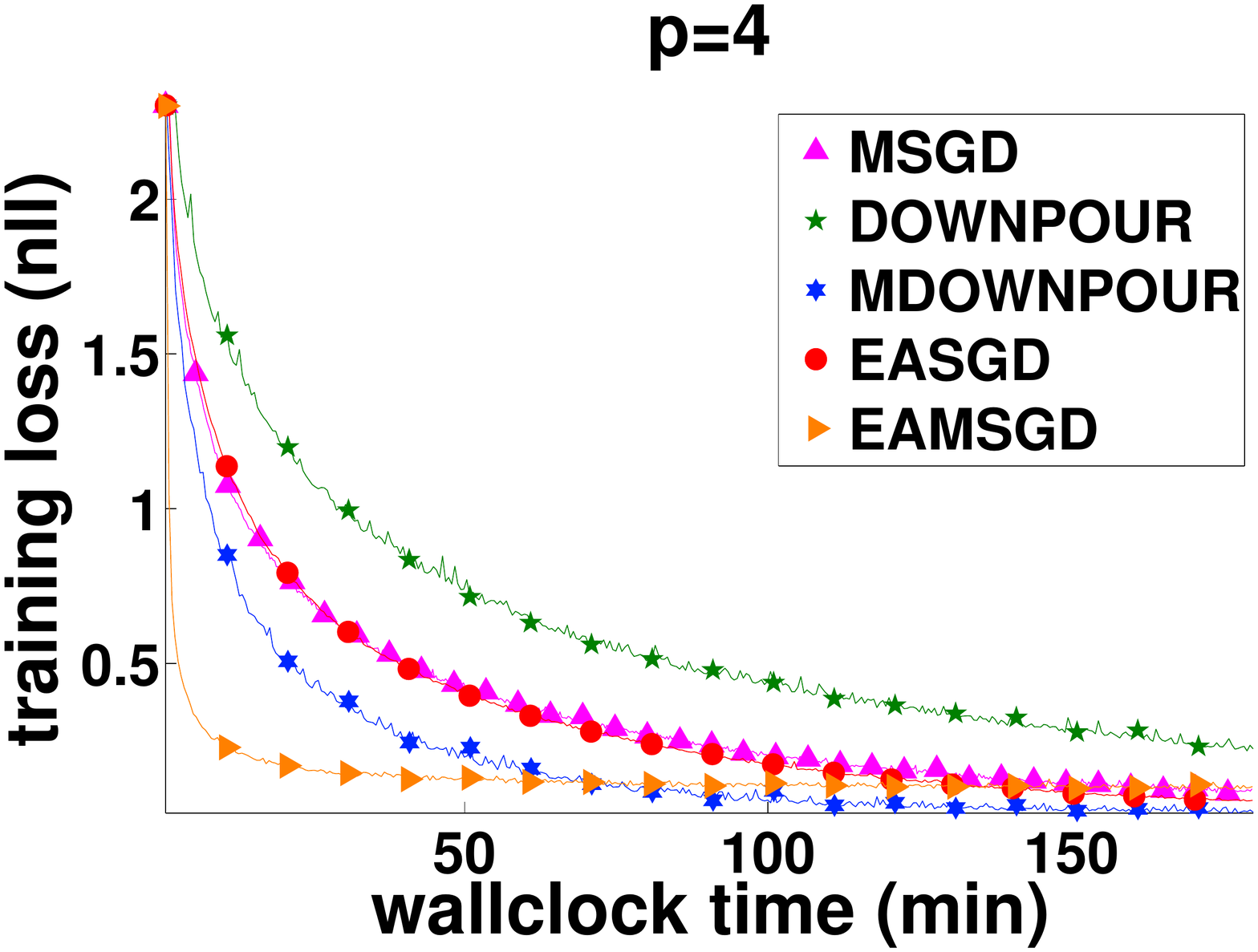}
\hspace{-0.3in}\includegraphics[trim=0cm 0cm 0cm 0cm,clip,width = 1.9in]{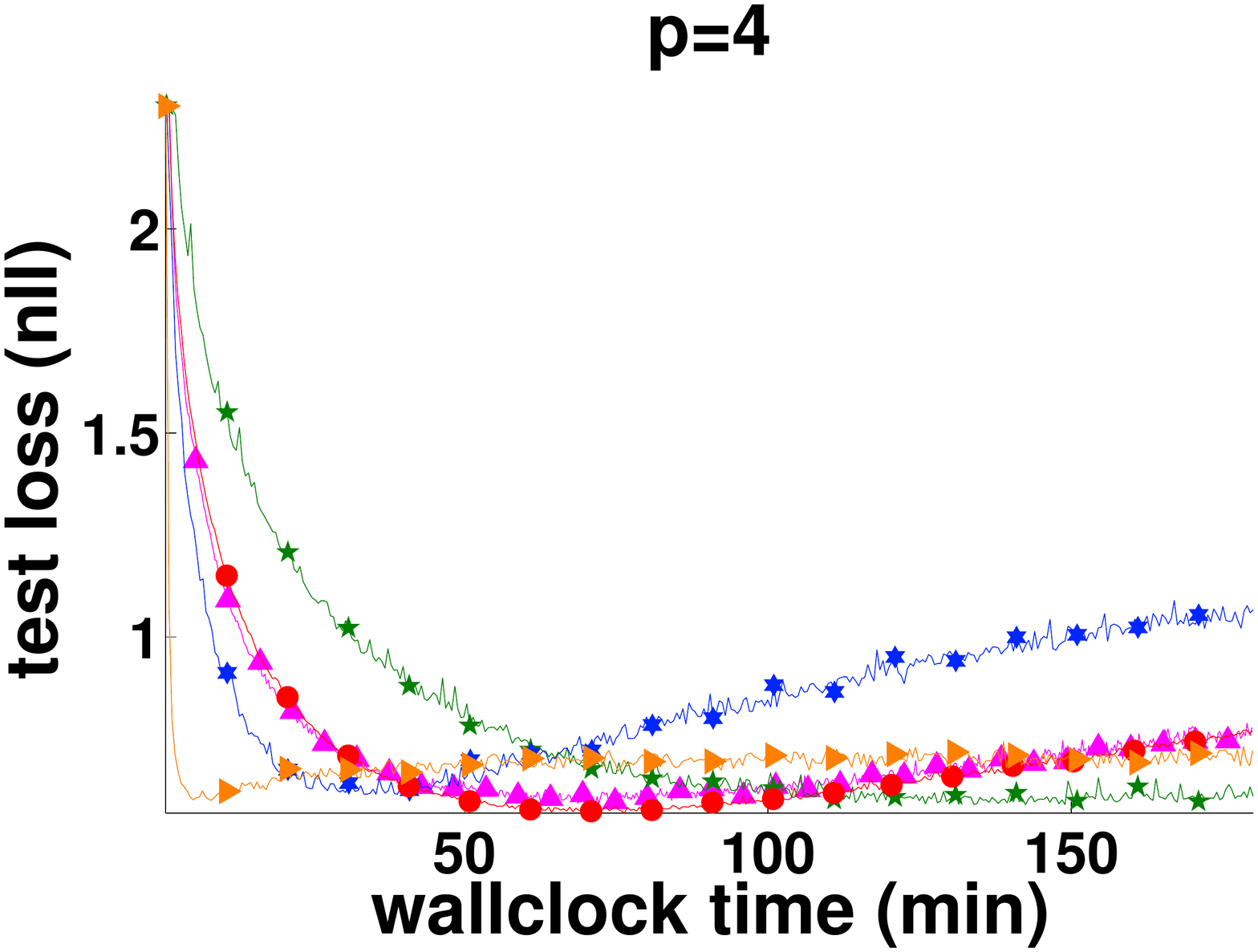} 
\hspace{-0.3in}\includegraphics[trim=0cm 0cm 0cm 0cm,clip,width = 1.9in]{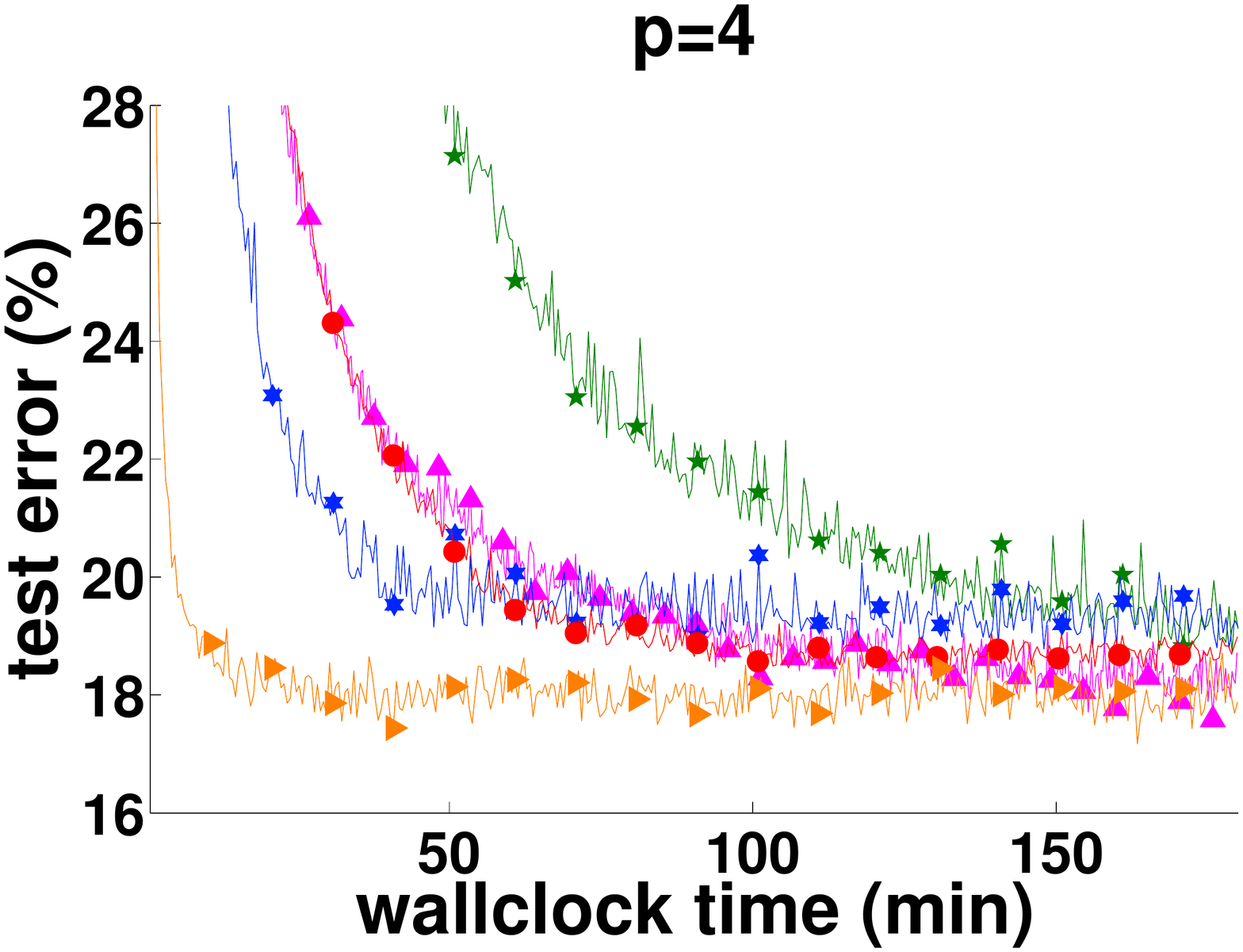}\\
\vspace{-0.13in}
\includegraphics[trim=0cm 0cm 0cm 0cm,clip,width = 1.9in]{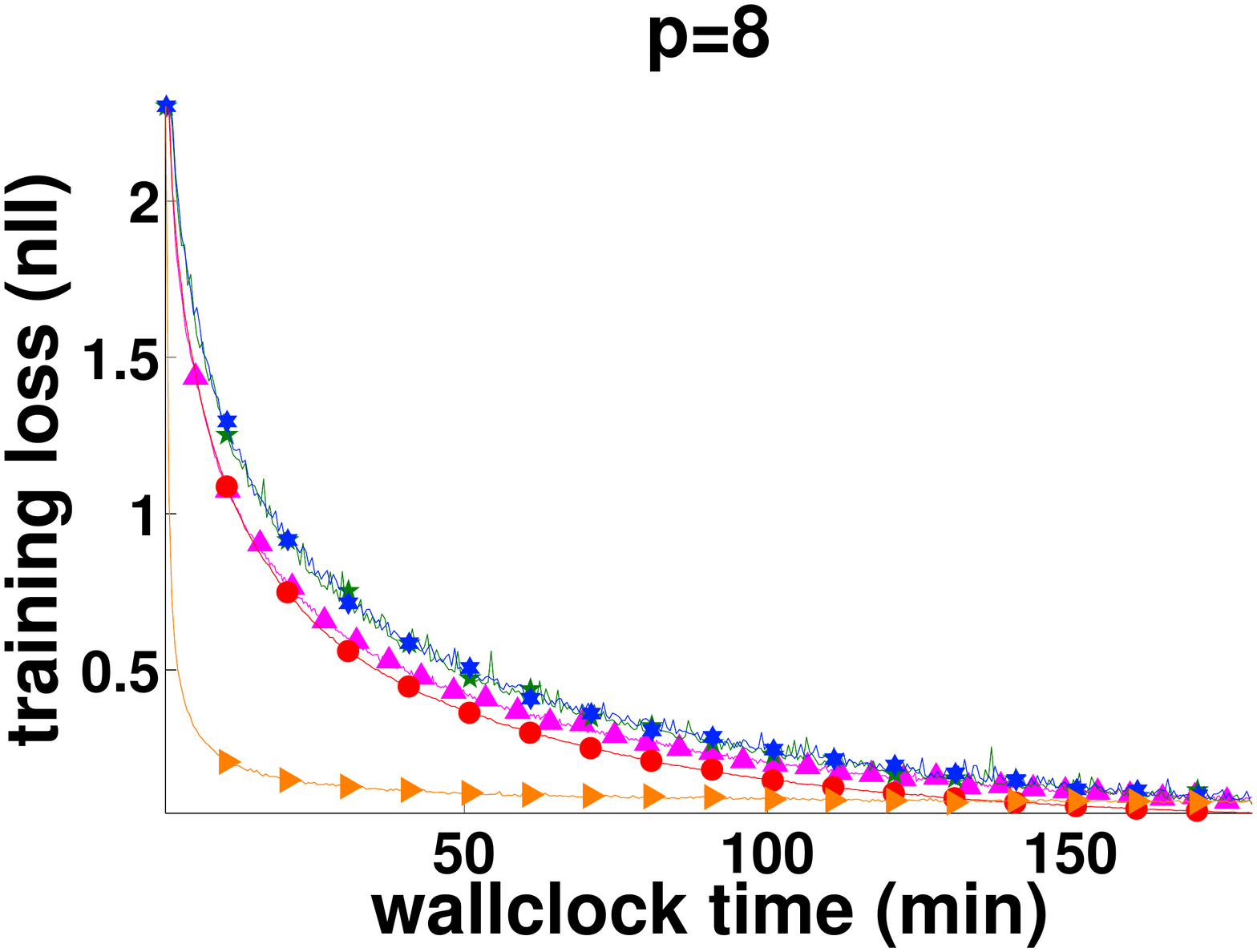}
\hspace{-0.3in}\includegraphics[trim=0cm 0cm 0cm 0cm,clip,width = 1.9in]{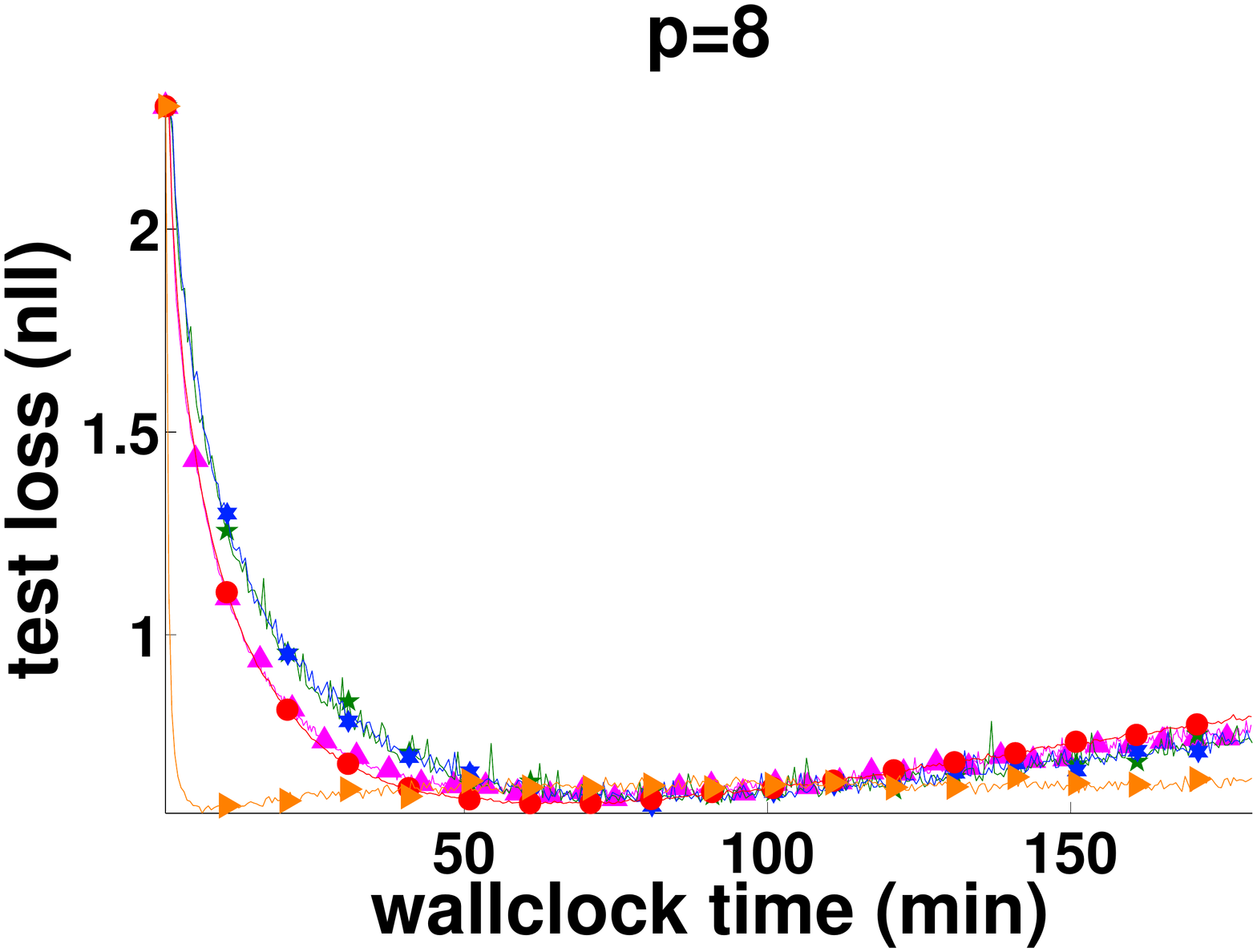} 
\hspace{-0.3in}\includegraphics[trim=0cm 0cm 0cm 0cm,clip,width = 1.9in]{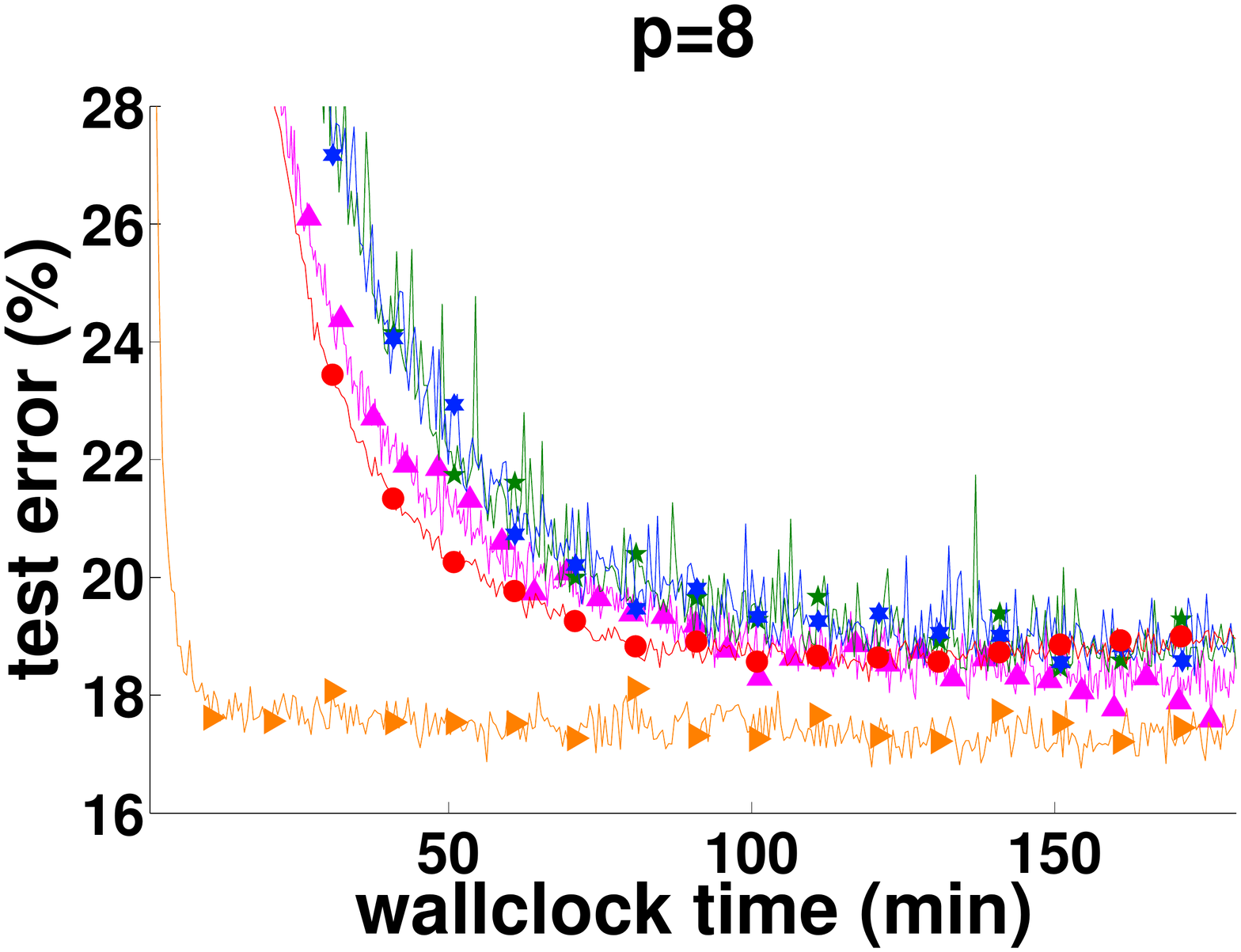}\\
\vspace{-0.13in}
\includegraphics[trim=0cm 0cm 0cm 0cm,clip,width = 1.9in]{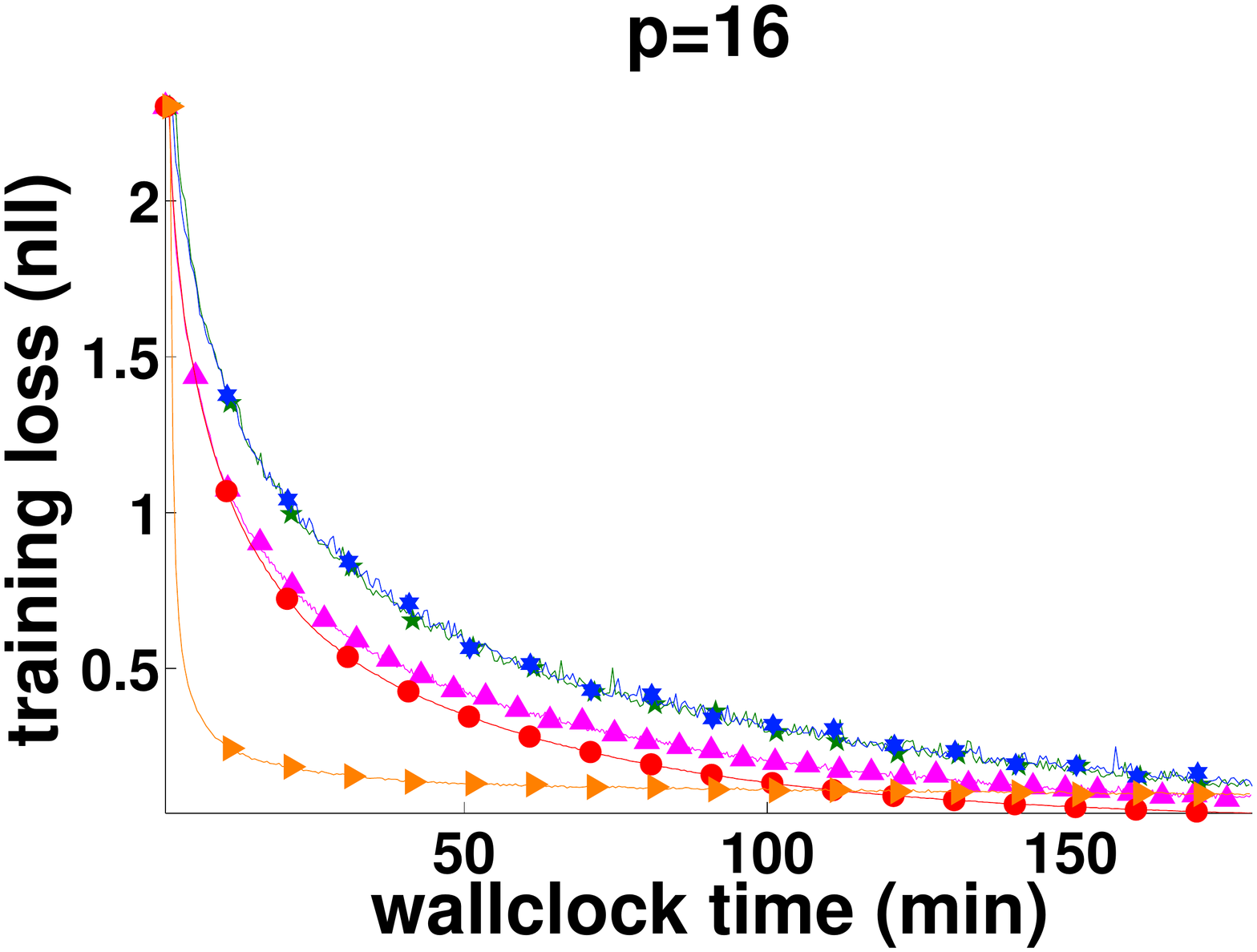}
\hspace{-0.3in}\includegraphics[trim=0cm 0cm 0cm 0cm,clip,width = 1.9in]{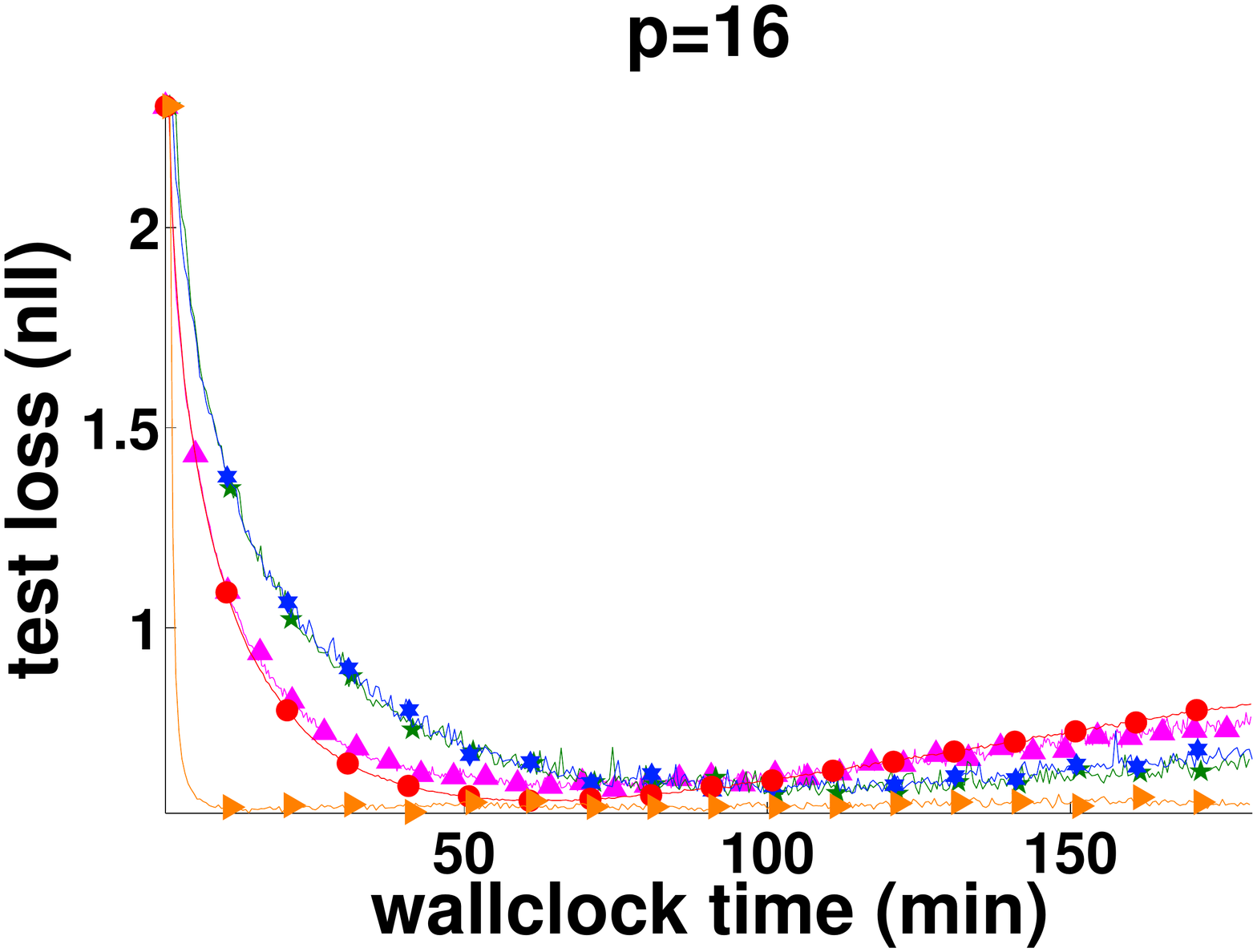} \hspace{-0.3in}\includegraphics[trim=0cm 0cm 0cm 0cm,clip,width = 1.9in]{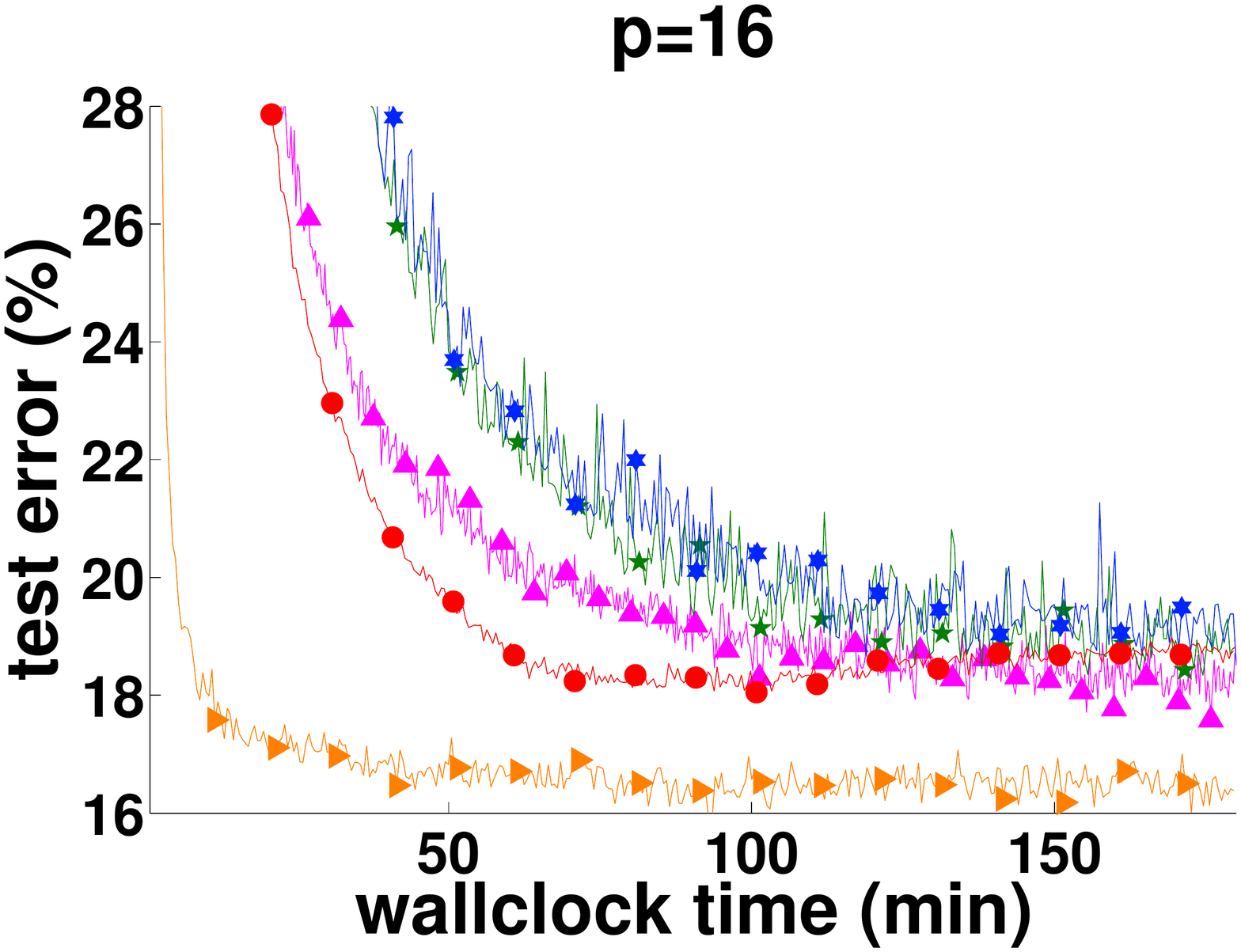}
\vspace{-0.13in}
\caption{Training and test loss and the test error for the center variable versus a wallclock time for different number of local workers $p$ for parallel methods (\textit{MSGD} uses $p=1$) on \textit{CIFAR} with the $7$-layer convolutional neural network. \textit{EAMSGD} achieves significant accelerations compared to other methods, e.g. the relative speed-up for $p=16$ (the best comparator method is then \textit{MSGD}) to achieve the test error $21\%$ equals $11.1$.}
\label{fig:CIFAR2}
\end{figure}
\begin{figure}[h!]
  \center
\includegraphics[trim=0cm 0cm 0cm 0cm,clip,width = 1.9in]{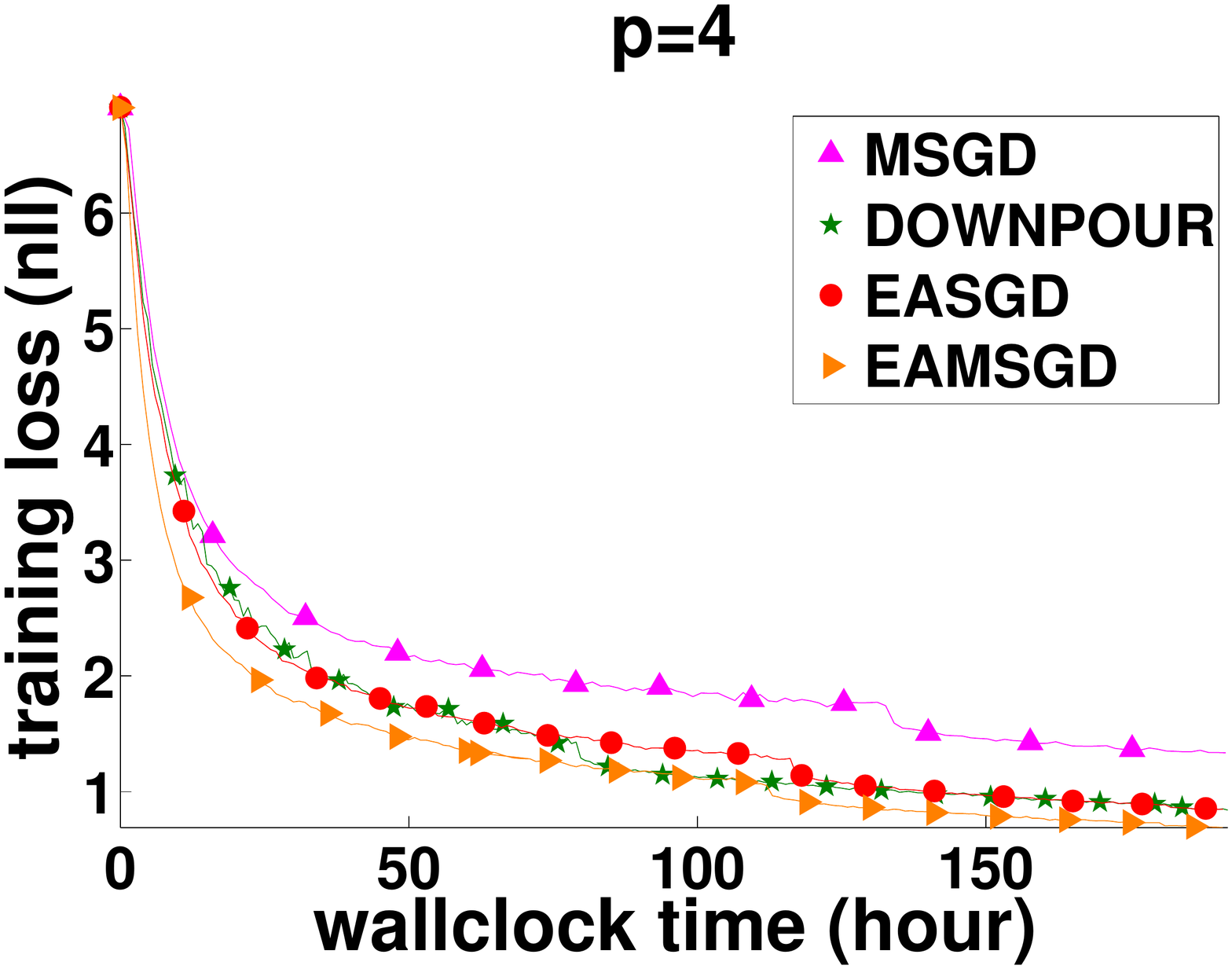}
\hspace{-0.3in}\includegraphics[trim=0cm 0cm 0cm 0cm,clip,width = 1.9in]{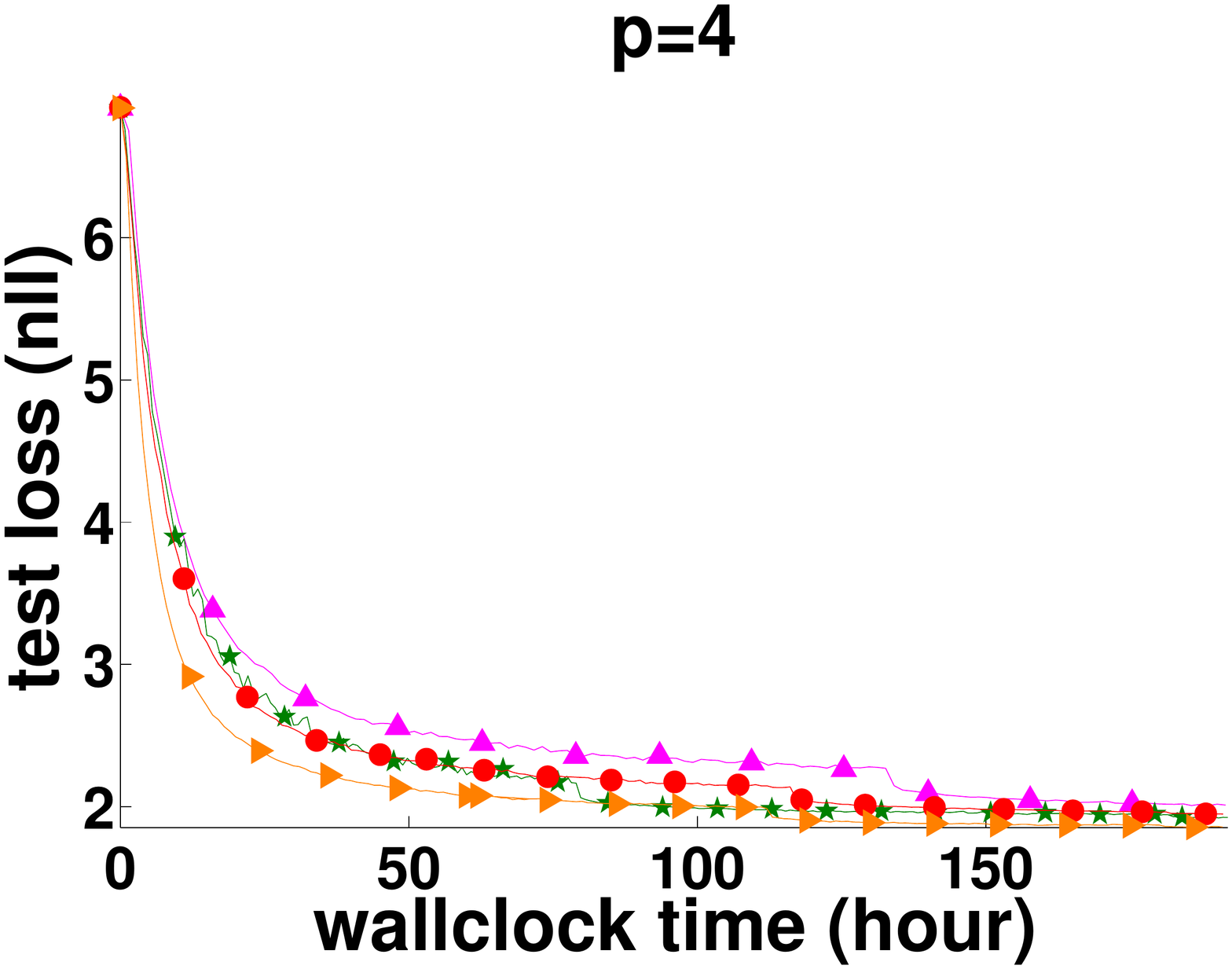} 
\hspace{-0.3in}\includegraphics[trim=0cm 0cm 0cm 0cm,clip,width = 1.9in]{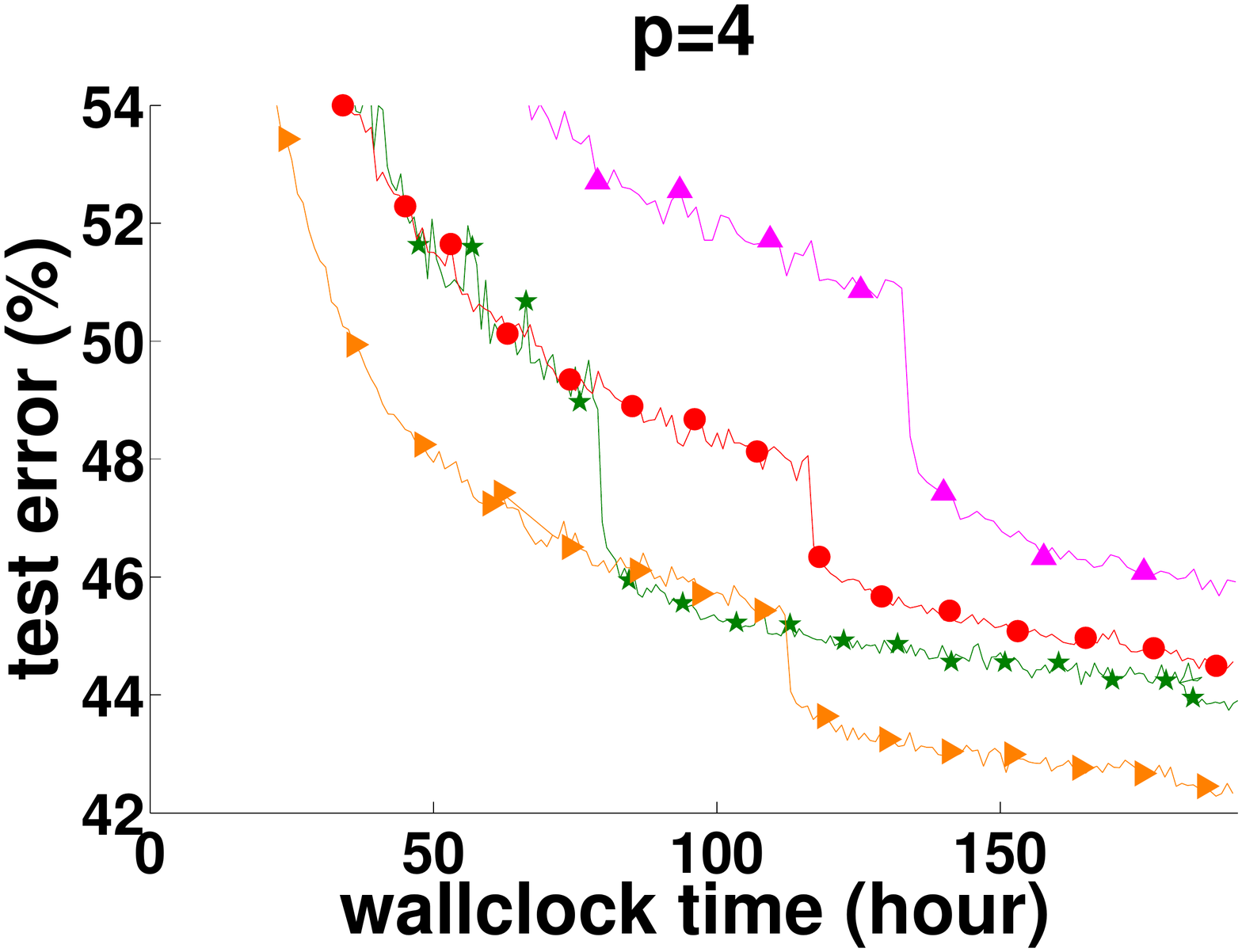}\\
\vspace{-0.13in}
\includegraphics[trim=0cm 0cm 0cm 0cm,clip,width = 1.9in]{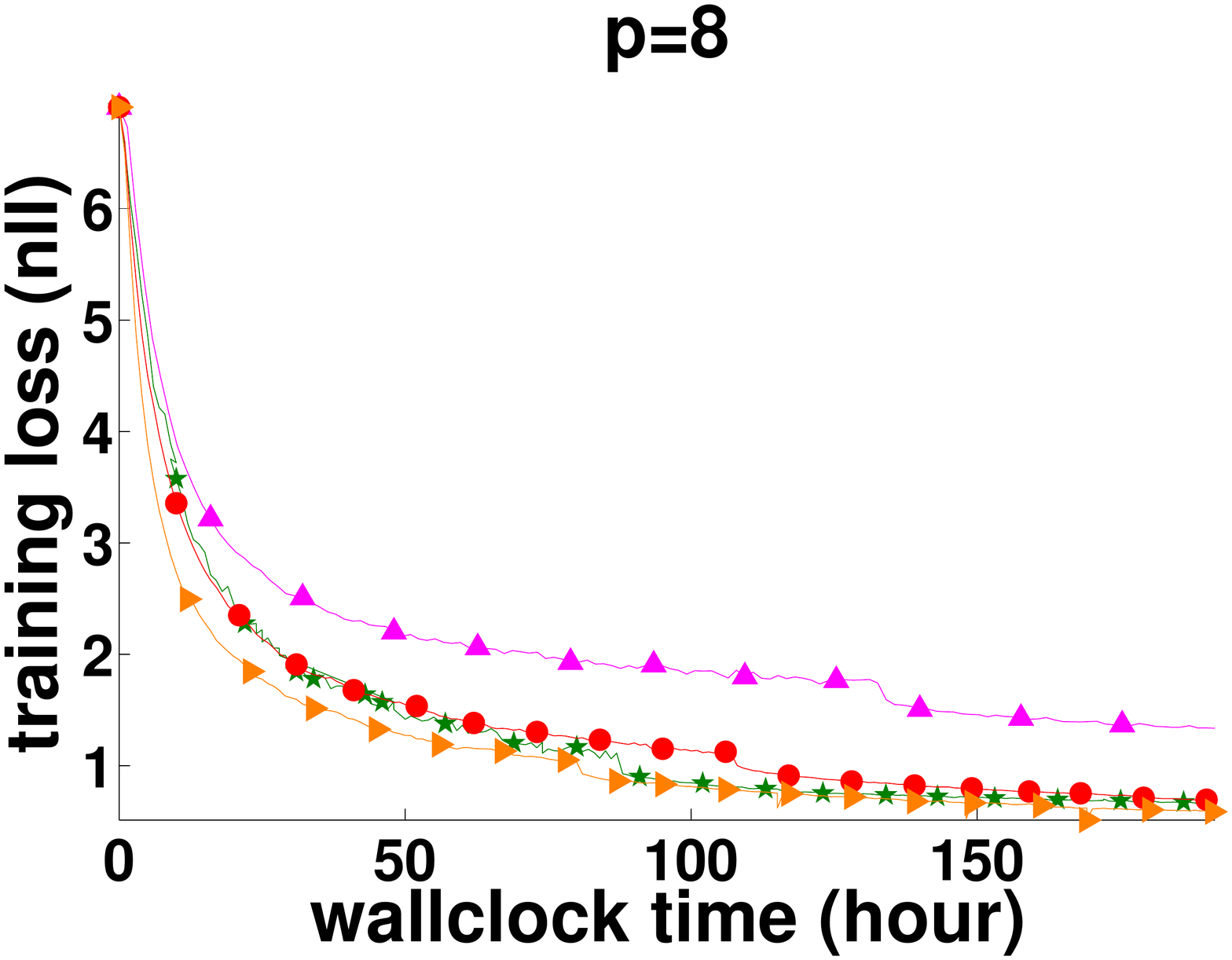}
\hspace{-0.3in}\includegraphics[trim=0cm 0cm 0cm 0cm,clip,width = 1.9in]{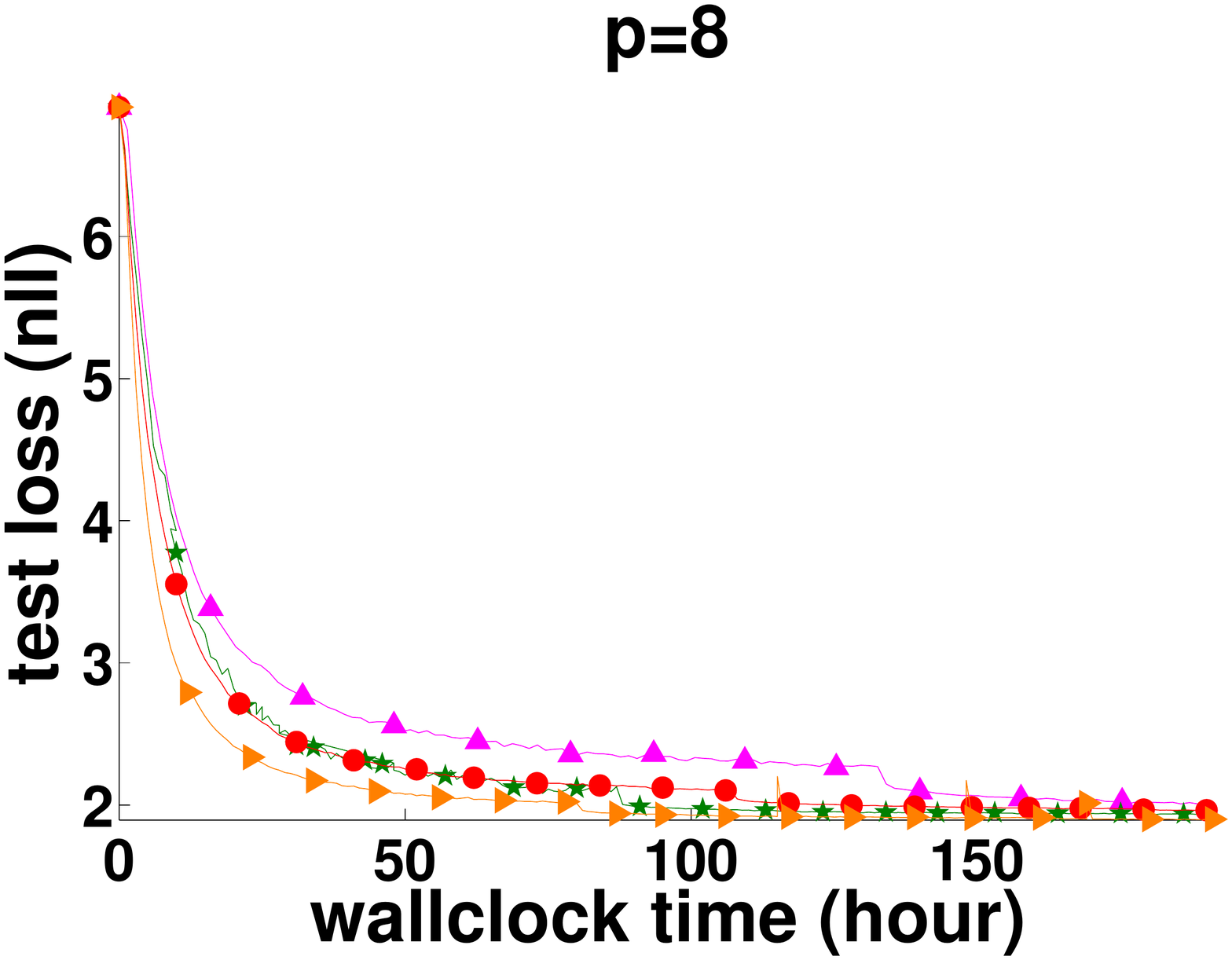} 
\hspace{-0.3in}\includegraphics[trim=0cm 0cm 0cm 0cm,clip,width = 1.9in]{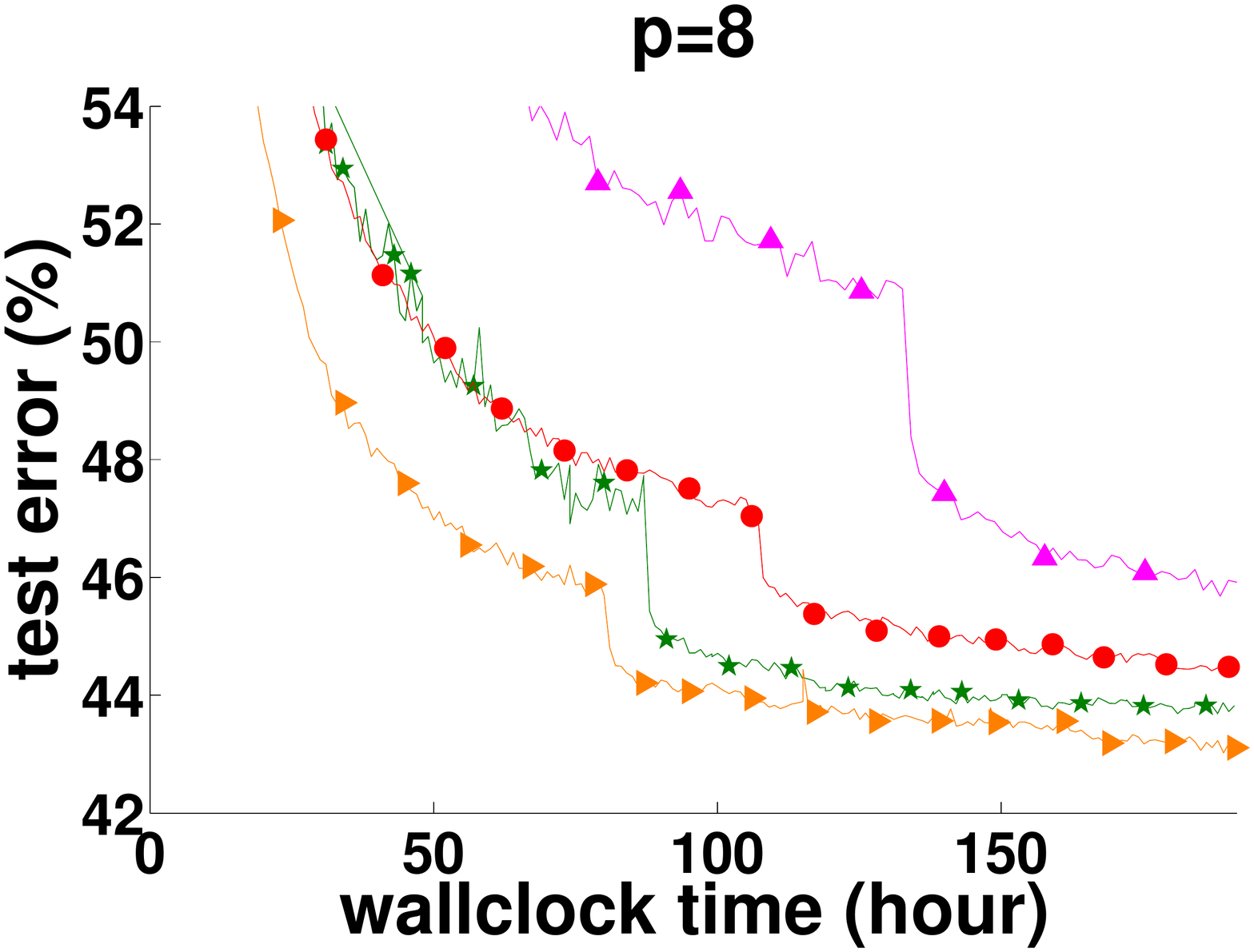}
\vspace{-0.13in}
\caption{Training and test loss and the test error for the center variable versus a wallclock time for different number of local workers $p$ (\textit{MSGD} uses $p=1$) on \textit{ImageNet} with the $11$-layer convolutional neural network. Initial learning rate is decreased twice, by a factor of $5$ and then $2$, when we observe that the online predictive loss~\cite{cesa2004generalization} stagnates. \textit{EAMSGD} achieves significant accelerations compared to other methods, e.g. the relative speed-up for $p=8$ (the best comparator method is then \textit{DOWNPOUR}) to achieve the test error $49\%$ equals $1.8$, and simultaneously it reduces the communication overhead (\textit{DOWNPOUR} uses communication period $\tau = 1$ and \textit{EAMSGD} uses $\tau = 10$).}
\label{fig:ImageNet}
\end{figure}

We next explore different number of local workers $p$ from the \redit{set} $p = \{4,8,16\}$ for the \textit{CIFAR} experiment, and $p = \{4,8\}$ for the \textit{ImageNet} experiment\footnote{For the \textit{ImageNet} experiment, the training loss is measured on a subset of the training data of size 50,000.}. For the \textit{ImageNet} experiment we report the results of one run with the best setting we have found. \textit{EASGD} and \textit{EAMSGD} were run with $\tau =10$ whereas \textit{DOWNPOUR} and \textit{MDOWNPOUR} were run with $\tau = 1$. The results are in Figure~\ref{fig:CIFAR2} and~\ref{fig:ImageNet}. For the \textit{CIFAR} experiment, it's noticeable that the lowest achievable test error by either \textit{EASGD} or \textit{EAMSGD} decreases with larger $p$. This can potentially be explained by the fact that larger $p$ allows for more exploration of the parameter space. In the Supplement, we discuss further the trade-off  between exploration and exploitation as a function of the learning rate (section~\ref{sec:deplr}) and the communication period (section~\ref{sec:deptau}). Finally, the results obtained for the \textit{ImageNet} experiment also shows the advantage of \textit{EAMSGD} over the competitor methods.


\section{Conclusion}
\label{sec:conclusion}
In this paper we describe a new algorithm called \textit{EASGD} and its variants for training deep neural networks in the stochastic setting when the computations are parallelized over multiple GPUs. Experiments demonstrate that this new algorithm quickly achieves improvement in test error compared to more common baseline approaches such as \textit{DOWNPOUR} and its variants. We show that our approach is very stable and plausible under communication constraints. We provide the stability analysis of the asynchronous \textit{EASGD} in the round-robin scheme, and show the theoretical advantage of the method over \textit{ADMM}. The different behavior of the \textit{EASGD} algorithm from its momentum-based variant \textit{EAMSGD} is intriguing and will be studied in future works.

\subsubsection*{Acknowledgments}
The authors thank R. Power, J. Li for implementation guidance, J. Bruna, O. Henaff, C. Farabet,
A. Szlam, Y. Bakhtin for helpful discussion, P. L. Combettes, S. Bengio and the referees for valuable feedback.

\newpage
\small{
\bibliography{nips2015}
\bibliographystyle{unsrtModif}
}


\newpage 
\normalsize

\vbox{\hsize\textwidth
\linewidth\hsize \vskip 0.1in \toptitlebar \centering
{\LARGE\bf Deep learning with Elastic Averaging SGD\\(Supplementary Material) \par}
\bottomtitlebar
\vskip 0.3in minus 0.1in}
\vspace{-0.3in}

\section{Additional theoretical results and proofs}

\subsection{Quadratic case}

We provide here the convergence analysis of the synchronous \textit{EASGD} algorithm with constant learning rate. The analysis is focused on the convergence of the center variable to the local optimum. We discuss one-dimensional quadratic case first, then the generalization to multi-dimensional setting (Lemma~\ref{lem:lemm3}) and finally to the strongly convex case (Theorem~\ref{thm:strongly}).

Our analysis in the quadratic case extends the analysis of ASGD in~\cite{polyak1992acceleration}. Assume each of the $p$ local workers $x^i_t \in \mathbb{R}^n$ observes a noisy gradient at time $t \ge 0$  of the linear form given in Equation~\ref{eq:gradlform}.
\begin{equation}
g_t^i(x^i_t) = A x^i_t - b - \xi^i_t, \quad i \in \{1,\ldots,p\},
\label{eq:gradlform}
\end{equation}
where the matrix $A$ is positive-definite (each eigenvalue is strictly positive)
and $\{ \xi^i_t
\}$'s are i.i.d. random variables, with zero mean and positive-definite covariance $\Sigma$. 
Let $x^\ast $ denote the optimum solution, where $x^\ast = A^{-1} b  \in \mathbb{R}^n$. In this section we analyze the behavior of the mean squared error (MSE) of the center variable
$\tilde{x}_t$, where this error is denoted as $\mathbb{E}[\norm{L2}{\tilde{x}_{t} - x^{*}}^2]$, as a function of $t$, $p$, $\eta$, $\alpha$ and $\beta$, where $\beta = p\alpha$. Note that the MSE error can be decomposed as (squared) bias and variance\footnote{In our notation, $\mathbb{V}$ denotes the variance.}: $\mathbb{E}[\norm{L2}{\tilde{x}_{t} - x^{*}}^2] = \norm{L2}{\mathbb{E} [\tilde{x}_{t} - x^{*}]}^2+ \mathbb{V}[\tilde{x}_{t} - x^{*}]$. For one-dimensional case ($n = 1$), we assume $A = h > 0$ and $\Sigma = \sigma^2 > 0$.

\begin{lemma}
\label{lem:lemm1}
Let $\tilde{x}_0$ and $\{x^i_0\}_{i=1,\ldots,p}$ be arbitrary constants, then
\begin{gather}
\mathbb{E} [\tilde{x}_{t} - x^\ast] =
\gamma^{t} (\tilde{x}_0 -  x^\ast) + \frac{\gamma^t-\phi^t}{\gamma-\phi} \alpha u_0,
\label{eq:lem1a} \\
\mathbb{V} [\tilde{x}_{t} - x^\ast] =
\frac{p^2 \alpha^2 \eta^2 }{(\gamma-\phi)^2}
\bigg(
\frac{\gamma^2-\gamma^{2t}}{1-\gamma^2}
+\frac{\phi^2-\phi^{2t}}{1-\phi^2}
-2 \frac{\gamma \phi - (\gamma \phi)^t}{1 - \gamma \phi}\bigg)
\frac{\sigma^2}{p}, \label{eq:lem1b}
\end{gather}
where $u_0 = \sum_{i=1}^{p} ( x^i_0 - x^\ast - \frac{\alpha}{1-p\alpha-\phi} (\cbar{x}_0-x^\ast))$, $a = \eta h + (p+1)\alpha$, $c^2 = \eta h p \alpha $, $\gamma = 1-\frac{a-\sqrt{a^2-4 c^2}}{2}$, and $\phi = 1-\frac{a+\sqrt{a^2-4
c^2}}{2}$.
\end{lemma}
It follows from Lemma~\ref{lem:lemm1} that for the center variable to be stable the following has to hold
\begin{equation}
-1 < \phi < \gamma < 1.
\label{eq:condition}
\end{equation} 
It can be verified that $\phi$ and $\gamma$ are the two zero-roots of the polynomial in $\lambda$: $\lambda^2 - (2-a)\lambda + (1-a+c^2)$. Recall that $\phi$ and $\lambda$ are the functions of $\eta$ and $\alpha$. Thus  (see proof in Section ~\ref{sec:condition11}) 
\begin{itemize}
\item $\gamma < 1$ iff $c^2 > 0$ (i.e. $\eta >0$ and $ \alpha > 0$).
\vspace{-0.05in}
\item $\phi > -1$ iff $(2-\eta h)(2-p \alpha) >2 \alpha$ and $(2-\eta h) +
(2-p\alpha) > \alpha $.
\vspace{-0.05in}
\item $\phi = \gamma$ iff $a^2 = 4 c^2$ (i.e. $\eta h= \alpha =0$).
\end{itemize}
\vspace{-0.05in}

The proof the above Lemma is based on the diagonalization of the linear gradient map (this map is symmetric due to the relation $\beta = p \alpha$). The stability analysis of the asynchronous \textit{EASGD} algorithm in the round-robin scheme is similar due to this elastic symmetry.  

\begin{proof}

Substituting the gradient from Equation~\ref{eq:gradlform} into the update rule used by each local worker in the synchronous \textit{EASGD} algorithm (Equation~\ref{eq:movingaverage1} and ~\ref{eq:movingaverage2}) we obtain
\begin{align}
x^i_{t+1} &= x^i_{t} - \eta ( A x^i_{t} - b - \xi^i_t) - \alpha (x^i_t - \tilde{x}_t),
\label{loc:local} \\
\tilde{x}_{t+1} &= \tilde{x}_t + \sum_{i=1}^p \alpha (x^i_t - \tilde{x}_t)
\label{loc:center},
\end{align}
where $\eta$ is the learning rate, and $\alpha$ is the moving rate. Recall that $\alpha = \eta \rho$ and $A=h$.

For the ease of notation we redefine $\cbar{x}_{t}$ and $x^i_t$ as follows:
\[\cbar{x}_{t} \triangleq \cbar{x}_{t} - x^\ast \text{\:\:\:and\:\:\:} x^i_t \triangleq x^i_t - x^\ast.
\]
We prove the lemma by explicitly solving the linear equations \ref{loc:local} and
\ref{loc:center}. Let $x_{t} = (x^1_t,\ldots,x^p_t,\tilde{x}_t)^T$. We rewrite the
recursive relation captured in Equation~\ref{loc:local} and~\ref{loc:center} as simply
\[x_{t+1} = M x_t + b_t,
\]
where the drift matrix $M$ is defined as
\begin{equation*}
M=
\begin{bmatrix}
1-\alpha-\eta h & 0 & ... & 0 & \alpha \\
0 & 1-\alpha-\eta h& 0 & ... & \alpha \\
... & 0 & ... &0 & ... \\
0 & ... &0 & 1-\alpha-\eta h  & \alpha \\
\alpha & \alpha & ... & \alpha & 1-p\alpha
\end{bmatrix},
\end{equation*}
and the (diffusion) vector $b_t = (\eta\xi^1_t,\ldots,\eta\xi^p_t,0)^T$. 

Note that one of the eigenvalues of matrix $M$, that we call $\phi$, satisfies $(1-\alpha-\eta h-\phi)(1-p\alpha-\phi) =p\alpha^2$. The corresponding eigenvector is $(1,1,\ldots,1,-\frac{p\alpha}{1-p\alpha-\phi})^T$. Let $u_t$ be the
projection of $x_{t}$ onto this eigenvector. Thus $u_t = \sum_{i=1}^{p} ( x^i_t - \frac{\alpha}{1-p\alpha-\phi} \cbar{x}_t)$. Let furthermore $\xi_t = \sum_{i=1}^p \xi_t^i$. Therefore we have
\begin{gather}
u_{t+1} = \phi u_t + \eta \xi_t. \label{eq:proj1}
\end{gather}
By combining Equation \ref{loc:center} and \ref{eq:proj1} as follows
 \begin{align*}
\tilde{x}_{t+1} &= \tilde{x}_t + \sum_{i=1}^p \alpha (x^i_t - \tilde{x}_t) = (1-p\alpha) \tilde{x}_t + \alpha (u_t + \frac{p\alpha}{1-p\alpha-\phi} \tilde{x}_t )\\
&= (1-p\alpha+\frac{p\alpha^2}{1-p\alpha-\phi}) \tilde{x}_t +  \alpha u_t = \gamma \cbar{x}_t + \alpha u_t,
\end{align*} 
where the last step results from the following relations: $\frac{p\alpha^2}{1-p\alpha-\phi} = 1-\alpha-\eta h-\phi $ and $\phi + \gamma = 1-\alpha-\eta h + 1-p\alpha $.
Thus we obtained
\begin{gather}
\cbar{x}_{t+1} = \gamma \cbar{x}_t + \alpha u_t. \label{eq:proj2}
\end{gather}

Based on Equation \ref{eq:proj1} and \ref{eq:proj2}, we can then expand $u_t$ and $\cbar{x}_t$
recursively,
\begin{gather}
u_{t+1} = \phi^{t+1} u_0 + \phi^{t} (\eta \xi_0) + \ldots + \phi^{0} (\eta \xi_t),
\label{eq:lem1c} \\
\cbar{x}_{t+1} = \gamma^{t+1} \cbar{x}_0 + \gamma^{t} (\alpha u_0) + \ldots +
\gamma^{0} (\alpha u_t). \label{eq:lem1d}
\end{gather}
Substituting $u_0, u_1, \dots, u_t$, each given through Equation \ref{eq:lem1c}, into Equation~\ref{eq:lem1d} we obtain
\begin{gather}
\cbar{x}_t = \gamma^t \cbar{x}_0 + 
            \frac{\gamma^t-\phi^t}{\gamma-\phi}  \alpha u_0 +
            \alpha \eta \sum_{l=1}^{t-1} \frac{\gamma^{t-l}-\phi^{t-l}}{\gamma-\phi} \xi_{l-1} .
\label{eq:lem1e}
\end{gather}

To be more specific, the Equation \ref{eq:lem1e} is obtained by integrating by parts,
\begin{align*}
\cbar{x}_{t+1} &= \gamma^{t+1} \cbar{x}_0 + \sum_{i=0}^t \gamma^{t-i} (\alpha u_i) \\
                          &= \gamma^{t+1} \cbar{x}_0 + 
\sum_{i=0}^t \gamma^{t-i} (\alpha (\phi^i u_0 + \sum_{l=0}^{i-1} \phi^{i-1-l} \eta \xi_l )) \\
			  &= \gamma^{t+1} \cbar{x}_0 + 
\sum_{i=0}^t \gamma^{t-i}  \phi^i (\alpha u_0) +
\sum_{l=0}^{t-1} \sum_{i=l+1}^{t} \gamma^{t-i} \phi^{i-1-l} (\alpha \eta \xi_l) \\
			  &=  \gamma^{t+1} \cbar{x}_0 + 
\frac{\gamma^{t+1}-\phi^{t+1}}{\gamma - \phi}  (\alpha u_0) + 
\sum_{l=0}^{t-1} \frac{\gamma^{t-l}-\phi^{t-l}}{\gamma-\phi} (\alpha \eta \xi_l).
\end{align*}
Since the random variables $\xi_l$ are i.i.d, we may sum the variance term by term as follows
\begin{align}
\sum_{l=0}^{t-1} \bigg( \frac{\gamma^{t-l}-\phi^{t-l}}{\gamma-\phi} \bigg)^2 &= 
\sum_{l=0}^{t-1} \frac{\gamma^{2(t-l)} - 2\gamma^{t-l}\phi^{t-l} +\phi^{2(t-l)}}{(\gamma-\phi)^2} \nonumber\\
&= \frac{1}{(\gamma-\phi)^2} \bigg( \frac{\gamma^2-\gamma^{2(t+1)}}{1-\gamma^2} 
- 2  \frac{\gamma \phi -  (\gamma \phi)^{t+1}}{1-\gamma \phi }
+ \frac{\phi^2-\phi^{2(t+1)}}{1-\phi^2}
\bigg).
\label{eq:new}
\end{align}

Note that $\mathbb{E}[\xi_t] = \sum_{i=1}^p \mathbb{E} [\xi_t^i] = 0$ and $\mathbb{V}[\xi_t] =
\sum_{i=1}^p \mathbb{V} [\xi_t^i] = p\sigma^2$. These two facts, the equality in Equation~\ref{eq:lem1e} and Equation~\ref{eq:new} can then be used to compute $\mathbb{E}[\tilde{x}_t]$ and $\mathbb{V}[\tilde{x}_t]$ as given in Equation~\ref{eq:lem1a} and~\ref{eq:lem1b} in Lemma~\ref{lem:lemm1}. 
\end{proof}

\subsubsection{Visualizing Lemma~\ref{lem:lemm1}}

\begin{figure}[htp!]
\begin{center}
\centerline{\includegraphics[trim=115 65 115
35,clip,width=1.0\columnwidth]{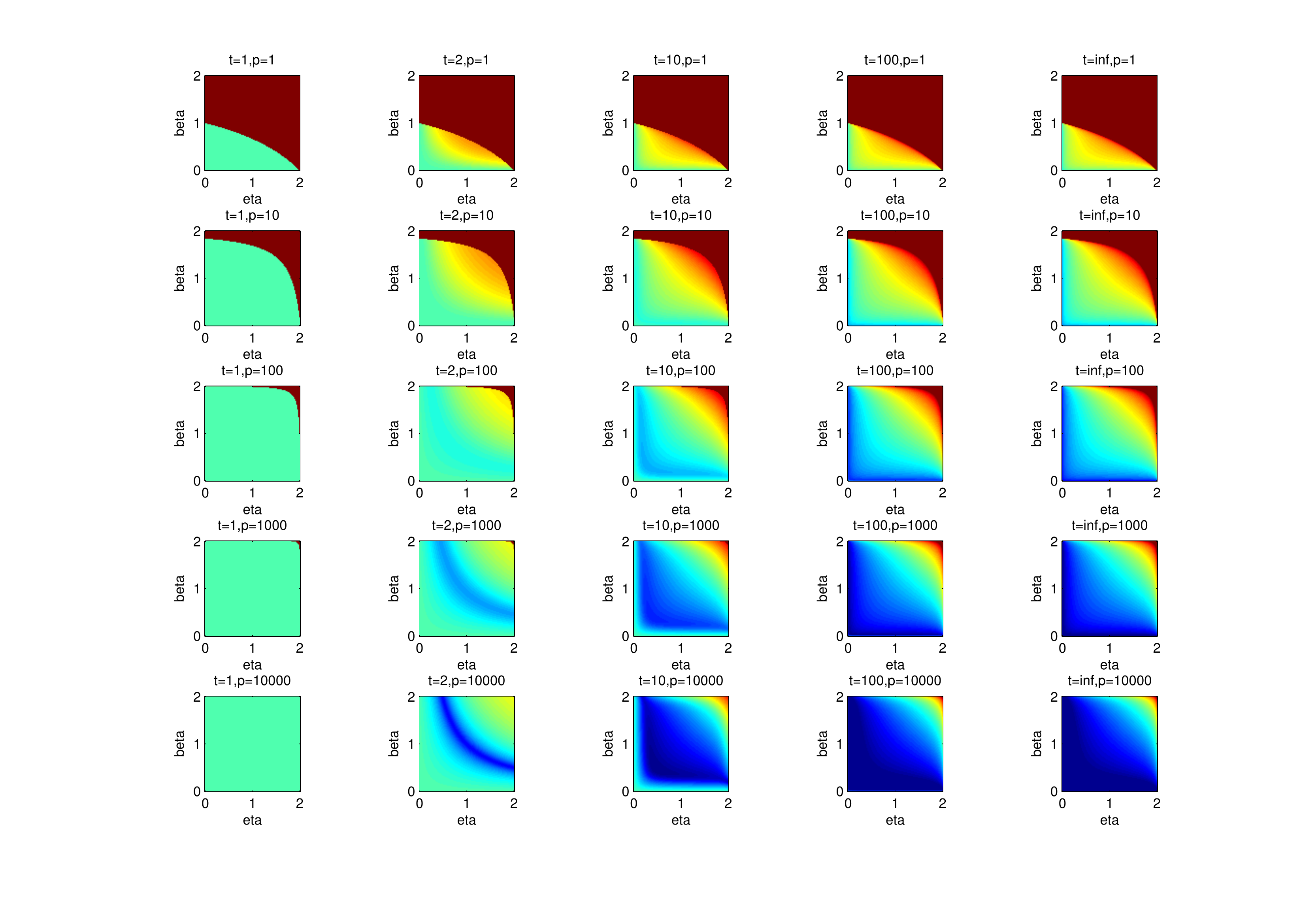}}
\caption{Theoretical mean squared error (MSE) of the center $\tilde{x}$ in the quadratic case, with
various choices of the learning rate $\eta$ (horizontal within each block), and the 
moving rate $\beta=p\alpha$ (vertical within each block),
the number of processors $p=\{1,10,100,1000,10000\}$ (vertical across blocks), and the time steps $t=\{1,2,10,100,\infty\}$ (horizontal across blocks). The MSE is plotted in log
scale, ranging from $10^{-3}$ to $10^{3}$ (from deep blue to red). The dark red (i.e. on the upper-right corners) indicates divergence.}
\label{fig:easgd_grid}
\end{center}
\end{figure}

 In Figure \ref{fig:easgd_grid}, we illustrate the dependence of MSE on $\beta$, $\eta$ and the number of processors $p$ over time $t$. We consider the large-noise setting where $\tilde{x}_0 = x^i_0= 1$, $h=1$ and $\sigma = 10$. The MSE error is color-coded such that the deep blue color corresponds to the MSE equal to $10^{-3}$, the green color corresponds to the MSE equal to $1$, the red color corresponds to MSE equal to $10^3$ and the dark red color corresponds to the divergence of algorithm \textit{EASGD} (condition in Equation~\ref{eq:condition} is then violated). The plot shows that
we can achieve significant variance reduction by increasing the number of local workers $p$. 
This effect is less sensitive to the choice of $\beta$ and $\eta$ for large $p$. 

\subsubsection{Condition in Equation~\ref{eq:condition}}
\label{sec:condition11}
We are going to show that
\begin{itemize}
\item $\gamma < 1$ iff $c^2 > 0$ (i.e. $\eta >0$ and $ \beta > 0$).
\item $\phi > -1$ iff $(2-\eta h)(2-\beta) >2\beta/p$ and $(2-\eta h) +
(2-\beta) > \beta/p $.
\item $\phi = \gamma$ iff $a^2 = 4 c^2$ (i.e. $\eta h= \beta =0$).
\end{itemize}

Recall that $a = \eta h + (p+1)\alpha$, $c^2 = \eta h p \alpha $, $\gamma = 1-\frac{a-\sqrt{a^2-4 c^2}}{2}$, $\phi = 1-\frac{a+\sqrt{a^2-4 c^2}}{2}$, and $\beta = p \alpha$. 
We have 
\begin{itemize}
\item 
$\gamma < 1 \Leftrightarrow \frac{a-\sqrt{a^2-4 c^2}}{2} > 0 \Leftrightarrow a > \sqrt{a^2-4 c^2} \Leftrightarrow a^2 > a^2-4 c^2 \Leftrightarrow c^2 > 0 $.
\item 
$\phi > -1 \Leftrightarrow 2 > \frac{a+\sqrt{a^2-4 c^2}}{2} \Leftrightarrow 4-a > \sqrt{a^2-4 c^2} 
\Leftrightarrow 4-a > 0, (4-a)^2 > a^2-4 c^2  \Leftrightarrow 4-a > 0, 4 - 2a + c^2 > 0
\Leftrightarrow  4 > \eta h + \beta + \alpha , 4 - 2(\eta h + \beta + \alpha) + \eta h \beta  > 0$.
\item $\phi = \gamma \Leftrightarrow \sqrt{a^2-4 c^2} = 0 \Leftrightarrow a^2 = 4c^2$.  
\end{itemize}

The next corollary is a consequence of Lemma~\ref{lem:lemm1}. As the number of workers $p$ grows, the averaging property of the \textit{EASGD} can be \redit{characterized} as follows

\begin{corollary}
Let the \textit{Elastic Averaging relation} $\beta = p \alpha$ and the condition \ref{eq:condition} hold, 
then
\begin{gather*}
\lim_{p \rightarrow \infty}\lim_{t \rightarrow \infty} p\mathbb{E} [ (\tilde{x}_{t} - x^{*})^2] = \frac{\beta \eta h}{(2-\beta)(2-\eta h)} \cdot
\frac{2 -\beta -\eta h + \beta \eta h}{\beta  + \eta h - \beta \eta h } \cdot
\frac{\sigma^2}{h^2}.
\end{gather*}
\label{cor:lemm1cor}
\end{corollary}

\begin{proof}

Note that when $\beta$ is fixed, $\lim_{p\rightarrow\infty}a = \eta h+ \beta$ and $c^2 = \eta h \beta$. Then $\lim_{p\rightarrow\infty} \phi = \min(1 - \beta,1 - \eta h)$ and $\lim_{p\rightarrow\infty} \gamma= \max(1 - \beta,1 - \eta h)$. Also note that using Lemma~\ref{lem:lemm1} we obtain
\begin{align*}
\lim_{t \rightarrow \infty} \mathbb{E} [(\tilde{x}_{t} - x^{*})^2] &=  
\frac{\beta^2 \eta^2}{(\gamma-\phi)^2} \bigg( \frac{\gamma^2}{1-\gamma^2} + 
\frac{\phi^2}{1-\phi^2} -  \frac{2 \gamma \phi}{1- \gamma \phi} \bigg) \frac{\sigma^2}{p} \\
&=
\frac{\beta^2 \eta^2}{(\gamma-\phi)^2} 
\bigg( \frac{\gamma^2(1-\phi^2)(1-\phi \gamma) + \phi^2(1-\gamma^2)(1-\phi \gamma) - 2\gamma \phi (1-\gamma^2)(1-\phi^2) }{(1-\gamma^2)(1-\phi^2)(1- \gamma \phi)} \bigg ) \frac{\sigma^2}{p} \\
&= \frac{\beta^2 \eta^2}{(\gamma-\phi)^2} 
\bigg( \frac{ (\gamma-\phi)^2(1+\gamma \phi) }{(1-\gamma^2)(1-\phi^2)(1- \gamma \phi)} \bigg ) \frac{\sigma^2}{p} \\
&= \frac{\beta ^2 \eta^2}{(1-\gamma^2)(1-\phi^2)} \cdot
\frac{1+\gamma \phi}{1-\gamma \phi} \cdot
\frac{\sigma^2}{p}.
\end{align*}

Corollary~\ref{cor:lemm1cor} is obtained by plugining in the limiting values of $\phi$ and $\gamma$.
\end{proof}

The crucial point of Corollary \ref{cor:lemm1cor} is that the MSE in the limit $t \rightarrow \infty$ is in the order of $1/p$ which implies that as the number of processors $p$ grows, the MSE will decrease for the \textit{EASGD} algorithm. Also note that the smaller the $\beta$ is (recall that $\beta = p\alpha = p\eta\rho$), the more exploration is allowed (small $\rho$) and simultaneously the smaller the MSE is.

\subsection{Generalization to multidimensional case}

The next lemma (Lemma~\ref{lem:lemm2}) shows that \textit{EASGD} algorithm achieves the highest possible rate of convergence when we consider the double averaging sequence (similarly to~\cite{polyak1992acceleration})
$\{z_1,z_2,\dots\}$  defined as below
\begin{equation}
z_{t+1} = \frac{1}{t+1} \sum_{k=0}^{t} \tilde{x}_k.
\label{eq:doubleseq}
\end{equation}

\begin{lemma}[Weak convergence]
\label{lem:lemm2}
If the condition in Equation~\ref{eq:condition} holds, then the
normalized double averaging sequence defined in Equation~\ref{eq:doubleseq} converges weakly to the
normal distribution with zero mean and variance $\sigma^2/p h^2$,
\begin{gather}
\sqrt{t} (z_{t} -  x^\ast) \rightharpoonup \mathcal{N}(0,\frac{\sigma^2}{p h^2}),
\quad  t\rightarrow \infty.
\end{gather}
\end{lemma}

\begin{proof}
As in the proof of Lemma~\ref{lem:lemm1}, for the ease of notation we redefine $\cbar{x}_{t}$ and $x^i_t$ as follows:
\[\cbar{x}_{t} \triangleq \cbar{x}_{t} - x^\ast \text{\:\:\:and\:\:\:} x^i_t \triangleq x^i_t - x^\ast.
\] 
Also recall that $\{ \xi^i_t \}$'s are \textit{i.i.d.} random variables (noise) with zero mean and the same covariance $\Sigma \succ 0$. We are interested in the asymptotic behavior of the double averaging sequence $\{z_1,z_2,\dots\}$ defined as
\begin{gather}
z_{t+1} = \frac{1}{t+1} \sum_{k=0}^{t} \cbar{x}_k.
\end{gather}

Recall the Equation~\ref{eq:lem1e} from the proof of Lemma~\ref{lem:lemm1} (for the convenience it is provided below):
\begin{gather*}
\cbar{x}_k = \gamma^k \cbar{x}_0 + \alpha u_0 \frac{\gamma^k-\phi^k}{\gamma-\phi}+\alpha \eta
            \sum_{l=1}^{k-1} \frac{\gamma^{k-l}-\phi^{k-l}}{\gamma-\phi} \xi_{l-1},
\end{gather*}
where $\xi_t = \sum_{i=1}^p \xi_t^i$.
Therefore
\begin{eqnarray*}
\sum_{k=0}^t \cbar{x}_k \!\!\!&=&\!\!\! \frac{1-\gamma^{t+1}}{1-\gamma} \cbar{x}_0 + \alpha u_0 \frac{1}{\gamma-\mu} \bigg(\frac{1-\gamma^{t+1}}{1-\gamma} - \frac{1-\phi^{t+1}}{1-\phi} \bigg) + \alpha \eta \sum_{l=1}^{t-1} \sum_{k=l+1}^{t} \frac{\gamma^{k-l}-\phi^{k-l}}{\gamma-\phi} \xi_{l-1} \\
&=&\!\!\! O(1) + \alpha \eta \sum_{l=1}^{t-1} \frac{1}{\gamma - \phi}
\bigg (\gamma \frac{1-\gamma^{t-l}}{1-\gamma} - \phi \frac{1-\phi^{t-l}}{1-\phi} \bigg) \xi_{l-1}
\end{eqnarray*}

Note that the only non-vanishing term (in weak convergence) of $1/\sqrt{t}\sum_{k=0}^t \cbar{x}_k$ as $t\rightarrow \infty$ is
\begin{align}
    \frac{1}{\sqrt{t}} \alpha \eta \sum_{l=1}^{t-1} \frac{1}{\gamma - \phi}
\bigg ( \frac{\gamma}{1-\gamma} - \frac{\phi}{1-\phi} \bigg) \xi_{l-1}.
\label{eq:normalized}
\end{align}
Also recall that $\mathbb{V}[\xi_{l-1}] = p \sigma^2$ and
\begin{align*}
     \frac{1}{\gamma - \phi}
\bigg ( \frac{\gamma}{1-\gamma} - \frac{\phi}{1-\phi} \bigg) = \frac{1}{(1-\gamma)(1-\phi) } = \frac{1}{\eta h p \alpha} .
\end{align*}
Therefore the expression in Equation~\ref{eq:normalized} is asymptotically normal with zero mean and variance $\sigma^2/p h^2$.
\end{proof}

The asymptotic variance in the Lemma~\ref{lem:lemm2} is optimal 
with any 
fixed $\eta$ and $\beta$ for which Equation~\ref{eq:condition} holds. The next lemma (Lemma~\ref{lem:lemm3}) extends the result in Lemma~\ref{lem:lemm2} to the multi-dimensional setting.

\begin{lemma}[Weak convergence]
\label{lem:lemm3}
Let $h$ denotes the largest eigenvalue of $A$. If $(2-\eta h)(2-\beta) >2 \beta/p$, $(2-\eta h) +
(2-\beta) > \beta/p $, $\eta>0$ and $\beta>0$,
then the
normalized double averaging sequence converges weakly to the
normal distribution with zero mean and the covariance matrix $V=A^{-1} \Sigma (A^{-1})^T$,
\begin{gather}
\sqrt{tp} (z_{t}- x^\ast) \rightharpoonup \mathcal{N}(0,V),
\quad  t\rightarrow \infty.
\end{gather}
\end{lemma}

\begin{proof}
Since $A$ is symmetric, one can use the proof technique of Lemma~\ref{lem:lemm2} to prove Lemma~\ref{lem:lemm3} by diagonalizing the matrix $A$. This diagonalization essentially generalizes Lemma~\ref{lem:lemm1} to the multidimensional case. We will not go into the details of this proof as we will provide a simpler way to look at the system. As in the proof of Lemma~\ref{lem:lemm1} and Lemma~\ref{lem:lemm2}, for the ease of notation we redefine $\cbar{x}_{t}$ and $x^i_t$ as follows:
\[\cbar{x}_{t} \triangleq \cbar{x}_{t} - x^\ast \text{\:\:\:and\:\:\:} x^i_t \triangleq x^i_t - x^\ast.
\]  
Let the spatial average of the local parameters at time $t$ be denoted as $y_t$ where
$y_t = \frac{1}{p} \sum_{i=1}^p x_t^i$, and let the average noise
be denoted as $\xi_t$, where $\xi_t =  \frac{1}{p} \sum_{i=1}^p \xi_t^i$. Equations~\ref{loc:local} and \ref{loc:center} can then be reduced to the following
\begin{eqnarray}
	y_{t+1}  &=& y_{t} - \eta ( Ay_t - \xi_t) 
+ \alpha ( \cbar{x}_t - y_t ), \label{eq:reduced1}\\
	\cbar{x}_{t+1} &=& \cbar{x}_t + \beta ( y_t - \cbar{x}_t).
\label{eq:reduced2}
\end{eqnarray}

We focus on the case where the learning rate $\eta$ and the moving rate $\alpha$ are kept constant over time\footnote{As a side note, notice that the center parameter $\cbar{x}_t$ is
tracking the spatial average $y_t$ of the local parameters with a
non-symmetric spring in Equation~\ref{eq:reduced1} and ~\ref{eq:reduced2}. To be more precise note that the update on $y_{t+1}$ contains $ ( \cbar{x}_t - y_t )$ scaled by $\alpha$, whereas the update on $\cbar{x}_{t+1}$ contains $-( \cbar{x}_t - y_t )$ scaled by $\beta$. Since $\alpha = \beta / p$ the impact of the center $\tilde{x}_{t+1}$ on the spatial local average $y_{t+1}$ becomes more negligible as $p$ grows.}. Recall $\beta = p \alpha$ and $\alpha = \eta \rho$.

Let's introduce the block notation $U_t = (y_t,\cbar{x}_t)$, $\Xi_t = (\eta \xi_t,0)$, $M = I - \eta L$ 
and 
\begin{align*}
L = 
\left( 
\begin{array}{cc}
 A + \frac{\alpha}{\eta} I  &  - \frac{\alpha}{\eta}I \\
-  \frac{\beta}{\eta} I &  \frac{\beta}{\eta} I  \end{array}
 \right).
\end{align*}
From Equations~\ref{eq:reduced1} and~\ref{eq:reduced2} it follows that $U_{t+1} = M U_t + \Xi_t$. Note that this linear system has a degenerate noise $\Xi_t$ which prevents us from directly applying results of~\cite{polyak1992acceleration}. Expanding this recursive relation and summing by parts, we have
\begin{eqnarray*}
\sum_{k=0}^t U_k &=& M^0 U_0 + \\
			         &&M^1 U_0 + M^0 \Xi_0+ \\
			         &&M^2 U_0 + M^1 \Xi_0+ M^0 \Xi_1 + \\
			         &&...  \\
			         &&M^t U_0 + M^{t-1} \Xi_0+ \cdots + M^0 \Xi_{t-1} .
\end{eqnarray*}

By Lemma~\ref{lem:additional}, $\|M\|_2 < 1$ and thus
\begin{align*}
	M^0 + M^1 + \cdots + M^t + \cdots = (I-M)^{-1} = \eta^{-1} L^{-1}.
\end{align*}

Since $A$ is invertible, we get
\begin{align*}
	L^{-1} = 
\left( 
\begin{array}{cc}
 A^{-1}  &  \frac{\alpha}{\beta} A^{-1} \\
 A^{-1}  &  \frac{\eta}{\beta}  + \frac{\alpha}{\beta} A^{-1}   \end{array}
 \right),
\end{align*}
thus
\begin{align*}
	\frac{1}{\sqrt{t}} \sum_{k=0}^t U_k = 
                \frac{1}{\sqrt{t}} U_0 + 
		\frac{1}{\sqrt{t}} \eta L^{-1} \sum_{k=1}^{t} \Xi_{k-1} - 
		\frac{1}{\sqrt{t}} \sum_{k=1}^{t} M^{k+1} \Xi_{k-1}.
\end{align*}

Note that the only non-vanishing term (in weak convergence) of $\frac{1}{\sqrt{t}} \sum_{k=0}^t U_k$ is
$\frac{1}{\sqrt{t}} (\eta L)^{-1} \sum_{k=1}^{t} \Xi_{k-1} $ thus
we have 
\begin{align}
\frac{1}{\sqrt{t}}	(\eta L)^{-1} \sum_{k=1}^{t} \Xi_{k-1} \rightharpoonup \mathcal{N}
\bigg( \left(  \begin{array}{c} 0  \\ 0  \end{array}  \right)  , \left(  \begin{array}{cc} V & V \\ V & V  \end{array}  \right) \bigg),
\end{align}
where $V=A^{-1} \Sigma (A^{-1})^T$.
\end{proof}

\begin{lemma}
If the following conditions hold:
\begin{eqnarray*}
(2-\eta h)(2-p\alpha) &>& 2\alpha\\
(2-\eta h) + (2-p\alpha) &>& \alpha\\
\eta &>& 0\\
\alpha &>& 0
\end{eqnarray*}
then $\|M\|_2 < 1$.
\label{lem:additional}
\end{lemma}
\begin{proof}
The eigenvalue $\lambda$ of $M$ and the (non-zero) eigenvector $(y,z)$ of $M$ satisfy
\begin{equation}
	M  \left(  \begin{array}{c} y  \\ z  \end{array}  \right) = 
\lambda   \left(  \begin{array}{c} y  \\ z  \end{array}  \right)  .
\label{eq:oneM}
\end{equation}
Recall that
\begin{align}
M =  I - \eta L  = 
\left( 
\begin{array}{cc}
 I - \eta A  -\alpha I  &  \alpha I \\
  \beta  I &  I-\beta I   \end{array}
 \right).
\label{eq:twoM}
\end{align}
From the Equations~\ref{eq:oneM} and~\ref{eq:twoM} we obtain
\begin{eqnarray}
\left\{
	\begin{array}{ll}
		y - \eta A y - \alpha y + \alpha z = \lambda y \\
		\beta y + (1-\beta) z = \lambda z
	\end{array} \right..
\label{eq:bundle}
\end{eqnarray}

Since $(y,z)$ is assumed to be non-zero, we can write $z = \beta y / (\lambda + \beta -1)$.
Then the Equation~\ref{eq:bundle} can be reduced to
\begin{eqnarray}
	\eta A y = (1-\alpha -\lambda) y + \frac{\alpha \beta}{\lambda + \beta - 1} y .
\label{eq:mapping}
\end{eqnarray}
Thus $y$ is the eigenvector of $A$. Let $\lambda_A$ be the eigenvalue of matrix $A$ such that $ A y = \lambda_A y$. Thus based on Equation~\ref{eq:mapping} it follows that
\begin{equation}
\eta \lambda_A = (1-\alpha -\lambda) +  \frac{\alpha \beta}{\lambda + \beta - 1}.
\label{eq:intermediate}
\end{equation}
Equation~\ref{eq:intermediate} is equivalent to
\begin{eqnarray}
	\lambda^2 - (2-a)\lambda + (1-a+c^2) = 0,
\end{eqnarray}
where $a = \eta  \lambda_A + (p+1)\alpha$, $c^2 = \eta  \lambda_A p \alpha $. It follows from the 
condition in Equation~\ref{eq:condition} that $-1 < \lambda < 1$ iff $\eta > 0$, $\beta > 0$, 
$(2-\eta  \lambda_A)(2-\beta) >2\beta/p$ and $(2-\eta  \lambda_A) +(2-\beta) > \beta/p $.
Let $h$ denote the maximum eigenvalue of $A$ and note that $2-\eta  \lambda_A   \ge 2-\eta h$. 
This implies that the condition of our lemma is sufficient.
\end{proof}

As in Lemma~\ref{lem:lemm2}, the asymptotic covariance in the Lemma~\ref{lem:lemm3} is optimal, i.e. meets the Fisher information lower-bound. The fact that this asymptotic covariance matrix $V$ does not contain any term involving $\rho$ is quite remarkable, since 
the penalty term $\rho$ does have an impact on the condition number of the Hessian in Equation \ref{eq:penalty}.

\subsection{Strongly convex case}

We now extend the above proof ideas to analyze the strongly convex case, in which the noisy gradient $g_t^i(x) = \nabla{F}(x) - \xi^i_t$ has the regularity that there exists some $0 < \mu \le L$, 
for which $ \mu \norm{L2}{x-y}^2 \le \scalprod{L2}{\nabla{F}(x)- \nabla{F}(y)}{x-y} \le L \norm{L2}{x-y}^2$ holds uniformly for any $x\in \mathbb{R}^d,y \in \mathbb{R}^d$. The noise $\{ \xi^i_t
\}$'s is assumed to be i.i.d. with zero mean and bounded variance $\mathbb{E}[\norm{L2}{\xi^i_t}^2] \le \sigma^2$. 

\begin{theorem}
Let $a_t = \mathbb{E} \norm{L2}{\frac{1}{p} \sum_{i=1}^p x^i_t -x^\ast}^2$, 
$b_t= \frac{1}{p} \sum_{i=1}^p \mathbb{E} \norm{L2}{ x^i_t -x^\ast}^2$,
$c_t = \mathbb{E} \norm{L2}{\tilde{x}_t-x^\ast}^2 $, $\gamma_1=2\eta \frac{\mu L}{\mu+L}$ and $\gamma_2 = 2 \eta L (1-\frac{2 \sqrt{\mu L}}{\mu+L})$. If $0 \le \eta \le \frac{2}{\mu+L}(1-\alpha)$, $0 \le \alpha<1$ and $0 \le  \beta \le 1$ then
\begin{gather*}
 \begin{pmatrix}
  a_{t+1} \\
  b_{t+1} \\
  c_{t+1} \\
 \end{pmatrix} \le
 \begin{pmatrix}
  1-\gamma_1-\gamma_2-\alpha & \gamma_2 & \alpha \\
  0 & 1-\gamma_1-\alpha & \alpha \\
  \beta & 0 & 1-\beta
 \end{pmatrix} 
 \begin{pmatrix}
  a_{t} \\
  b_{t} \\
  c_{t} \\
 \end{pmatrix}
+ 
 \begin{pmatrix}
  \eta^2 \frac{\sigma^2}{p} \\
    \eta^2  \sigma^2 \\
 0 \\
 \end{pmatrix}
.
\end{gather*}
\label{thm:strongly}
\end{theorem}

\begin{proof}
The idea of the proof is based on the point of view in Lemma~\ref{lem:lemm3}, i.e. how close 
the center variable $\tilde{x}_t$ is to the spatial average of the local variables $y_t = \frac{1}{p} \sum_{i=1}^p x^i_t$. To further simplify the notation, let the noisy gradient be $\nabla f^i_{t,\xi}=g_t^i(x_t^i) = \nabla{F}(x_t^i) - \xi^i_t$, and $\nabla f^i_{t} = \nabla{F}(x_t^i)$ be its deterministic part.
Then \textit{EASGD} updates can be rewritten as follows,
\begin{eqnarray}
	x_{t+1}^i &=& x_t^i - \eta \nabla f^i_{t,\xi} - \alpha(x^i_t - \tilde{x}_t), \label{eq:strong_updates1} \\
	\cbar{x}_{t+1} &=& \cbar{x}_t + \beta ( y_t - \cbar{x}_t). \label{eq:strong_updates2}
\end{eqnarray}
We have thus the update for the spatial average,
\begin{eqnarray}
	y_{t+1}  &=& y_{t} - \eta  \frac{1}{p} \sum_{i=1}^p \nabla f^i_{t,\xi} - \alpha (y_t - \tilde{x}_t). 
\label{eq:strong_updates3}
\end{eqnarray}
The idea of the proof is to bound the distance $\norm{L2}{\tilde{x}_t-x^\ast}^2$ through $\norm{L2}{y_t-x^\ast}^2$  and $\frac{1}{p} \sum_i^p \norm{L2}{x^i_t-x^\ast}^2$. W start from the following estimate for the strongly convex function~\cite{nesterov2004introductory},
\begin{eqnarray*}
	\scalprod{L2}{\nabla{F}(x)-\nabla{F}(y)}{x-y} \ge \frac{\mu L}{\mu+L} \norm{L2}{x-y}^2+
\frac{1}{\mu+L}  \norm{L2}{\nabla{F}(x)-\nabla{F}(y)}^2.
\end{eqnarray*}

Since $\nabla{f( x^\ast)}=0$, we have
\begin{eqnarray}
	 \scalprod{L2}{ \nabla f^i_{t}}{x^i_t - x^\ast} \ge  \frac{\mu L}{\mu+L} \norm{L2}{x^i_t - x^\ast}^2+ \frac{1}{\mu+L}  \norm{L2}{\nabla f^i_{t}}^2.  \label{eq:est0}
\end{eqnarray}

From Equation~\ref{eq:strong_updates1} the following relation holds,
\begin{eqnarray}
	\norm{L2}{x^i_{t+1}-x^\ast}^2 &=& \norm{L2}{x^i_{t}-x^\ast}^2 
+ \eta^2 \norm{L2}{ \nabla f^i_{t,\xi} } ^2  + \alpha^2   \norm{L2}{x^i_t - \tilde{x}_t} ^2  \nonumber \\
&-&2\eta \scalprod{L2}{ \nabla f^i_{t,\xi}}{x^i_t - x^\ast} 
-2 \alpha \scalprod{L2}{x^i_t - \tilde{x}_t}{x^i_t - x^\ast} \nonumber \\ 
&+&2 \eta \alpha \scalprod{L2}{ \nabla f^i_{t,\xi}}{ x^i_t - \tilde{x}_t} .
\label{eq:est1a}
\end{eqnarray}
By the cosine rule ($2 
\scalprod{L2}{a-b}{c-d} = \norm{L2}{a-d}^2 - \norm{L2}{a-c}^2 + \norm{L2}{c-b}^2 - \norm{L2}{d-b}^2 $), we have
\begin{eqnarray}
	2 \scalprod{L2}{x^i_t - \tilde{x}_t}{x^i_t - x^\ast} = \norm{L2}{x^i_t-x^\ast}^2 + 
		\norm{L2}{x^i_t-\tilde{x}_t}^2 - \norm{L2}{\tilde{x}_t - x^\ast}^2 .
\label{eq:est1b}
\end{eqnarray}
By the Cauchy-Schwarz inequality, we have 
\begin{eqnarray}
\scalprod{L2}{ \nabla f^i_{t}}{ x^i_t - \tilde{x}_t} \le \norm{L2}{ \nabla f^i_{t} } \norm{L2}{ x^i_t - \tilde{x}_t}.
\label{eq:est1c}
\end{eqnarray}
Combining the above estimates in Equations~\ref{eq:est0},~\ref{eq:est1a},~\ref{eq:est1b},~\ref{eq:est1c}, we obtain
\begin{eqnarray}
\norm{L2}{x^i_{t+1}-x^\ast}^2 &\le& \norm{L2}{x^i_{t}-x^\ast}^2 
+ \eta^2 \norm{L2}{ \nabla f^i_{t} -  \xi^i_t } ^2 + \alpha^2   \norm{L2}{x^i_t - \tilde{x}_t}^2 \nonumber  \\
&-&2\eta \bigg( \frac{\mu L}{\mu+L} \norm{L2}{x^i_t - x^\ast}^2+ \frac{1}{\mu+L} \norm{L2}{\nabla f^i_{t}}^2 \bigg) + 2\eta \scalprod{L2}{ \xi^i_t}{x^i_t - x^\ast} \nonumber \\
&-&\alpha \big (  \norm{L2}{x^i_t-x^\ast}^2 + 
			\norm{L2}{x^i_t-\tilde{x}_t}^2 - \norm{L2}{\tilde{x}_t - x^\ast}^2 \big) \nonumber \\ 
&+& 2 \eta \alpha  \norm{L2}{ \nabla f^i_{t} } \norm{L2}{ x^i_t - \tilde{x}_t} 
- 2\eta \alpha \scalprod{L2}{ \xi^i_t}{x^i_t-\tilde{x}_t} \label{eq:est1}.
\end{eqnarray}

Choosing $0 \le \alpha<1$, we can have this upper-bound for the terms $\alpha^2   \norm{L2}{x^i_t - \tilde{x}_t}^2  - \alpha \norm{L2}{x^i_t-\tilde{x}_t}^2 + 2 \eta \alpha  \norm{L2}{ \nabla f^i_{t} } \norm{L2}{ x^i_t - \tilde{x}_t} = -\alpha(1-\alpha)  \norm{L2}{x^i_t - \tilde{x}_t}^2 +2 \eta \alpha  \norm{L2}{ \nabla f^i_{t} } \norm{L2}{ x^i_t - \tilde{x}_t} \le \frac{\eta^2 \alpha }{1-\alpha}  \norm{L2}{ \nabla f^i_{t} }^2 $ by applying $-ax^2+bx \le \frac{b^2}{4a}$ with $x= \norm{L2}{ x^i_t - \tilde{x}_t}$. Thus we can further
bound Equation~\ref{eq:est1} with
\begin{eqnarray}
\norm{L2}{x^i_{t+1}-x^\ast}^2 &\le& (1-2\eta \frac{\mu L}{\mu + L} - \alpha) \norm{L2}{x^i_{t}-x^\ast}^2 + (\eta^2 + \frac{\eta^2 \alpha }{1-\alpha} -\frac{2 \eta}{\mu+L} ) \norm{L2}{ \nabla f^i_{t}} ^2 \nonumber \\
&-& 2 \eta^2 \scalprod{L2}{ \nabla f^i_{t} }{ \xi^i_t }
+2\eta \scalprod{L2}{ \xi^i_t }{ x^i_t - x^\ast} - 2\eta \alpha \scalprod{L2}{ \xi^i_t }{ x^i_t - \tilde{x}_t } \label{eq:est2a}  \\
&+&  \eta^2 \norm{L2}{  \xi^i_t }^2 + \alpha \norm{L2}{ \tilde{x}_t - x^\ast}^2  \label{eq:est2b}
\end{eqnarray}

As in Equation~\ref{eq:est2a} and ~\ref{eq:est2b}, the noise $\xi^i_t$ is zero mean ($\mathbb{E} \xi^i_t = 0$) and the variance of the noise $\xi^i_t$ is bounded  ($\mathbb{E} \norm{L2}{  \xi^i_t }^2  \le \sigma^2 $), if $\eta$ is \redit{chosen small enough such that}
$\eta^2 + \frac{\eta^2 \alpha }{1-\alpha} -\frac{2 \eta}{\mu+L} \le 0$, then
\begin{eqnarray}
\mathbb{E} \norm{L2}{x^i_{t+1}-x^\ast}^2  &\le& (1-2\eta \frac{\mu L}{\mu + L} - \alpha) 
\mathbb{E} \norm{L2}{x^i_{t}-x^\ast}^2 + \eta^2 \sigma^2 + \alpha \mathbb{E} \norm{L2}{  \tilde{x}_t -  x^\ast}^2 .
\label{eq:key1}
\end{eqnarray}

Now we apply similar idea to estimate $\norm{L2}{y_t-x^\ast}^2$. From Equation~\ref{eq:strong_updates3} the following relation holds,
\begin{eqnarray}
\norm{L2}{y_{t+1}-x^\ast}^2 &=& \norm{L2}{y_{t}-x^\ast}^2 
+ \eta^2 \norm{L2}{ \frac{1}{p} \sum_{i=1}^p \nabla f^i_{t,\xi}} ^2  + \alpha^2   \norm{L2}{y_t - \tilde{x}_t} ^2  \nonumber \\
&-&2\eta \scalprod{L2}{ \frac{1}{p} \sum_{i=1}^p \nabla f^i_{t,\xi} }{y_t - x^\ast} 
-2 \alpha \scalprod{L2}{y_t - \tilde{x}_t}{y_t - x^\ast} \nonumber \\ 
&+&2 \eta \alpha \scalprod{L2}{  \frac{1}{p} \sum_{i=1}^p \nabla f^i_{t,\xi} }{ y_t - \tilde{x}_t} .
\label{eq:est3}
\end{eqnarray}

By $\scalprod{L2}{\frac{1}{p} \sum_{i=1}^p a_i}{\frac{1}{p} \sum_{j=1}^p b_j}=\frac{1}{p} 
 \sum_{i=1}^p \scalprod{L2}{a_i}{b_i} -  \frac{1}{p^2}  \sum_{i>j}  \scalprod{L2}{a_i-a_j}{b_i-b_j}$,
we have
\begin{eqnarray}
\scalprod{L2}{ \frac{1}{p} \sum_{i=1}^p \nabla f^i_{t} }{y_t - x^\ast} = 
\frac{1}{p} 
 \sum_{i=1}^p \scalprod{L2}{ \nabla f^i_{t} }{x_t^i - x^\ast} -  \frac{1}{p^2}  \sum_{i>j}  \scalprod{L2}{ \nabla f^i_{t}  -  \nabla f^j_{t} }{x_t^i  - x_t^j }.
\label{eq:est3a}
\end{eqnarray}
By the cosine rule, we have 
\begin{eqnarray}
2 \scalprod{L2}{y_t - \tilde{x}_t}{y_t - x^\ast} = \norm{L2}{y_t-x^\ast}^2 + 
\norm{L2}{y_t-\tilde{x}_t}^2 - \norm{L2}{\tilde{x}_t - x^\ast}^2. 
\label{eq:est3b}
\end{eqnarray}

Denote $\xi_t = \frac{1}{p}  \sum_{i=1}^p  \xi^i_t $, we can rewrite Equation~\ref{eq:est3} as
\begin{eqnarray}
\norm{L2}{y_{t+1}-x^\ast}^2 &=& \norm{L2}{y_{t}-x^\ast}^2 
+ \eta^2 \norm{L2}{ \frac{1}{p} \sum_{i=1}^p \nabla f^i_{t} - \xi_t } ^2  + \alpha^2   \norm{L2}{y_t - \tilde{x}_t} ^2  \nonumber \\
&-&2\eta \scalprod{L2}{ \frac{1}{p} \sum_{i=1}^p \nabla f^i_{t} - \xi_t }{y_t - x^\ast} 
-2 \alpha \scalprod{L2}{y_t - \tilde{x}_t}{y_t - x^\ast} \nonumber \\ 
&+&2 \eta \alpha \scalprod{L2}{  \frac{1}{p} \sum_{i=1}^p \nabla f^i_{t} - \xi_t }{ y_t - \tilde{x}_t} .
\label{eq:est3c}
\end{eqnarray}

By combining the above Equations~\ref{eq:est3a},~\ref{eq:est3b} with ~\ref{eq:est3c}, we obtain
\begin{eqnarray}
\norm{L2}{y_{t+1}-x^\ast}^2 &=& \norm{L2}{y_{t}-x^\ast}^2 
+ \eta^2 \norm{L2}{ \frac{1}{p} \sum_{i=1}^p \nabla f^i_{t} - \xi_t } ^2  + \alpha^2   \norm{L2}{y_t - \tilde{x}_t} ^2  \nonumber \\
&-& 2\eta \bigg(\frac{1}{p} 
\sum_{i=1}^p \scalprod{L2}{ \nabla f^i_{t} }{x_t^i - x^\ast} -  \frac{1}{p^2}  \sum_{i>j}  \scalprod{L2}{ \nabla f^i_{t}  -  \nabla f^j_{t} }{x_t^i  - x_t^j } \bigg) \\
&+&  2 \eta\scalprod{L2}{  \xi_t }{y_t - x^\ast} 
- \alpha ( \norm{L2}{y_t-x^\ast}^2 + 
\norm{L2}{y_t-\tilde{x}_t}^2 - \norm{L2}{\tilde{x}_t - x^\ast}^2 ) \nonumber \\ 
&+&2 \eta \alpha \scalprod{L2}{  \frac{1}{p} \sum_{i=1}^p \nabla f^i_{t} - \xi_t }{ y_t - \tilde{x}_t} .
\label{eq:est3d}
\end{eqnarray}
Thus it follows from Equation~\ref{eq:est0} and~\ref{eq:est3d} that
\begin{eqnarray}
\norm{L2}{y_{t+1}-x^\ast}^2 &\le& \norm{L2}{y_{t}-x^\ast}^2 
+ \eta^2 \norm{L2}{ \frac{1}{p} \sum_{i=1}^p \nabla f^i_{t} - \xi_t } ^2  + \alpha^2   \norm{L2}{y_t - \tilde{x}_t} ^2  \nonumber \\
&-& 2\eta  \frac{1}{p} 
\sum_{i=1}^p \bigg( \frac{\mu L}{\mu+L} \norm{L2}{x^i_t - x^\ast}^2+ \frac{1}{\mu+L} \norm{L2}{\nabla f^i_{t}}^2 \bigg) \nonumber \\
&+& 2\eta   \frac{1}{p^2}  \sum_{i>j}  \scalprod{L2}{ \nabla f^i_{t}  -  \nabla f^j_{t} }{x_t^i  - x_t^j }  \nonumber \\
&+& 2 \eta \scalprod{L2}{  \xi_t }{y_t - x^\ast} 
- \alpha ( \norm{L2}{y_t-x^\ast}^2 + 
\norm{L2}{y_t-\tilde{x}_t}^2 - \norm{L2}{\tilde{x}_t - x^\ast}^2 ) \nonumber \\ 
&+&2 \eta \alpha \scalprod{L2}{  \frac{1}{p} \sum_{i=1}^p \nabla f^i_{t} - \xi_t }{ y_t - \tilde{x}_t} .
\label{eq:est3e}
\end{eqnarray}

Recall $y_t = \frac{1}{p} \sum_{i=1}^p x^i_t$, we have the following bias-variance relation,
\begin{eqnarray}
\frac{1}{p}  \sum_{i=1}^p \norm{L2}{x^i_t - x^\ast}^2 &=& \frac{1}{p}  \sum_{i=1}^p \norm{L2}{x^i_t - y_t}^2 +  \norm{L2}{ y_t - x^\ast}^2 = 
 \frac{1}{p^2}  \sum_{i>j} \norm{L2}{x^i_t - x^j_t}^2 +  \norm{L2}{ y_t - x^\ast}^2, 
\nonumber \\
\frac{1}{p}  \sum_{i=1}^p \norm{L2}{  \nabla f^i_{t}}^2 &=& 
 \frac{1}{p^2}  \sum_{i>j} \norm{L2}{ \nabla f^i_{t} -  \nabla f^j_{t}}^2 +  \norm{L2}{ \frac{1}{p} \sum_{i=1}^p \nabla f^i_{t} }^2.
\label{eq:est3f}
\end{eqnarray}

By the Cauchy-Schwarz inequality, we have 
\begin{eqnarray}
	\frac{\mu L}{\mu+L}  \norm{L2}{x^i_t - x^j_t}^2 
+  \frac{1}{\mu+L} \norm{L2}{ \nabla f^i_{t} -  \nabla f^j_{t}}^2 \ge
\frac{2\sqrt{\mu L}}{\mu+L}  \scalprod{L2}{ \nabla f^i_{t}  -  \nabla f^j_{t} }{x_t^i  - x_t^j } . 
\label{eq:est3g}
\end{eqnarray}

Combining the above estimates in Equations~\ref{eq:est3e},~\ref{eq:est3f},~\ref{eq:est3g}, we obtain
\begin{eqnarray}
\norm{L2}{y_{t+1}-x^\ast}^2 &\le& \norm{L2}{y_{t}-x^\ast}^2 
+ \eta^2 \norm{L2}{ \frac{1}{p} \sum_{i=1}^p \nabla f^i_{t} - \xi_t } ^2  + \alpha^2   \norm{L2}{y_t - \tilde{x}_t} ^2  \nonumber \\
&-& 2\eta \bigg( \frac{\mu L}{\mu+L} \norm{L2}{y_t - x^\ast}^2+ \frac{1}{\mu+L}   \norm{L2}{ \frac{1}{p} \sum_{i=1}^p \nabla f^i_{t} }^2 \bigg)  \nonumber \\
&+& 2 \eta  \bigg(1-\frac{2\sqrt{\mu L}}{\mu+L}\bigg)  \frac{1}{p^2}  \sum_{i>j}  \scalprod{L2}{ \nabla f^i_{t}  -  \nabla f^j_{t} }{x_t^i  - x_t^j } \nonumber \\
&+& 2 \eta \scalprod{L2}{  \xi_t }{y_t - x^\ast} 
- \alpha ( \norm{L2}{y_t-x^\ast}^2 + 
\norm{L2}{y_t-\tilde{x}_t}^2 - \norm{L2}{\tilde{x}_t - x^\ast}^2 ) \nonumber \\ 
&+&2 \eta \alpha \scalprod{L2}{  \frac{1}{p} \sum_{i=1}^p \nabla f^i_{t} - \xi_t }{ y_t - \tilde{x}_t} .
\label{eq:est4}
\end{eqnarray}

Similarly if $0 \le \alpha<1$, we can have this upper-bound for the terms $\alpha^2  \norm{L2}{y_t - \tilde{x}_t}^2  - \alpha \norm{L2}{y_t-\tilde{x}_t}^2 + 2 \eta \alpha  \norm{L2}{  \frac{1}{p} \sum_{i=1}^p \nabla f^i_{t} } \norm{L2}{y_t - \tilde{x}_t} \le \frac{\eta^2 \alpha }{1-\alpha}  \norm{L2}{  \frac{1}{p} \sum_{i=1}^p \nabla f^i_{t}}^2 $ by applying $-ax^2+bx \le \frac{b^2}{4a}$ with $x= \norm{L2}{ y_t - \tilde{x}_t}$. Thus we have the following bound for the Equation~\ref{eq:est4}
\begin{eqnarray}
\norm{L2}{y_{t+1}-x^\ast}^2 &\le& 
 (1-2\eta \frac{\mu L}{\mu + L} - \alpha) \norm{L2}{y_{t}-x^\ast}^2 
+  (\eta^2 + \frac{\eta^2 \alpha }{1-\alpha} -\frac{2 \eta}{\mu+L} )  
\norm{L2}{ \frac{1}{p} \sum_{i=1}^p \nabla f^i_{t} } ^2 \nonumber \\
&-& 2 \eta^2 \scalprod{L2}{  \frac{1}{p} \sum_{i=1}^p \nabla f^i_{t} }{ \xi_t }
+2\eta \scalprod{L2}{ \xi_t }{y_t - x^\ast} - 2\eta \alpha \scalprod{L2}{ \xi_t }{y_t - \tilde{x}_t } \nonumber \\
&+& 2 \eta  \bigg(1-\frac{2\sqrt{\mu L}}{\mu+L}\bigg)  \frac{1}{p^2}  \sum_{i>j}  \scalprod{L2}{ \nabla f^i_{t}  -  \nabla f^j_{t} }{x_t^i  - x_t^j } \nonumber \\
&+&  \eta^2 \norm{L2}{\xi_t }^2 + \alpha \norm{L2}{ \tilde{x}_t - x^\ast}^2.
\label{eq:est5}
\end{eqnarray}
Since $ \frac{2\sqrt{\mu L}}{\mu+L} \le 1 $, we need also bound the non-linear term
$\scalprod{L2}{ \nabla f^i_{t}  -  \nabla f^j_{t} }{x_t^i  - x_t^j } \le L  \norm{L2}{ x_t^i  - x_t^j}^2$.
Recall the bias-variance relation
$\frac{1}{p}  \sum_{i=1}^p \norm{L2}{x^i_t - x^\ast}^2 = 
 \frac{1}{p^2}  \sum_{i>j} \norm{L2}{x^i_t - x^j_t}^2 +  \norm{L2}{ y_t - x^\ast}^2$.
The key observation is that if $\frac{1}{p}  \sum_{i=1}^p \norm{L2}{x^i_t - x^\ast}^2$ remains bounded,
then larger variance $  \sum_{i>j} \norm{L2}{x^i_t - x^j_t}^2$ implies smaller bias $ \norm{L2}{ y_t - x^\ast}^2$. Thus this non-linear term can be compensated.

Again choose $\eta$ \redit{small enough such that}
$\eta^2 + \frac{\eta^2 \alpha }{1-\alpha} -\frac{2 \eta}{\mu+L}  \le 0$ and take expectation
in Equation~\ref{eq:est5},
\begin{eqnarray}
\mathbb{E} \norm{L2}{y_{t+1}-x^\ast}^2 &\le& 
 (1-2\eta \frac{\mu L}{\mu + L} - \alpha) \mathbb{E} \norm{L2}{y_{t}-x^\ast}^2  \nonumber \\
&+& 2 \eta  L \bigg(1-\frac{2\sqrt{\mu L}}{\mu+L}\bigg)   \bigg( \frac{1}{p} \sum_{i=1}^p \mathbb{E}  \norm{L2}{x^i_t - x^\ast}^2  -  \mathbb{E} \norm{L2}{y_{t}-x^\ast}^2 \bigg) \nonumber \\
&+&  \eta^2 \frac{\sigma^2}{p} + \alpha \mathbb{E} \norm{L2}{ \tilde{x}_t - x^\ast}^2.
\label{eq:key2}
\end{eqnarray}

As for the center variable in Equation~\ref{eq:strong_updates2}, we apply simply the convexity of the norm $\norm{L2}{\cdot}^2$ to obtain
\begin{eqnarray}
\norm{L2}{\tilde{x}_{t+1} - x^\ast}^2 \le (1-\beta) \norm{L2}{\cbar{x}_t -  x^\ast}^2  + \beta  \norm{L2}{ y_t - x^\ast}^2 .
\label{eq:key3}
\end{eqnarray}

Combing the estimates from Equations~\ref{eq:key1},~\ref{eq:key2},~\ref{eq:key3}, and denote
$a_t = \mathbb{E} \norm{L2}{y_t -x^\ast}^2 $,
$b_t= \frac{1}{p} \sum_{i=1}^p \mathbb{E} \norm{L2}{ x^i_t -x^\ast}^2$,
$c_t = \mathbb{E} \norm{L2}{\tilde{x}_t-x^\ast}^2 $, $\gamma_1=2\eta \frac{\mu L}{\mu+L}$, 
$\gamma_2 = 2 \eta L (1-\frac{2 \sqrt{\mu L}}{\mu+L})$, then
\begin{gather*}
 \begin{pmatrix}
  a_{t+1} \\
  b_{t+1} \\
  c_{t+1} \\
 \end{pmatrix} \le
 \begin{pmatrix}
  1-\gamma_1-\gamma_2-\alpha & \gamma_2 & \alpha \\
  0 & 1-\gamma_1-\alpha & \alpha \\
  \beta & 0 & 1-\beta
 \end{pmatrix} 
 \begin{pmatrix}
  a_{t} \\
  b_{t} \\
  c_{t} \\
 \end{pmatrix}
+ 
 \begin{pmatrix}
  \eta^2 \frac{\sigma^2}{p} \\
    \eta^2  \sigma^2 \\
 0 \\
 \end{pmatrix}
,
\end{gather*}
as long as $0\le \beta \le 1$, $0 \le \alpha < 1$ and $\eta^2 + \frac{\eta^2 \alpha }{1-\alpha} -\frac{2 \eta}{\mu+L}  \le 0$, i.e. $0  \le \eta \le \frac{2}{\mu+L}(1-\alpha).$
\end{proof}

\section{Additional pseudo-codes of the algorithms}

\subsection{DOWNPOUR pseudo-code} 

Algorithm~\ref{alg:asyncD} captures the pseudo-code of the implementation of the DOWNPOUR used in this paper.

\RestyleAlgo{boxruled}
\begin{algorithm}[h]
\caption{DOWNPOUR: Processing by worker $i$ and the master}
\label{alg:asyncD}
\begin{tabular}{l}
\textbf{Input:} learning rate $\eta$, communication period $\tau \in \mathbb{N}$\\
 \textbf{Initialize:} $\tilde{x}$ is initialized randomly, $x^i = \tilde{x}$, $v^i = 0$, $t^i = 0\:\:\:\:\:\:\:\:\:\:\:\:\:\:\:\:\:\:\:\:\:\:\:\:\:\:\:\:\:\:\:\:\:\:\:\:\:\:\:\:\:\:\:\:\:\:\:\:\:\:\:\:\:\:\:\:\:\:\:\:\:\:\:\:\:\:\:$\\
\hline
\textbf{Repeat}\\
\:\:\:\:\:\textbf{if} ($\tau $ divides $t^i $) \textbf{then}\\
\:\:\:\:\:\:\:\:\:\:\:$\tilde{x} \:\:\:\!\!\leftarrow \tilde{x} \:\!+ v^i $\\
\:\:\:\:\:\:\:\:\:\:\:$x^i \leftarrow \tilde{x}$\\
\:\:\:\:\:\:\:\:\:\:\:$v^i \leftarrow 0$\\
\:\:\:\:\:\textbf{end}\\
\:\:\:\:\:$ x^i \leftarrow x^i - \eta g^i_{t^i}(x^i)$\\
\:\:\:\:\:$ v^i \leftarrow v^i - \eta g^i_{t^i}(x^i)$\\
\:\:\:\:\:$t^i\:\:\:\:\!\!\!\!\leftarrow \:\!t^i \:\:\!\!+ 1$\\
\textbf{Until forever}
\end{tabular}
\end{algorithm}

\subsection{MDOWNPOUR pseudo-code} 

Algorithms~\ref{alg:asyncMDworker} and~\ref{alg:asyncMDmaster} capture the pseudo-codes of the implementation of momentum DOWNPOUR (MDOWNPOUR) used in this paper. Algorithm~\ref{alg:asyncMDworker} shows the behavior of each local worker and Algorithm~\ref{alg:asyncMDmaster} shows the behavior of the master.

\RestyleAlgo{boxruled}
\begin{algorithm}[h]
\caption{MDOWNPOUR: Processing by worker $i$}
\label{alg:asyncMDworker}
\begin{tabular}{l}
 \textbf{Initialize:} $x^i = \tilde{x}\:\:\:\:\:\:\:\:\:\:\:\:\:\:\:\:\:\:\:\:\:\:\:\:\:\:\:\:\:\:\:\:\:\:\:\:\:\:\:\:\:\:\:\:\:\:\:\:\:\:\:\:\:\:\:\:\:\:\:\:\:\:\:\:\:\:\:\:\:\:\:\:\:\:\:\:\:\:\:\:\:\:\:\:\:\:\:\:\:\:\:\:\:\:\:\:\:\:\:\:\:\:\:\:\:\:\:\:\:\:\:\:\:\:\:\:\:\:\:\:\:\:\:\:\:\:\:\:\:\:\:\:\:\:\:\:\:\:\:\:\:$\\
\hline
\textbf{Repeat}\\
\:\:\:\:\:\textsf{Receive} $\tilde{x}$ \textsf{from the master}: $x^i \leftarrow \tilde{x}$\\
\:\:\:\:\:\textsf{Compute gradient} $g^i = g^i(x^i)$\\
\:\:\:\:\:\textsf{Send} $g^i$ \textsf{to the master}\\
\textbf{Until forever}
\end{tabular}
\end{algorithm}

\RestyleAlgo{boxruled}
\begin{algorithm}[h]
\caption{MDOWNPOUR: Processing by the master}
\label{alg:asyncMDmaster}
\begin{tabular}{l}
\textbf{Input:} learning rate $\eta$, momentum term $\delta$\\
 \textbf{Initialize:} $\tilde{x}$ is initialized randomly, $v^i = 0$, $\:\:\:\:\:\:\:\:\:\:\:\:\:\:\:\:\:\:\:\:\:\:\:\:\:\:\:\:\:\:\:\:\:\:\:\:\:\:\:\:\:\:\:\:\:\:\:\:\:\:\:\:\:\:\:\:\:\:\:\:\:\:\:\:\:\:\:\:\:\:\:\:\:\:\:\:\:\:\:\:\:\:\:\:\:\:\:\:\:\:\:\:\:$\\
\hline
\textbf{Repeat}\\
\:\:\:\:\:\textsf{Receive} $g^i$\\
\:\:\:\:\:$v \leftarrow \delta v - \eta g^i$\\
\:\:\:\:\:$\tilde{x} \leftarrow \tilde{x} + \delta v$\\
\textbf{Until forever}
\end{tabular}
\end{algorithm}

\newpage
\section{Experiments - additional material}

\subsection{Data preprocessing}

For the \textit{ImageNet} experiment, we re-size each RGB image so that the smallest dimension is $256$ pixels. We also re-scale each pixel value to the interval $[0,1]$. We then extract random crops (and their horizontal flips) of size $3\times 221\times 221$ pixels and present these to the network in mini-batches of size $128$. 

For the \textit{CIFAR} experiment, we use the original RGB image of size $3 \times 32 \times 32$. As before, we re-scale each pixel value to the interval $[0,1]$. We then extract random crops (and their horizontal flips) of size $3\times 28\times 28$ pixels and present these to the network in mini-batches of size $128$. 

The training and test loss and the test error are only computed from the center patch ($3 \times 28 \times 28$) for the \textit{CIFAR} experiment and the center patch ($3 \times 221 \times 221$) for the \textit{ImageNet} experiment.

\subsection{Data prefetching (Sampling the dataset by the local workers)}
We will now explain precisely how the dataset is sampled by each local worker as uniformly and efficiently as possible. The general parallel data loading scheme on a single machine is as follows: we use $k$ CPUs, where $k=8$, to load the data in parallel. Each data loader reads from the memory-mapped (mmap) file a chunk of $c$ raw images (preprocessing was described in the previous subsection) and their labels (for \textit{CIFAR} $c = 512$ and for \textit{ImageNet} $c = 64$).  For the \textit{CIFAR}, the mmap file of each data loader contains the entire dataset whereas for \textit{ImageNet}, each mmap file of each data loader contains different $1/k$ fractions of the entire dataset. A chunk of data is always sent by one of the data loaders to the first worker who requests the data. The next worker requesting the data from the same data loader will get the next chunk. Each worker requests in total $k$ data chunks from $k$ different data loaders and then process them before asking for new data chunks. Notice that each data loader cycles through the data in the mmap file, sending consecutive chunks to the workers in order in which it receives requests from them. When the data loader reaches the end of the mmap file, it selects the address in memory uniformly at random from the interval $[0,s]$, where $s = (\textsf{number of images in the mmap file} \text{\:\:modulo\:\:} \textsf{mini-batch size})$, and uses this address to start cycling again through the data in the mmap file. After the local worker receives the $k$ data chunks from the data loaders, it shuffles them and divides it into mini-batches of size $128$.

\subsection{Learning rates}

In Table~\ref{tab:learningrate} we summarize the learning rates $\eta$ (we used constant learning rates) explored for each method shown in Figure~\ref{fig:CIFAR}. For all values of $\tau$ the same set of learning rates was explored for each method.
\begin{table}[h]
\center
\setlength{\tabcolsep}{4pt}
\caption{Learning rates explored for each method shown in Figure~\ref{fig:CIFAR} (\textit{CIFAR} experiment).}
\begin{tabular}{c|c|}
  \cline{2-2} 
\multirow{2}{*}{} & $\eta$\\
  \hline
\multicolumn{1}{|c|}{\multirow{1}{*}{EASGD}}& $\{0.05,0.01,0.005\}$\\
  \hline
\multicolumn{1}{|c|}{\multirow{1}{*}{EAMSGD}}&  $\{0.01,0.005,0.001\}$\\
\hline
\multicolumn{1}{|c|}{\multirow{1}{*}{DOWNPOUR}}& \\
\multicolumn{1}{|c|}{\multirow{1}{*}{ADOWNPOUR}}& $\{0.005,0.001,0.0005\}$\\
\multicolumn{1}{|c|}{\multirow{1}{*}{MVADOWNPOUR}}& \\
\hline
\multicolumn{1}{|c|}{\multirow{1}{*}{MDOWNPOUR}}&  $\{0.00005,0.00001,0.000005\}$\\
\hline
\hline
\multicolumn{1}{|c|}{\multirow{1}{*}{SGD, ASGD, MVASGD}}& $\{0.05,0.01,0.005\}$\\
\hline
\multicolumn{1}{|c|}{\multirow{1}{*}{MSGD}}& $\{0.001,0.0005,0.0001\}$\\
\hline
\end{tabular} 
\label{tab:learningrate}
\end{table} 

In Table~\ref{tab:learningrate2} we summarize the learning rates $\eta$ (we used constant learning rates) explored for each method shown in Figure~\ref{fig:CIFAR2}. For all values of $p$ the same set of learning rates was explored for each method.
\begin{table}[htp!]
\center
\setlength{\tabcolsep}{4pt}
\caption{Learning rates explored for each method shown in Figure~\ref{fig:CIFAR2} (\textit{CIFAR} experiment).}
\begin{tabular}{c|c|}
  \cline{2-2} 
\multirow{2}{*}{} & $\eta$\\
  \hline
\multicolumn{1}{|c|}{\multirow{1}{*}{EASGD}}& $\{0.05,0.01,0.005\}$\\
  \hline
\multicolumn{1}{|c|}{\multirow{1}{*}{EAMSGD}}&  $\{0.01,0.005,0.001\}$\\
\hline
\multicolumn{1}{|c|}{\multirow{1}{*}{DOWNPOUR}}& $\{0.005,0.001,0.0005\}$\\
\hline
\multicolumn{1}{|c|}{\multirow{1}{*}{MDOWNPOUR}}& $\{0.00005,0.00001,0.000005\}$\\
\hline
\hline
\multicolumn{1}{|c|}{\multirow{1}{*}{SGD, ASGD, MVASGD}}& $\{0.05,0.01,0.005\}$\\
\hline
\multicolumn{1}{|c|}{\multirow{1}{*}{MSGD}}& $\{0.001,0.0005,0.0001\}$\\
\hline
\end{tabular} 
\label{tab:learningrate2}
\end{table}

In Table~\ref{tab:learningrate3} we summarize the initial learning rates $\eta$ we use for each method shown in Figure~\ref{fig:ImageNet}. For all values of $p$ the same set of learning rates was explored for each method. We also used the rule of the thumb to decrease the initial learning rate twice, first time we divided it by $5$ and the second time by $2$, when we observed that the decrease of the online predictive (training) loss saturates. 
\vspace{-0.1in}
\begin{table}[htp!]
\center
\setlength{\tabcolsep}{4pt}
\caption{Learning rates explored for each method shown in Figure~\ref{fig:ImageNet} (\textit{ImageNet} experiment).}
\begin{tabular}{c|c|}
  \cline{2-2} 
\multirow{2}{*}{} & $\eta$\\
  \hline
\multicolumn{1}{|c|}{\multirow{1}{*}{EASGD}}& $0.1$\\
  \hline
\multicolumn{1}{|c|}{\multirow{1}{*}{EAMSGD}}&  $0.001$\\
\hline
\multicolumn{1}{|c|}{\multirow{1}{*}{DOWNPOUR}}&  for $p=4$: $0.02$\\
\multicolumn{1}{|c|}{\multirow{1}{*}{}}& for $p=8$: $0.01$\\
\hline
\hline
\multicolumn{1}{|c|}{\multirow{1}{*}{SGD, ASGD, MVASGD}}& $0.05$\\
\hline
\multicolumn{1}{|c|}{\multirow{1}{*}{MSGD}}& $0.0005$\\
\hline
\end{tabular} 
\label{tab:learningrate3}
\end{table}

\subsection{Comparison of \textit{SGD}, \textit{ASGD}, \textit{MVASGD} and \textit{MSGD}}

\begin{figure}[h]
  \center
\includegraphics[trim=0cm 0cm 0cm 0cm,clip,width = 1.56in]{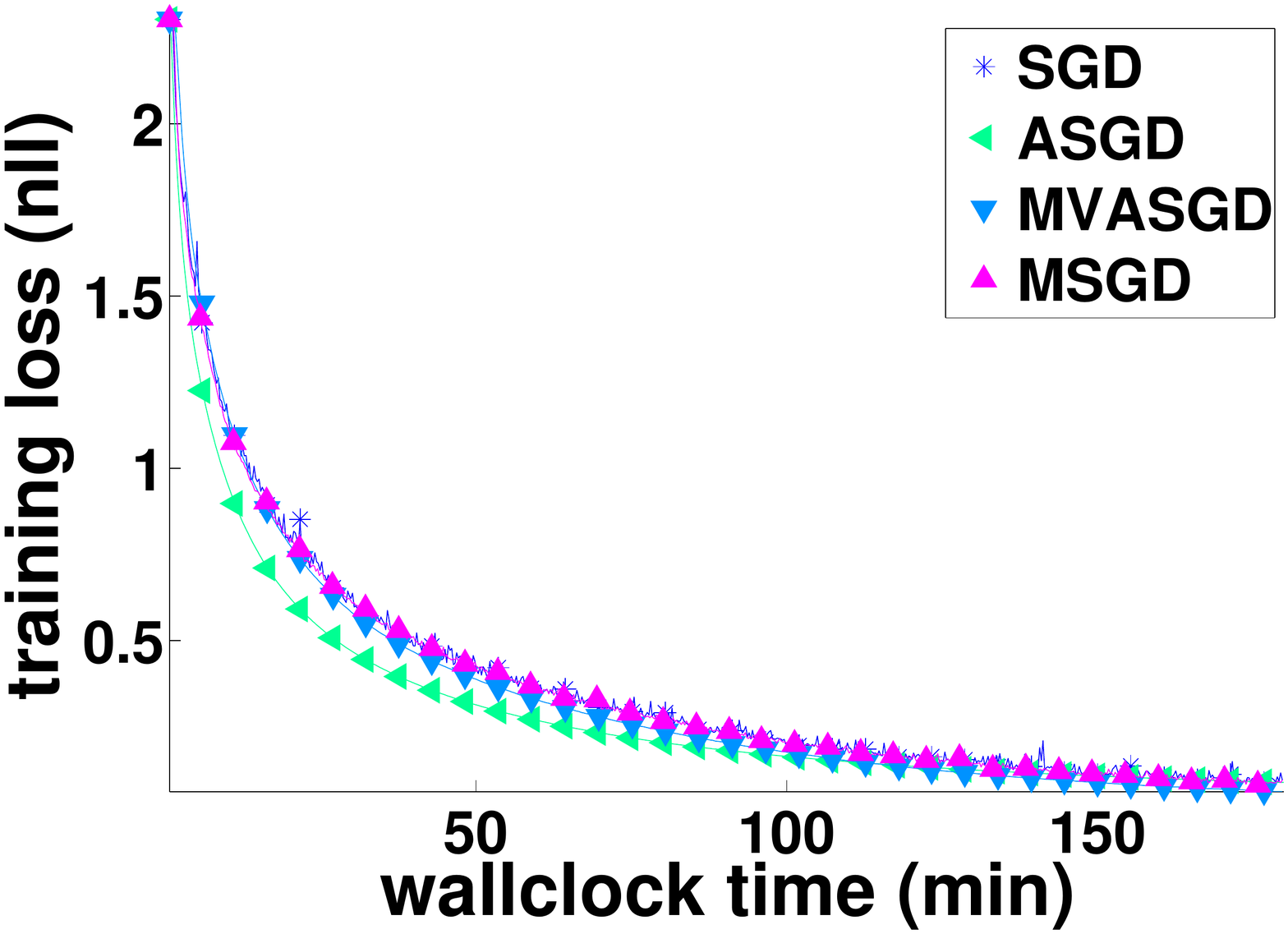}
\hspace{-0.28in}\includegraphics[trim=0cm 0cm 0cm 0cm,clip,width = 1.56in]{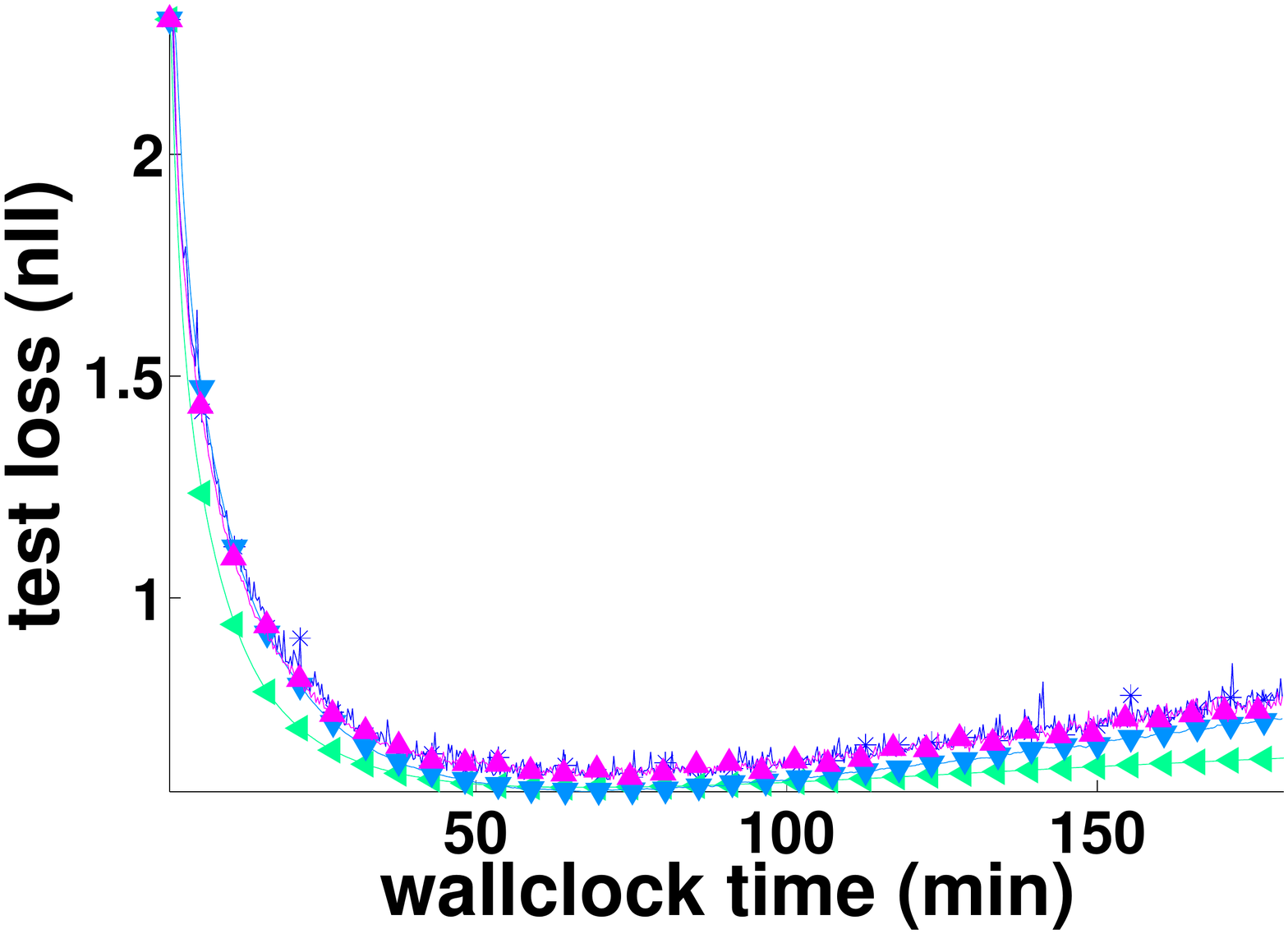} 
\hspace{-0.28in}\includegraphics[trim=0cm 0cm 0cm 0cm,clip,width = 1.56in]{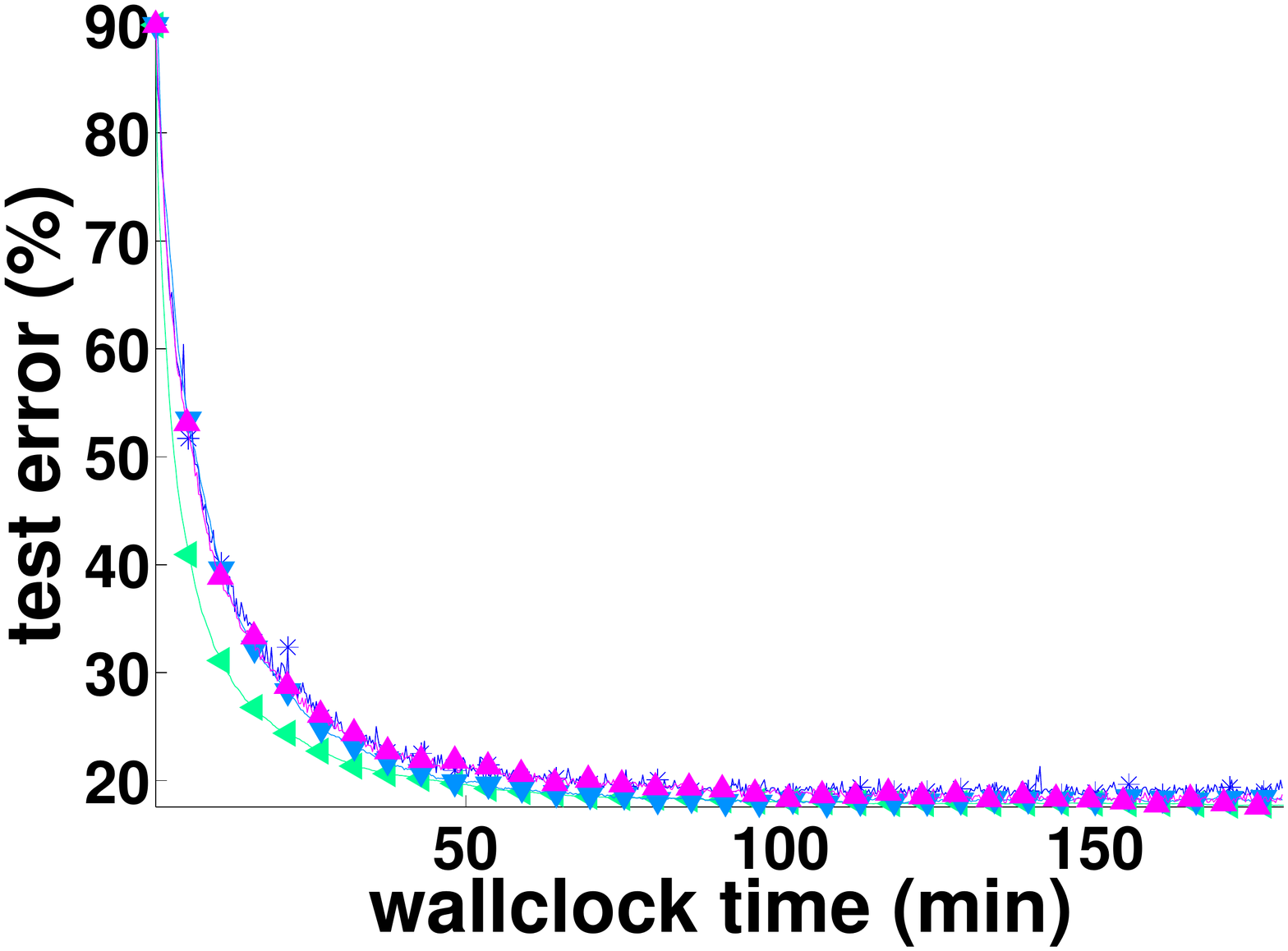}
\hspace{-0.28in}\includegraphics[trim=0cm 0cm 0cm 0cm,clip,width = 1.56in]{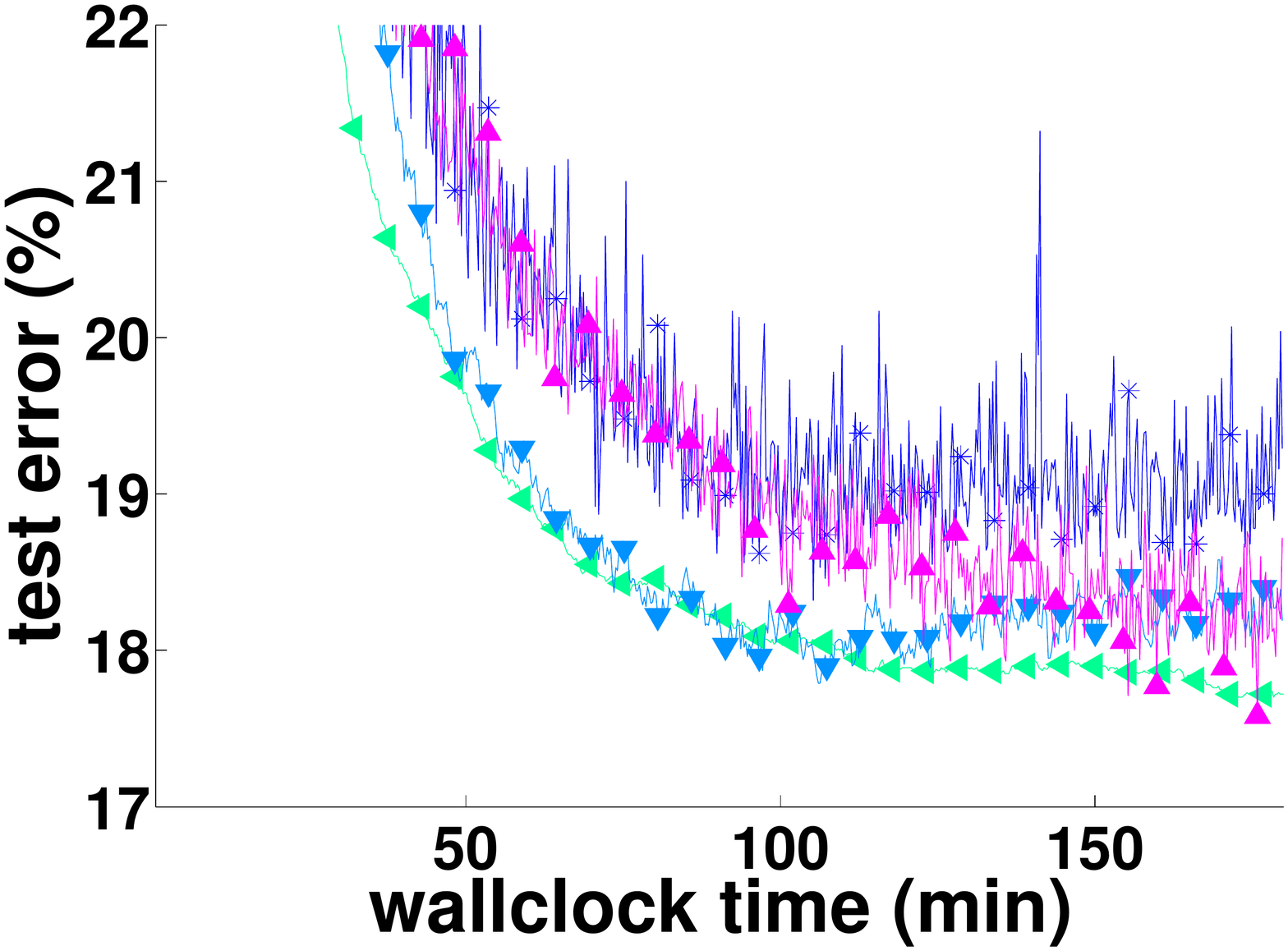}
\vspace{-0.2in}
\caption{Convergence of the training and test loss (negative log-likelihood) and the test error (original and zoomed) computed for the center variable as a function of wallclock time for \textit{SGD}, \textit{ASGD}, \textit{MVASGD} and \textit{MSGD} ($p = 1$) on the \textit{CIFAR} experiment.}
\label{fig:CIFAR_supp}
\vspace{-0.1in}
\end{figure}

\begin{figure}[htp!]
  \center
\includegraphics[trim=0cm 0cm 0cm 0cm,clip,width = 1.56in]{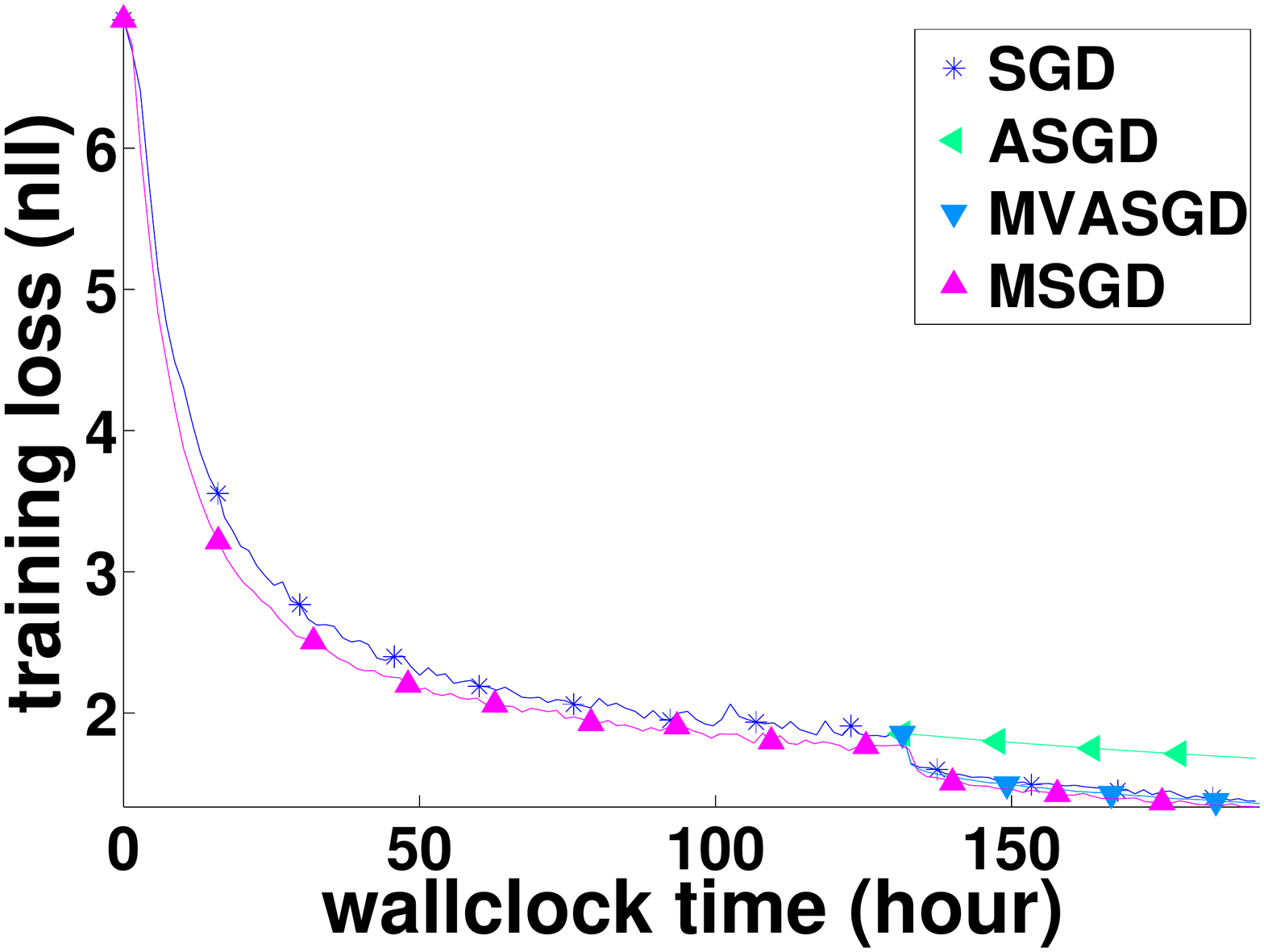}
\hspace{-0.28in}\includegraphics[trim=0cm 0cm 0cm 0cm,clip,width = 1.56in]{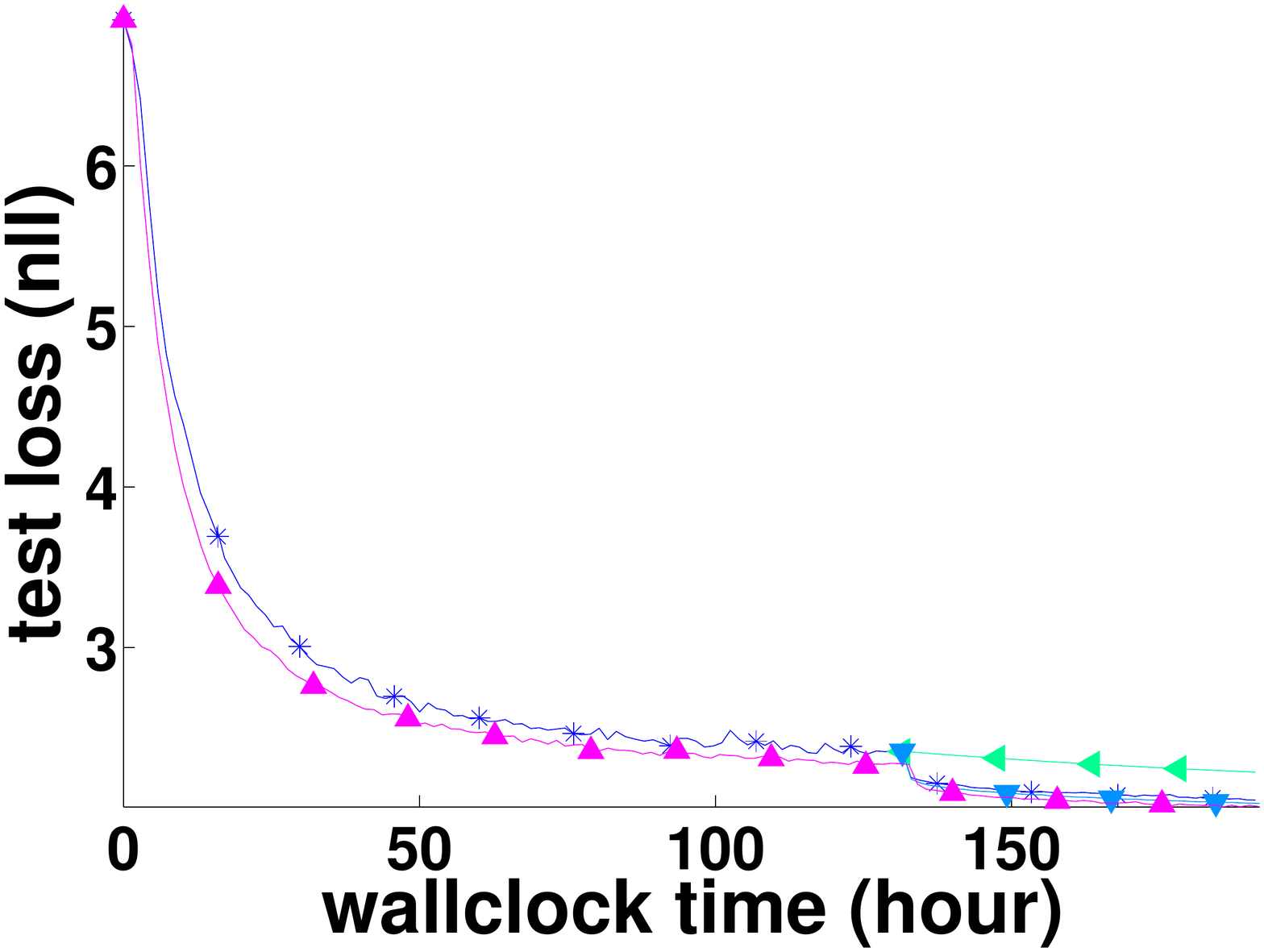} 
\hspace{-0.28in}\includegraphics[trim=0cm 0cm 0cm 0cm,clip,width = 1.56in]{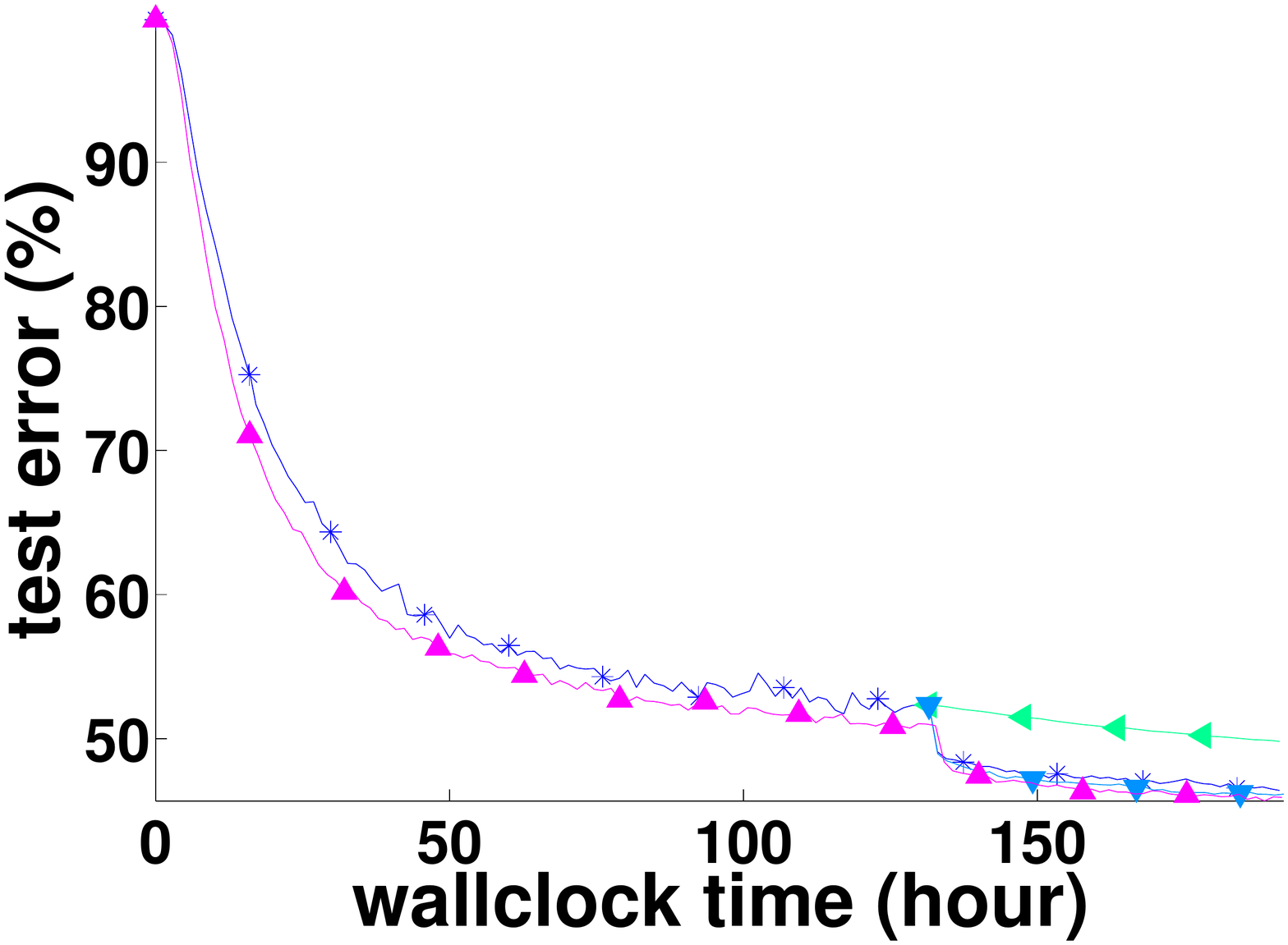}
\hspace{-0.28in}\includegraphics[trim=0cm 0cm 0cm 0cm,clip,width = 1.56in]{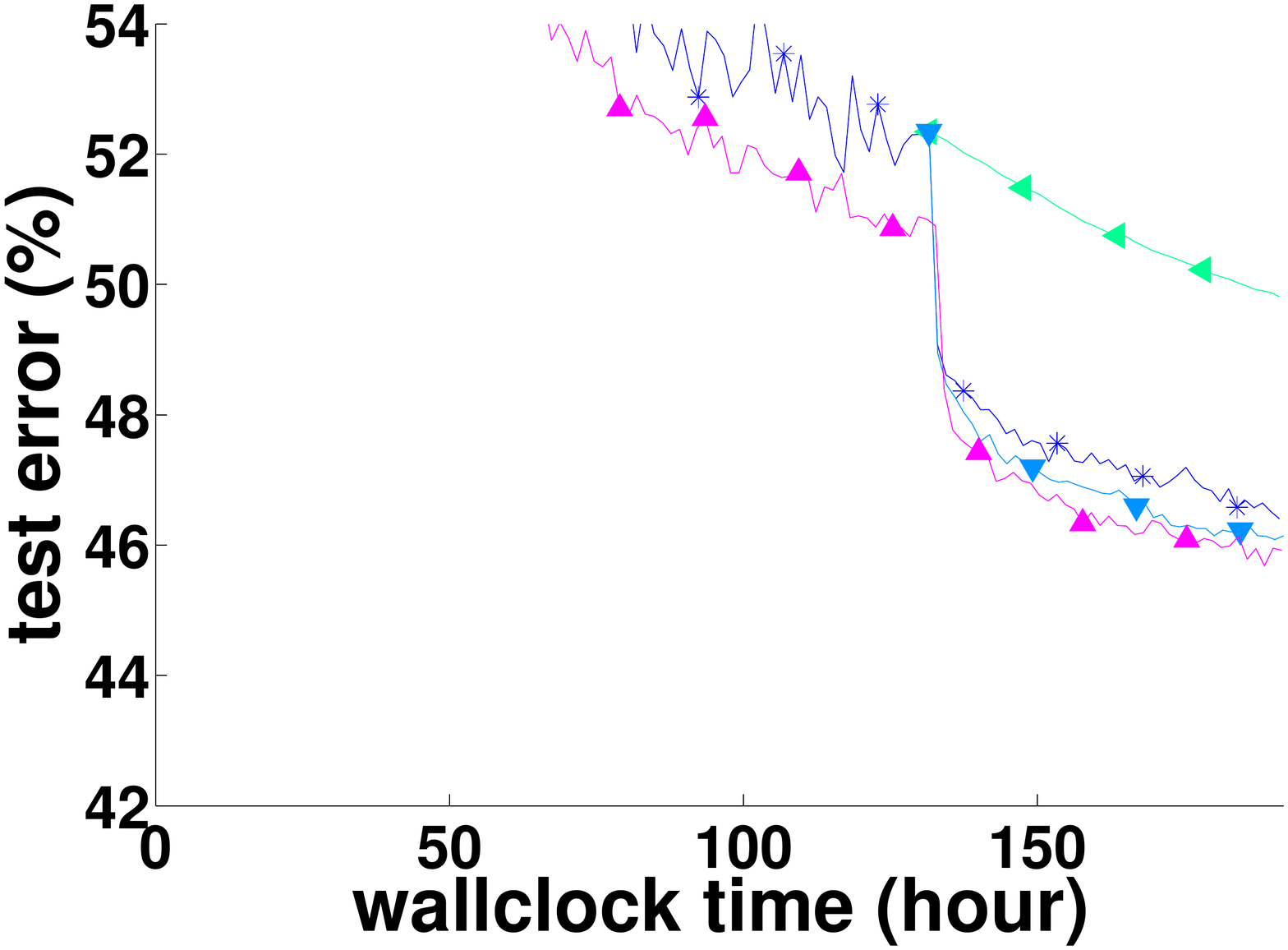}
\vspace{-0.2in}
\caption{Convergence of the training and test loss (negative log-likelihood) and the test error (original and zoomed) computed for the center variable as a function of wallclock time for \textit{SGD}, \textit{ASGD}, \textit{MVASGD} and \textit{MSGD} ($p = 1$) on the \textit{ImageNet} experiment.}
\label{fig:ImageNet_supp}
\vspace{-0.1in}
\end{figure}

Figure~\ref{fig:CIFAR_supp} shows the convergence of the training and test loss (negative log-likelihood) and the test error computed for the center variable as a function of wallclock time for \textit{SGD}, \textit{ASGD}, \textit{MVASGD} and \textit{MSGD} ($p=1$) on the \textit{CIFAR} experiment. For all \textit{CIFAR} experiments we always start the averaging for the $ADOWNPOUR$ and $ASGD$ methods from the very beginning of each experiment. For all \textit{ImageNet} experiments we start the averaging for the $ASGD$ at the same time when we first reduce the initial learning rate. 

Figure~\ref{fig:ImageNet_supp} shows the convergence of the training and test loss (negative log-likelihood) and the test error computed for the center variable as a function of wallclock time for \textit{SGD}, \textit{ASGD}, \textit{MVASGD} and \textit{MSGD} ($p=1$) on the \textit{ImageNet} experiment.

\subsection{Dependence of the learning rate}
\label{sec:deplr}
This section discusses the dependence of the trade-off between exploration and exploitation on the learning rate. We compare the performance of respectively \textit{EAMSGD} and \textit{EASGD} for different learning rates $\eta$ when $p=16$ and $\tau = 10$ on the \textit{CIFAR} experiment. We observe in Figure~\ref{fig:CIFAR3} that higher learning rates $\eta$ lead to better test performance for the \textit{EAMSGD} algorithm which potentially can be justified by the fact that they sustain higher fluctuations of the local workers. We conjecture that higher fluctuations lead to more exploration and simultaneously they also impose higher regularization. This picture however seems to be opposite for the \textit{EASGD} algorithm for which larger learning rates hurt the performance of the method and lead to overfitting. Interestingly in this experiment for both \textit{EASGD} and \textit{EAMSGD} algorithm, the learning rate for which the best training performance was achieved simultaneously led to the worst test performance.

\begin{figure}[h!]
\vspace{-0.1in}
  \center
\includegraphics[trim=0cm 0cm 0cm 0cm,clip,width = 1.56in]{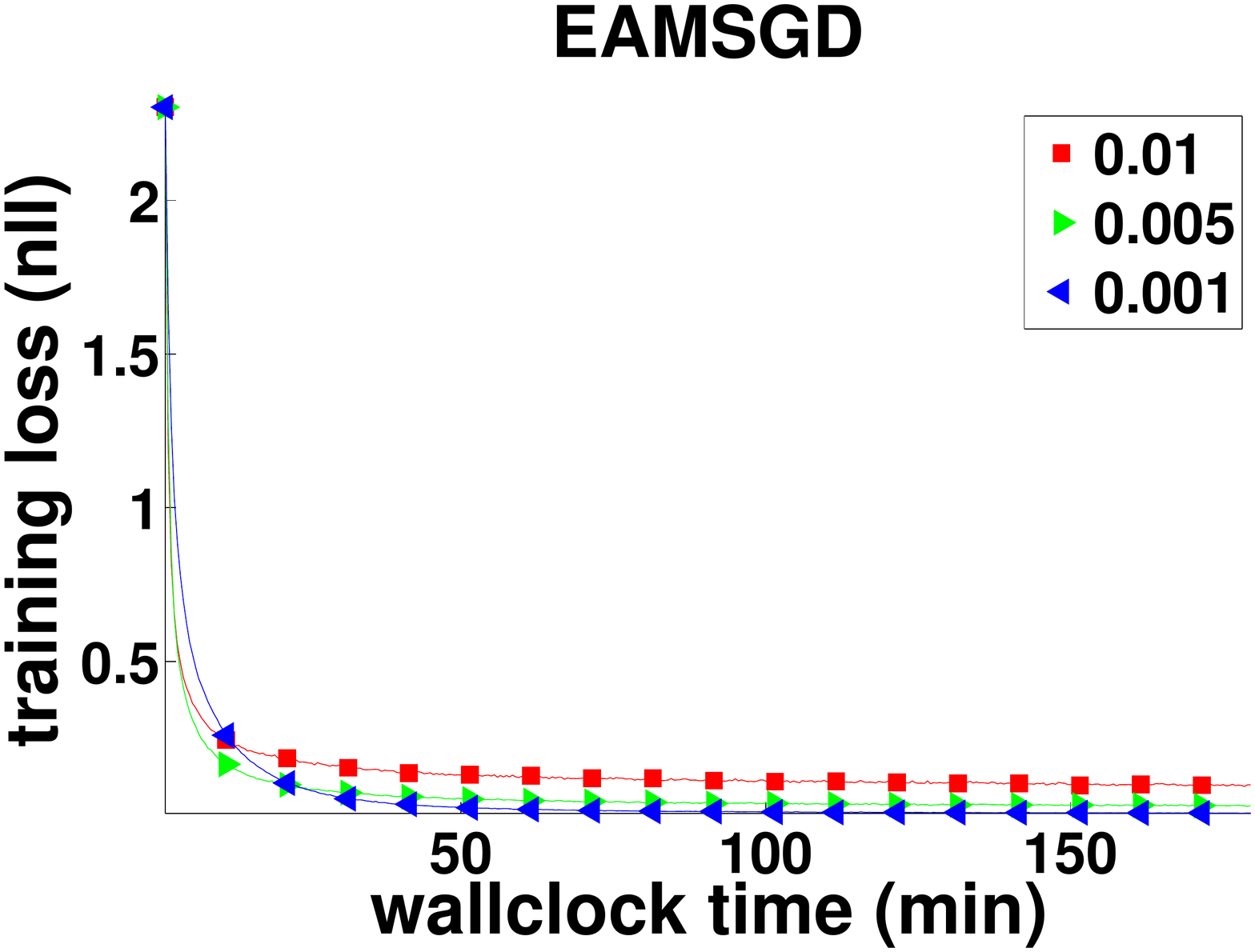}
\hspace{-0.28in}\includegraphics[trim=0cm 0cm 0cm 0cm,clip,width = 1.56in]{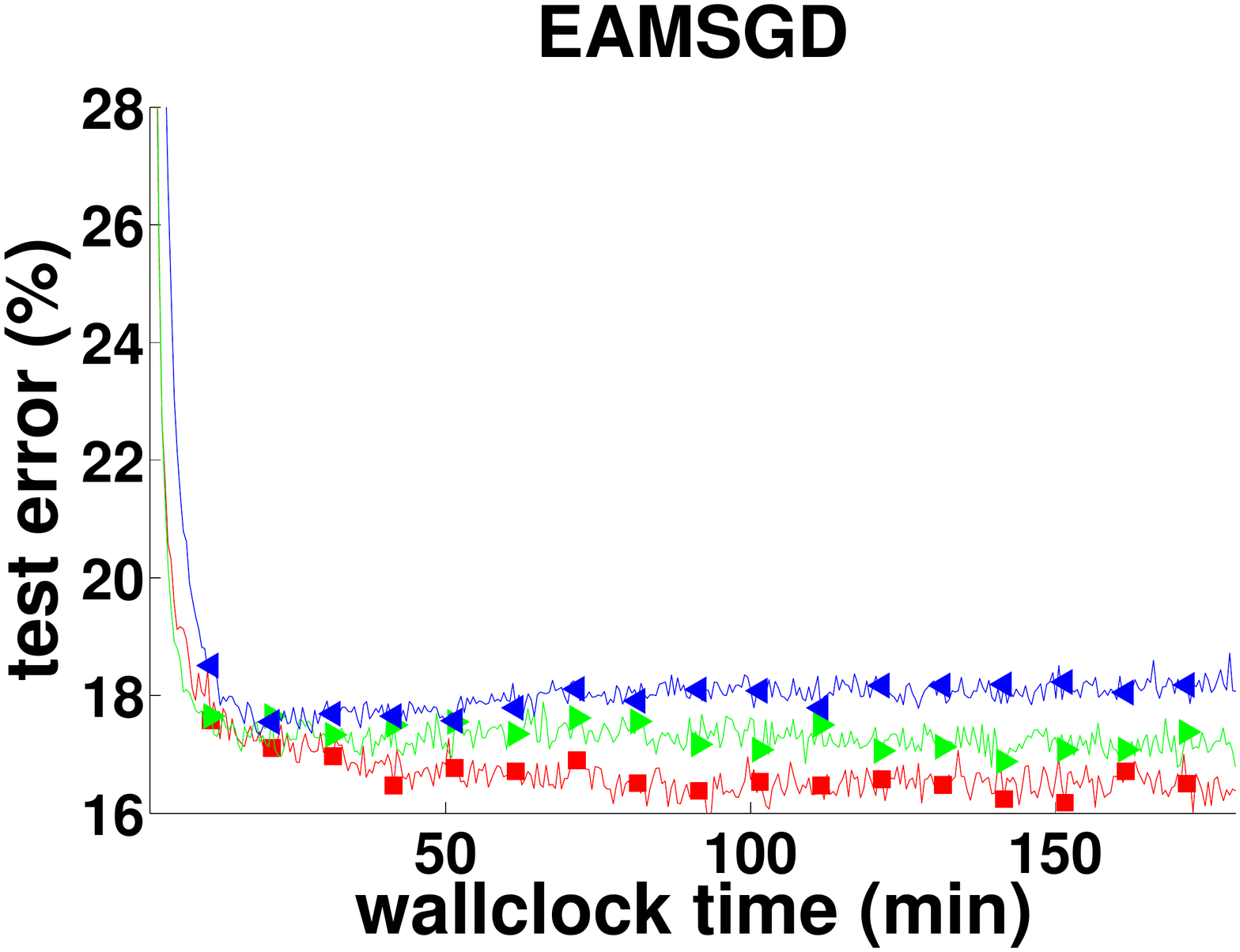}
\hspace{-0.28in}\includegraphics[trim=0cm 0cm 0cm 0cm,clip,width = 1.56in]{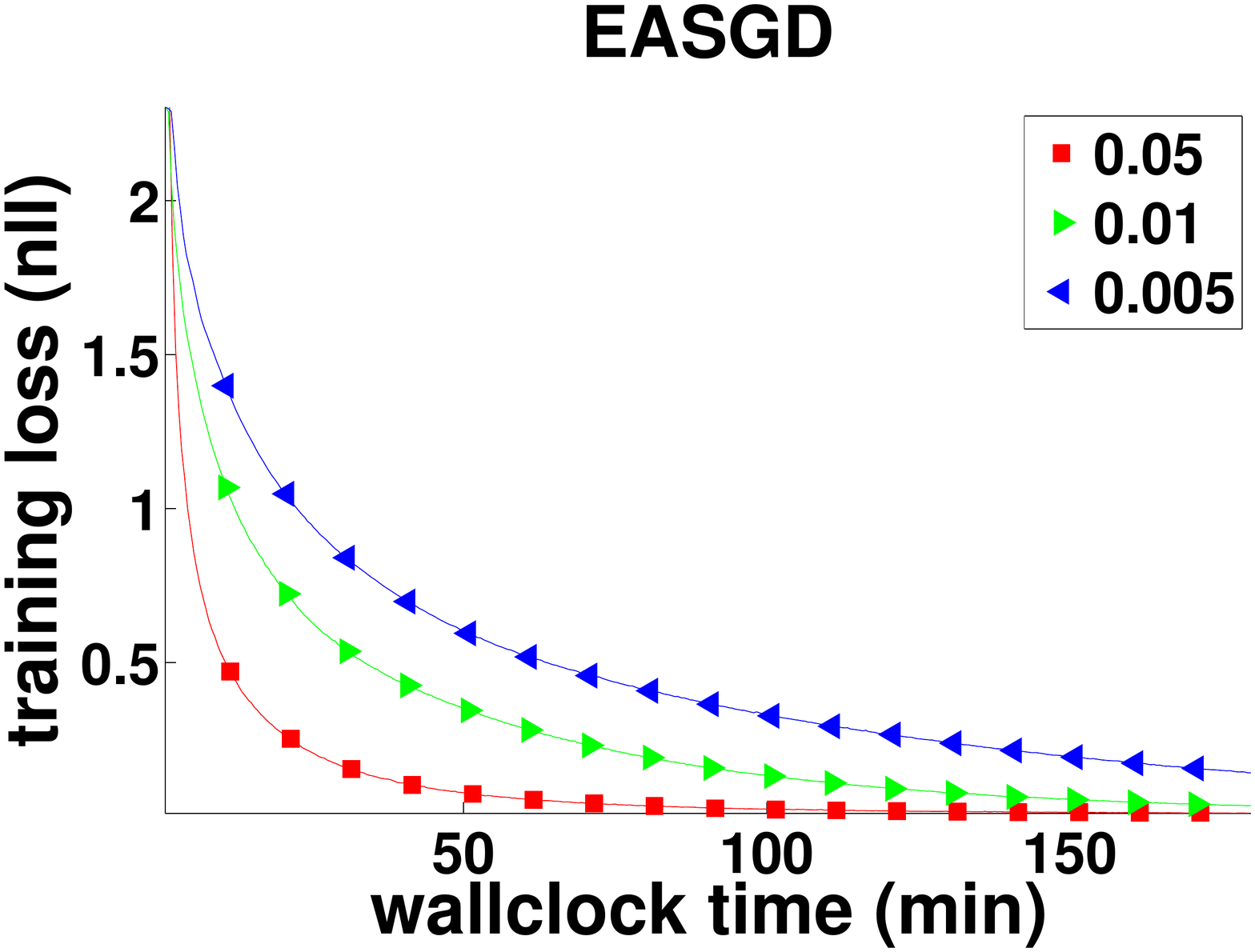} 
\hspace{-0.28in}\includegraphics[trim=0cm 0cm 0cm 0cm,clip,width = 1.56in]{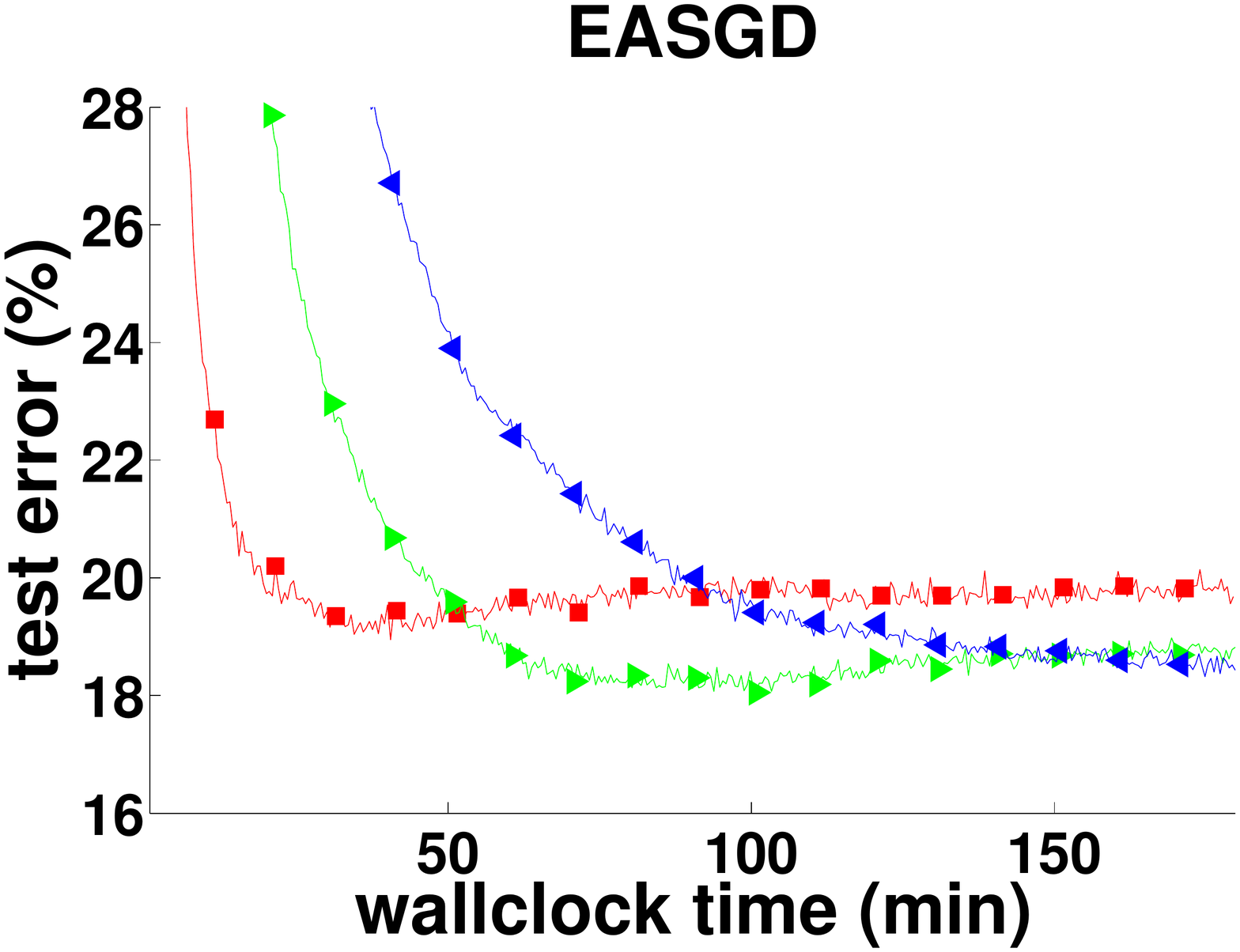}
\vspace{-0.2in}
\caption{Convergence of the training loss (negative log-likelihood, original) and the test error (zoomed) computed for the center variable as a function of wallclock time for \textit{EAMSGD} and  \textit{EASGD} run with different values of $\eta$ on the \textit{CIFAR} experiment. $p=16$, $\tau = 10$.}
\label{fig:CIFAR3}
\vspace{-0.1in}
\end{figure}

\subsection{Dependence of the communication period}
\label{sec:deptau}
This section discusses the dependence of the trade-off between exploration and exploitation on the communication period. We have observed from the \textit{CIFAR} experiment that \textit{EASGD} algorithm exhibits very similar convergence behavior when $\tau=1$ up to even $\tau=1000$, whereas \textit{EAMSGD} can get trapped at worse energy (loss) level for $\tau=100$. This behavior of \textit{EAMSGD} is most likely due to the non-convexity of the objective function. Luckily, it can be avoided by gradually decreasing the learning rate, i.e. increasing the penalty term $\rho$ (recall $\alpha = \eta \rho$), as shown in Figure~\ref{fig:CIFAR3_supp}. In contrast, the \textit{EASGD} algorithm
does not seem to get trapped at all along its trajectory. 
The performance of \textit{EASGD} is less sensitive to increasing the communication period compared to \textit{EAMSGD}, whereas for the \textit{EAMSGD} the careful choice of the learning rate for large communication periods seems crucial.

Compared to all earlier results, the experiment in this section is re-run three times with a new random\footnote{To clarify, the random initialization we use is by default in Torch's implementation.} seed and with faster cuDNN\footnote{\url{https://developer.nvidia.com/cuDNN}} package\footnote{\url{https://github.com/soumith/cudnn.torch}}. All our methods are implemented in Torch\footnote{\url{http://torch.ch}}. The Message Passing Interface implementation MVAPICH2\footnote{\url{http://mvapich.cse.ohio-state.edu}} is used for the GPU-CPU communication.

\begin{figure}[htp!]
\center
\includegraphics[trim=0cm 0cm 0cm 0cm,clip,width = 1.56in]{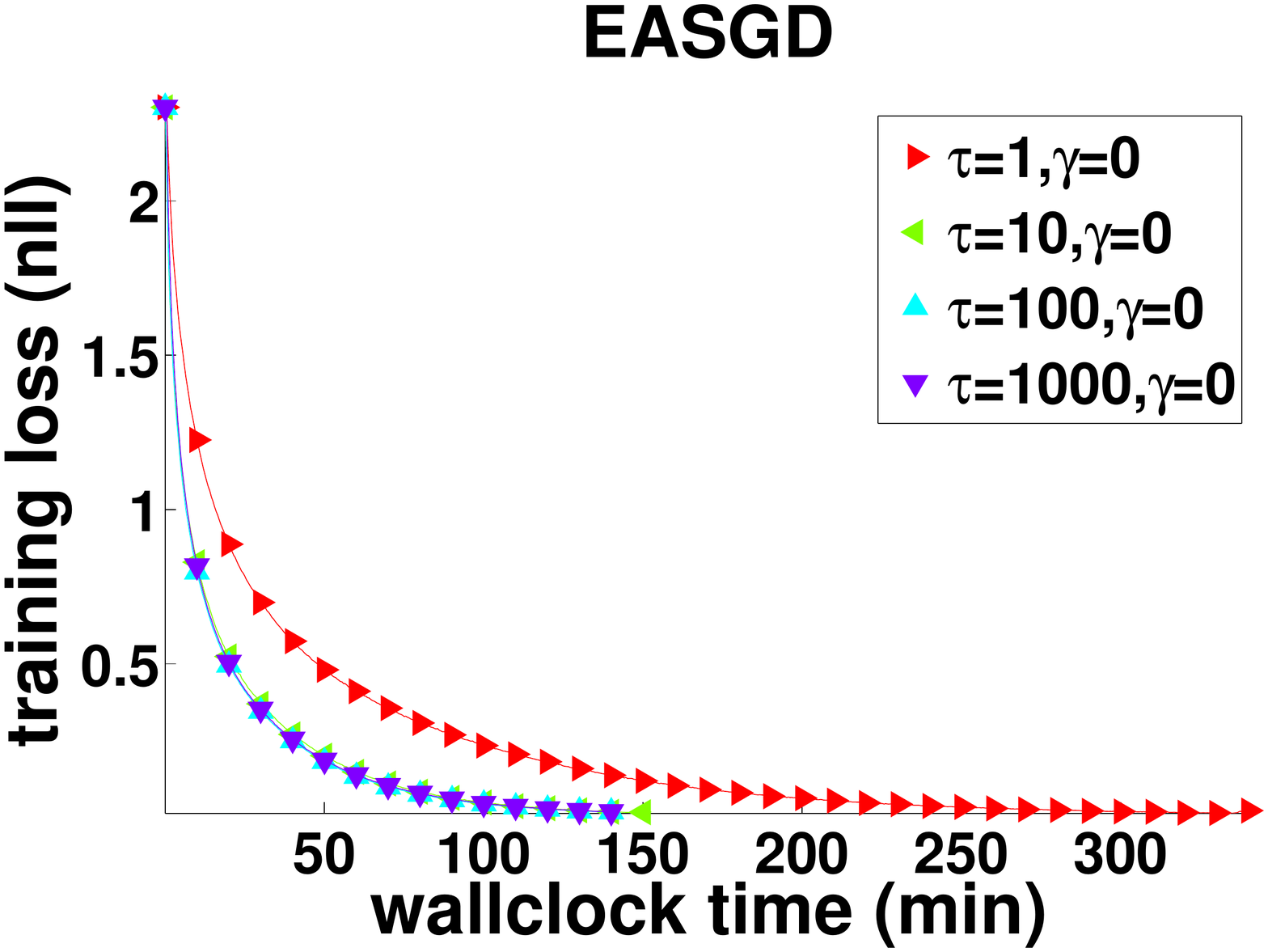}
\hspace{-0.28in}\includegraphics[trim=0cm 0cm 0cm 0cm,clip,width = 1.56in]{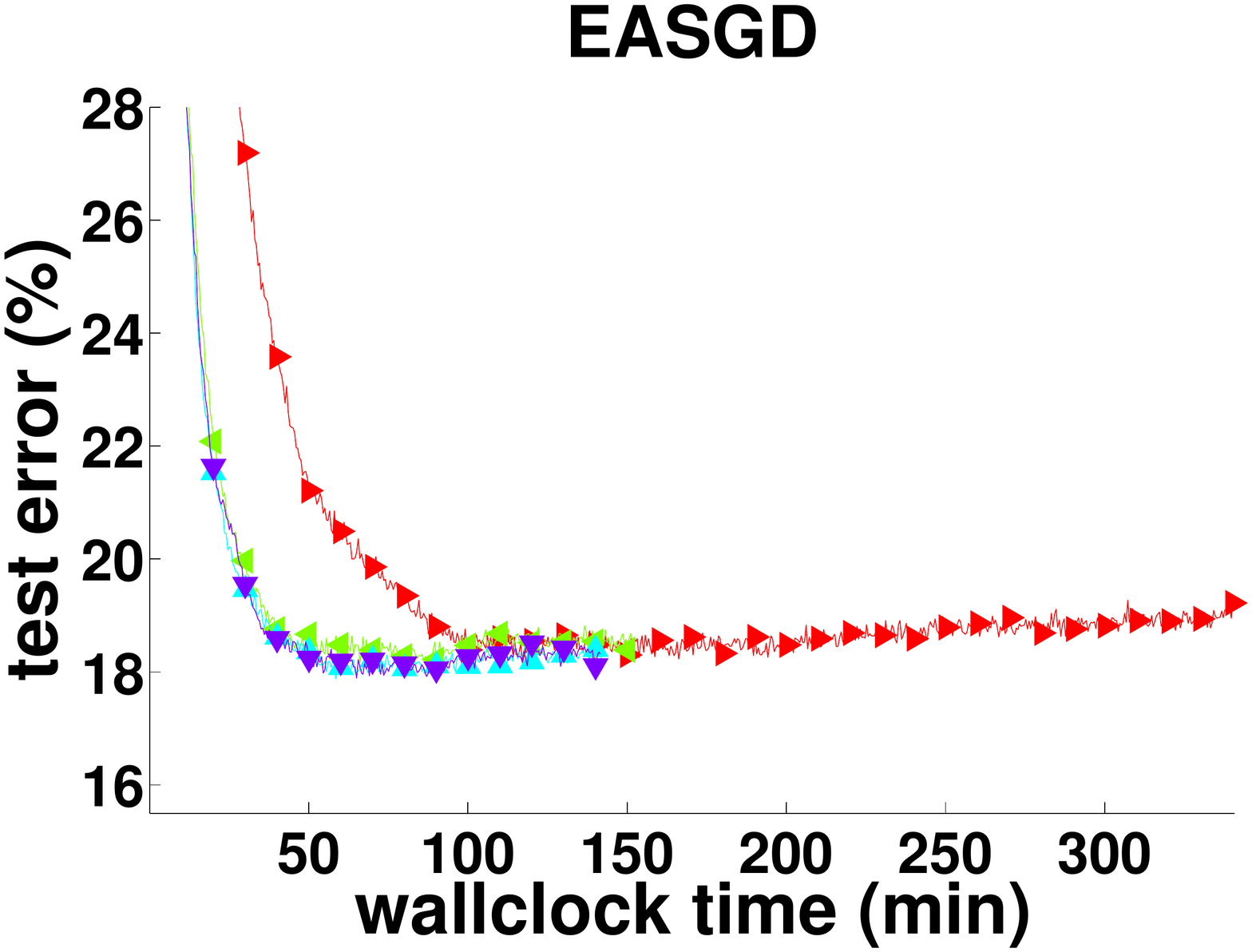}
\hspace{-0.28in}\includegraphics[trim=0cm 0cm 0cm 0cm,clip,width = 1.56in]{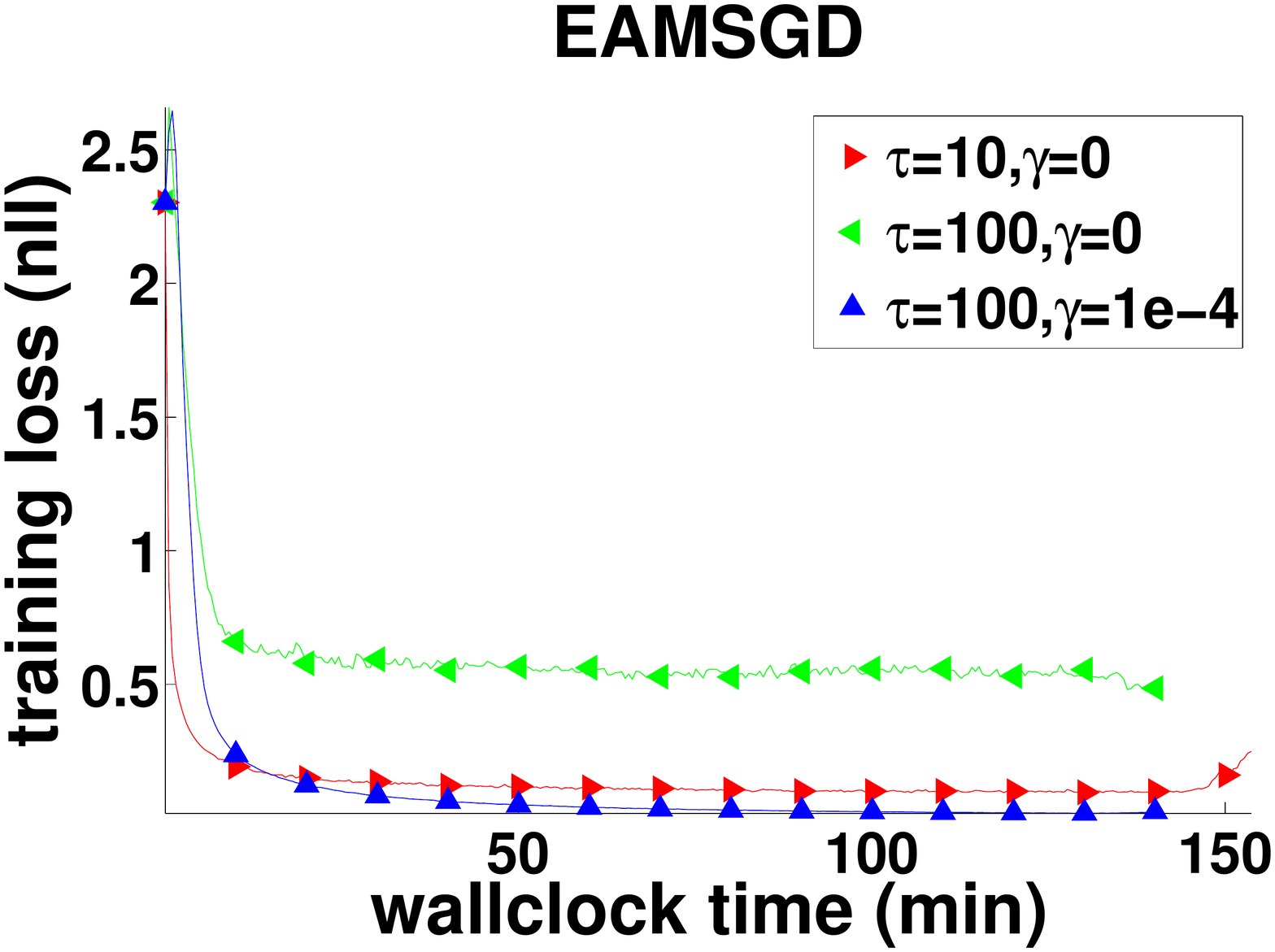} 
\hspace{-0.28in}\includegraphics[trim=0cm 0cm 0cm 0cm,clip,width = 1.56in]{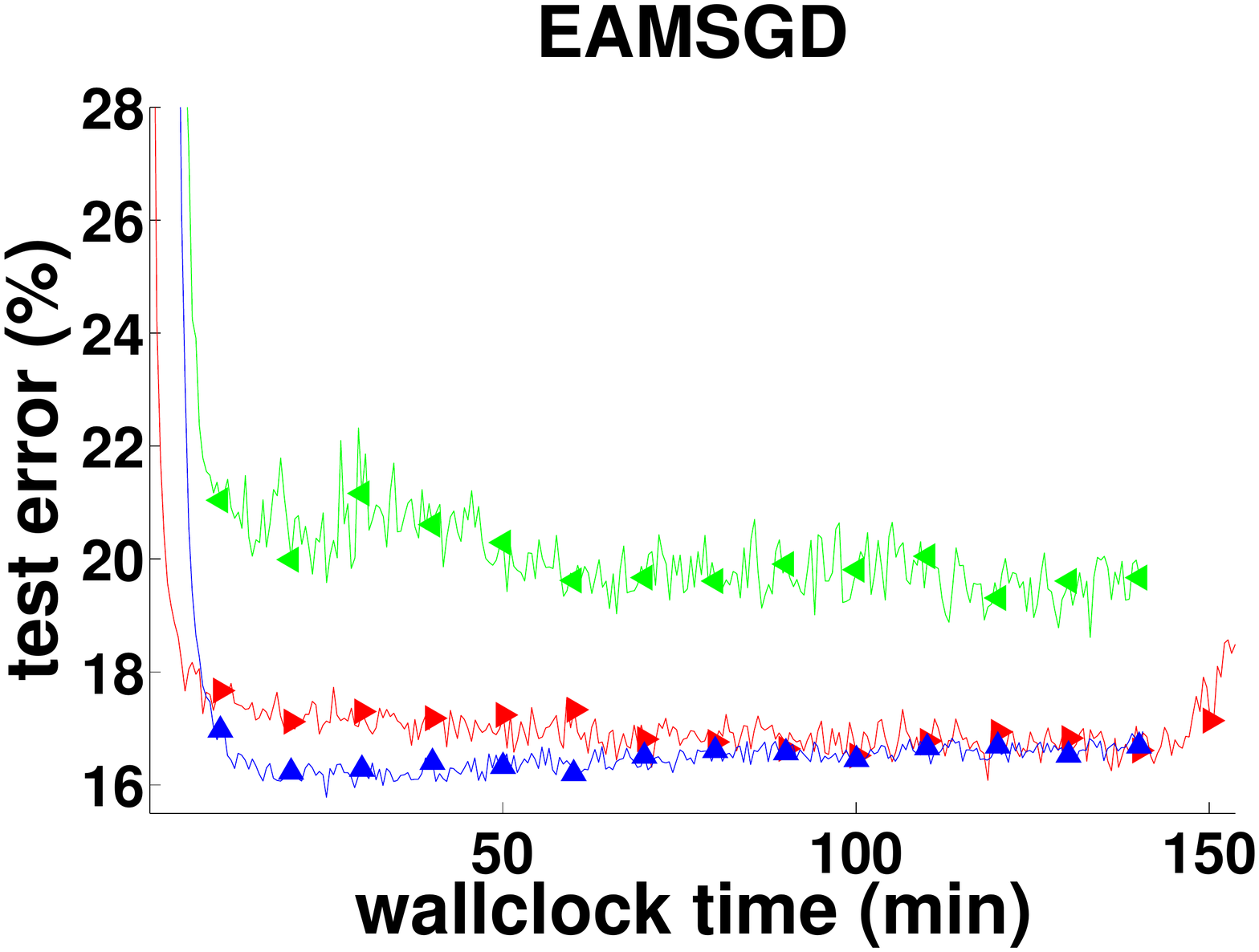}
\caption{Convergence of the training loss (negative log-likelihood, original) and the test error (zoomed) computed for the center variable as a function of wallclock time for \textit{EASGD} and \textit{EAMSGD} 
($p = 16,\eta=0.01,\beta=0.9,\delta=0.99$) on the \textit{CIFAR} experiment with various communication period $\tau$ and learning rate decay $\gamma$. The learning rate is decreased gradually over time based each local worker's own clock $t$ with $\eta_t=  \eta/(1+\gamma t)^{0.5}$.}
\label{fig:CIFAR3_supp}
\vspace{-0.1in}
\end{figure}

\subsection{Breakdown of the wallclock time}

In addition, we report in Table~\ref{tab:comptime} the breakdown of the total running time for \textit{EASGD} when $\tau=10$ (the time breakdown for \textit{EAMSGD} is almost identical) and  \textit{DOWNPOUR} when $\tau = 1$ into computation time, data loading time and parameter communication time. For the \textit{CIFAR} experiment the reported time corresponds to processing $400 \times 128$ data samples whereas for the \textit{ImageNet} experiment it corresponds to processing $1024 \times 128$ data samples. For $\tau=1$ and $p \in \{8,16\}$ we observe that the communication time accounts for significant portion of the total running time whereas for $\tau=10$ the communication time becomes negligible compared to the total running time (recall that based on previous results \textit{EASGD} and \textit{EAMSGD} achieve best performance with larger $\tau$ which is ideal in the setting when communication is time-consuming). 
\begin{table}[h!]
\vspace{-0.10in}
\setlength{\tabcolsep}{1.75pt}
\begin{minipage}{.5\linewidth}
  \begin{center}
    \begin{tabular}{c|c|c|c|c|}
  \cline{2-5} 
\multirow{1}{*}{} & $p=1$ & $p=4$ & $p=8$ & $p=16$\\
    \hline
\multicolumn{1}{|c|}{\multirow{1}{*}{$\tau = 1$}} & $12/1/0$ & $11/2/3$ & $11/2/5$ & $11/2/9$\\
\hline
\hline
\multicolumn{1}{|c|}{\multirow{1}{*}{$\tau = 10$}} & NA & $11/2/1$ & $11/2/1$ & $12/2/1$\\
    \hline
    \end{tabular}
  \end{center}
\end{minipage}
\begin{minipage}{.5\linewidth}
  \begin{center}
    \begin{tabular}{c|c|c|c|}
  \cline{2-4} 
\multirow{1}{*}{} & $p=1$ & $p=4$ & $p=8$\\
    \hline
\multicolumn{1}{|c|}{\multirow{1}{*}{$\tau = 1$}} & $1248/20/0$ & $1323/24/173$ & $1239/61/284$\\
\hline
\hline
\multicolumn{1}{|c|}{\multirow{1}{*}{$\tau = 10$}} & NA & $1254/58/7$ & $1266/84/11$\\
    \hline
    \end{tabular}
  \end{center}
\end{minipage}
\caption{Approximate computation time, data loading time and parameter communication time [sec] for \textit{DOWNPOUR} (top line for $\tau=1$) and \textit{EASGD} (the time breakdown for \textit{EAMSGD} is almost identical) (bottom line for $\tau=10$). Left time corresponds to \textit{CIFAR} experiment and right table corresponds to \textit{ImageNet} experiment.}
\label{tab:comptime}
\vspace{-0.05in}
\end{table}

\subsection{Time speed-up}
In Figure~\ref{fig:CIFARspeedup} and~\ref{fig:ImageNetspeedup}, we summarize the wall clock time needed to achieve the same level of the test error for all the methods in the \textit{CIFAR} and \textit{ImageNet} experiment as a function of the number of local workers $p$. For the \textit{CIFAR} (Figure~\ref{fig:CIFARspeedup}) we examined the following levels: $\{21\%,20\%,19\%,18\%\}$ and for the \textit{ImageNet} (Figure~\ref{fig:ImageNetspeedup}) we examined: $\{49\%,47\%,45\%,43\%\}$. If some method does not appear on the figure for a given test error level, it indicates that this method never achieved this level. For the \textit{CIFAR} experiment we observe that from among \textit{EASGD}, \textit{DOWNPOUR} and \textit{MDOWNPOUR} methods, the \textit{EASGD} method needs less time to achieve a particular level of test error. We observe that with higher $p$ each of these methods does not necessarily need less time to achieve the same level of test error. This seems counter intuitive though recall that the learning rate for the methods is selected based on the smallest achievable test error. For larger $p$ smaller learning rates were selected than for smaller $p$ which explains our results. Meanwhile, the \textit{EAMSGD} method achieves significant speed-up over other methods for all the test error levels. For the \textit{ImageNet} experiment we observe that all methods outperform \textit{MSGD} and furthermore with $p=4$ or $p=8$ each of these methods requires less time to achieve the same level of test error. The \textit{EAMSGD} consistently needs less time than any other method, in particular \textit{DOWNPOUR}, to achieve any of the test error levels. 
\begin{figure}[h]
\vspace{-0.15in}
  \center
\includegraphics[trim=0cm 0cm 0cm 0cm,clip,width = 1.56in]{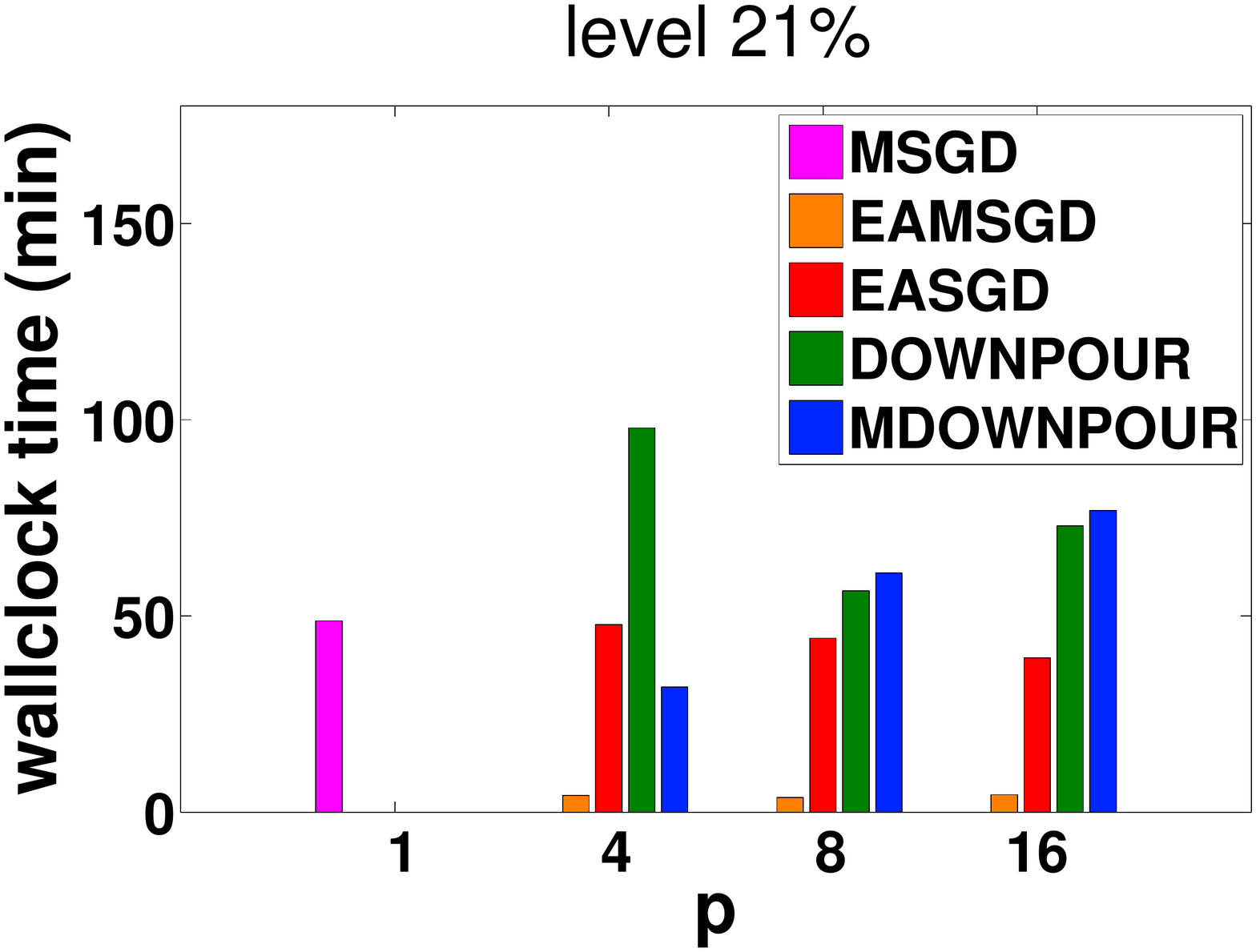}
\hspace{-0.28in}\includegraphics[trim=0cm 0cm 0cm 0cm,clip,angle=0,width = 1.56in]{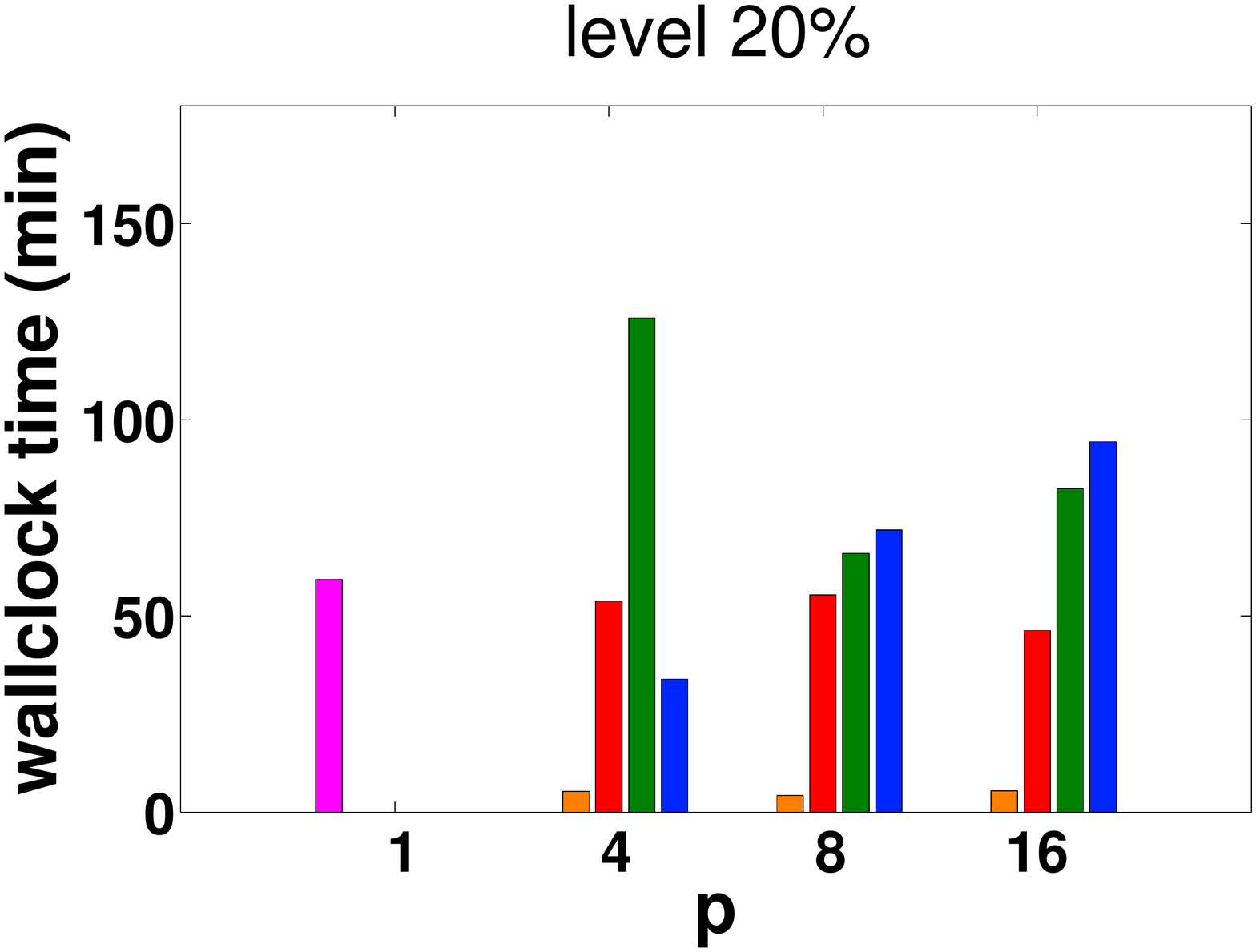} 
\hspace{-0.28in}\includegraphics[trim=0cm 0cm 0cm 0cm,clip,angle=0,width = 1.56in]{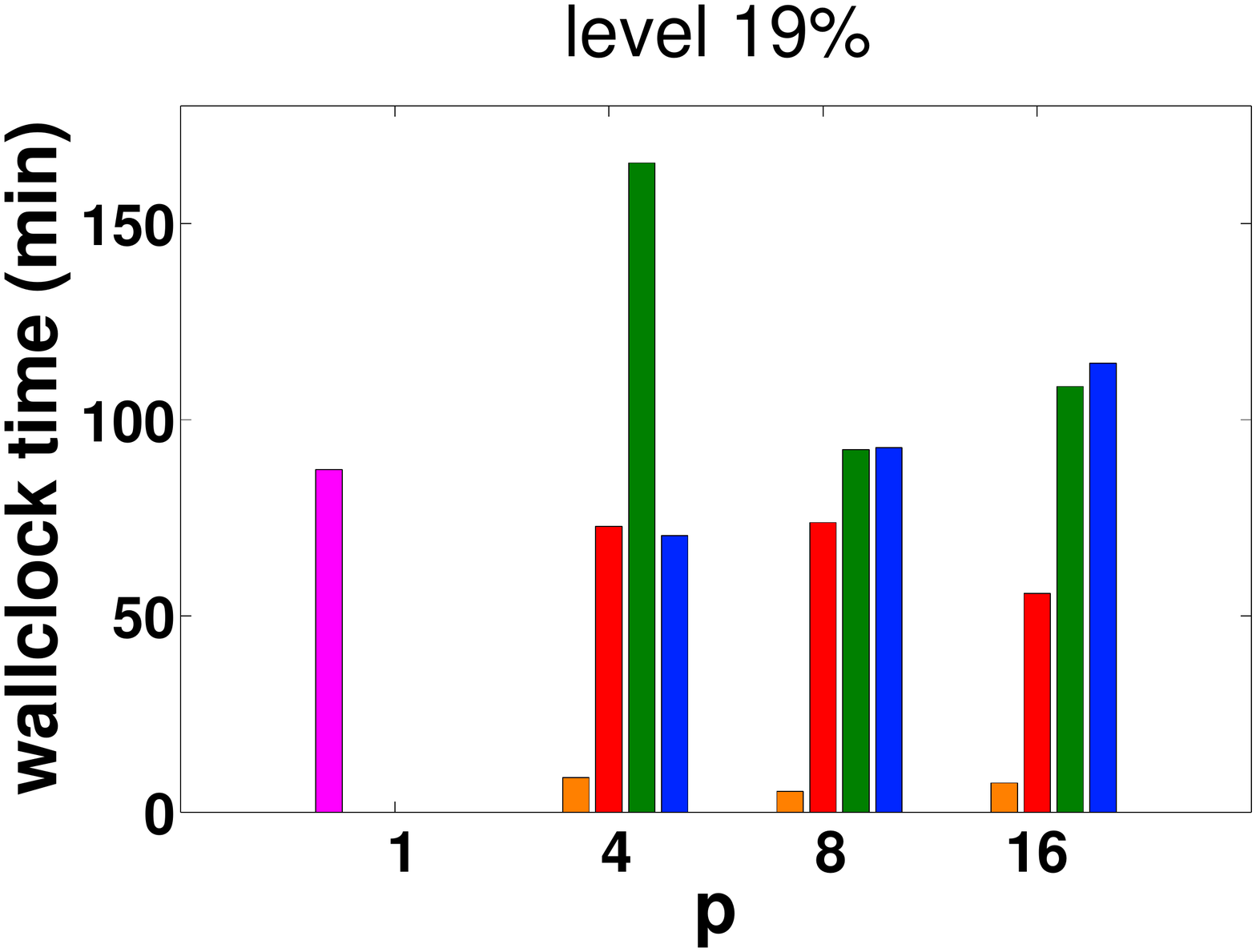}
\hspace{-0.28in}\includegraphics[trim=0cm 0cm 0cm 0cm,clip,angle=0,width = 1.56in]{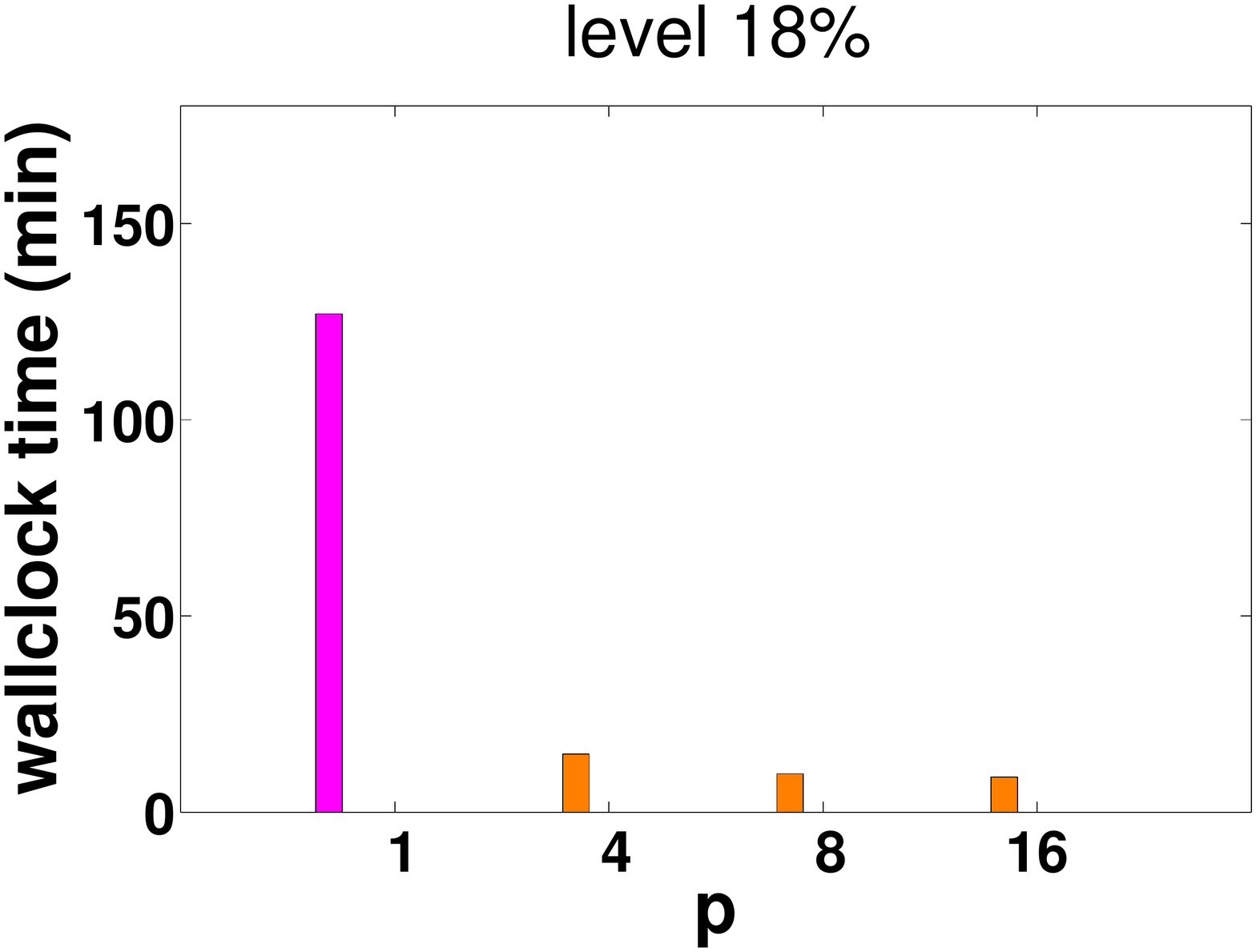}
\vspace{-0.3in}
\caption{The wall clock time needed to achieve the same level of the test error $thr$ as a function of the number of local workers $p$ on the \textit{CIFAR} dataset. From left to right:  $thr = \{21\%,20\%,19\%,18\%\}$. Missing bars denote that the method never achieved specified level of test error.}.
\label{fig:CIFARspeedup}
\vspace{-0.20in}
\end{figure}

\begin{figure}[htp!]
\vspace{-0.25in}
  \center
\includegraphics[trim=0cm 0cm 0cm 0cm,clip,width = 1.56in]{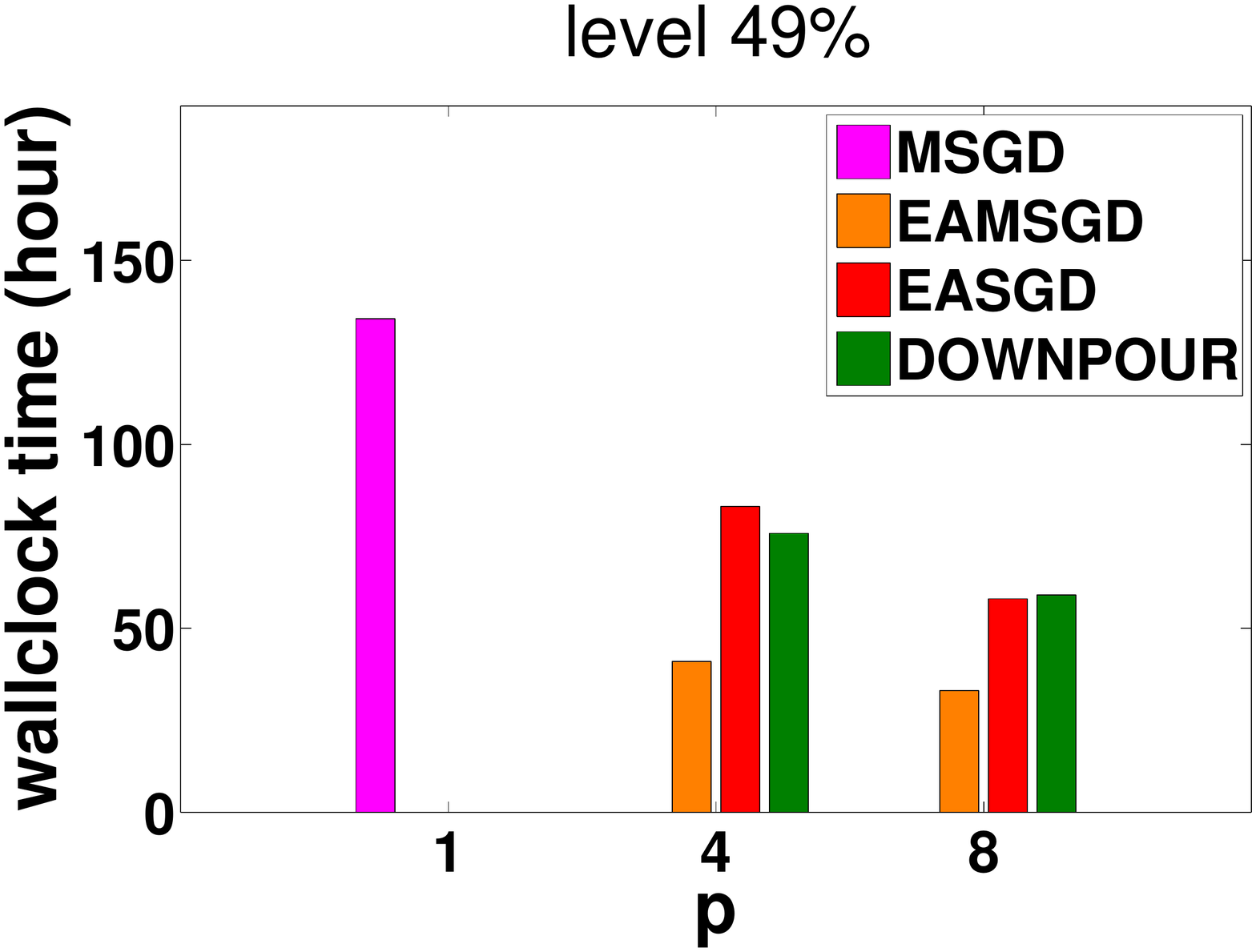}
\hspace{-0.28in}\includegraphics[trim=0cm 0cm 0cm 0cm,clip,angle=0,width = 1.56in]{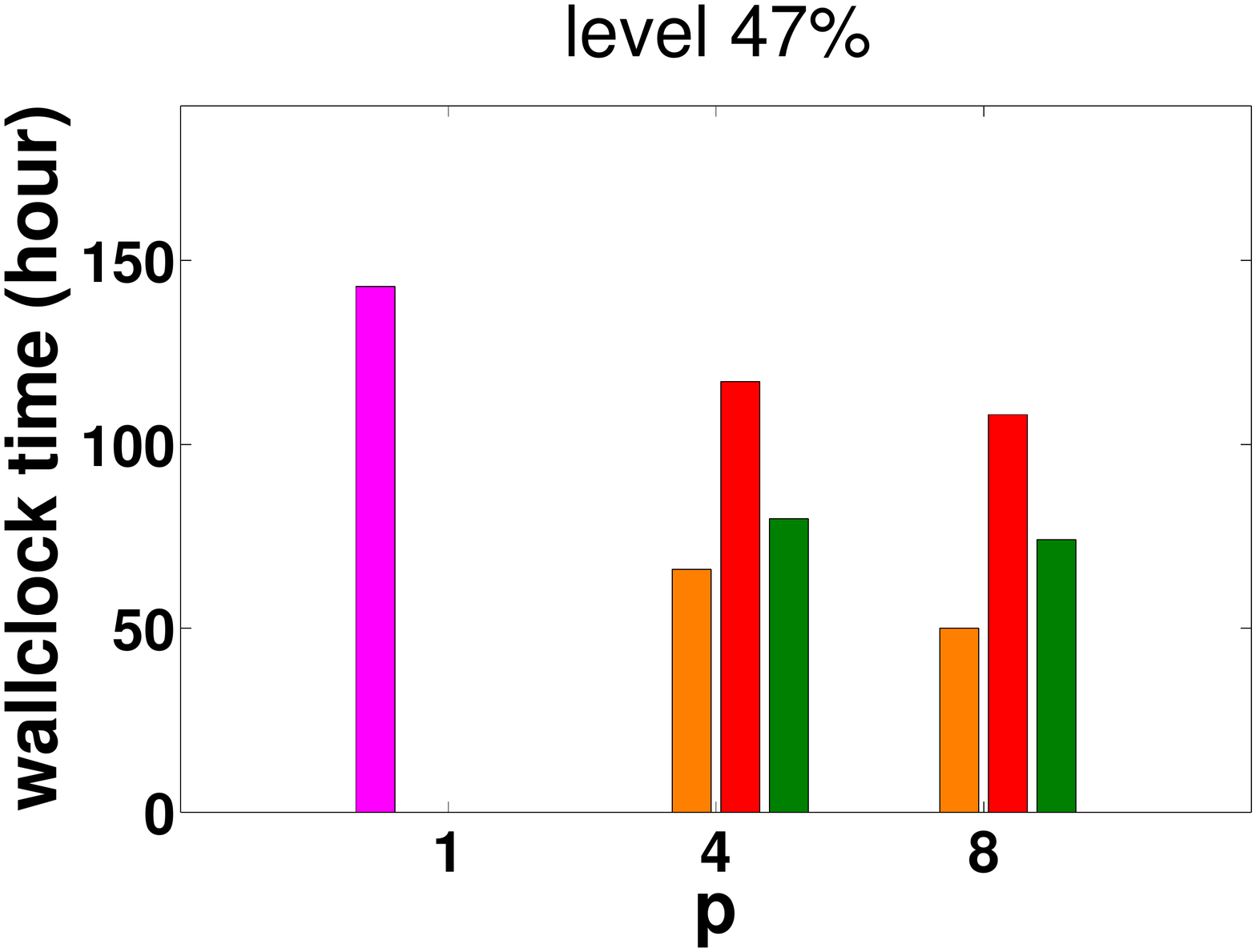} 
\hspace{-0.28in}\includegraphics[trim=0cm 0cm 0cm 0cm,clip,angle=0,width = 1.56in]{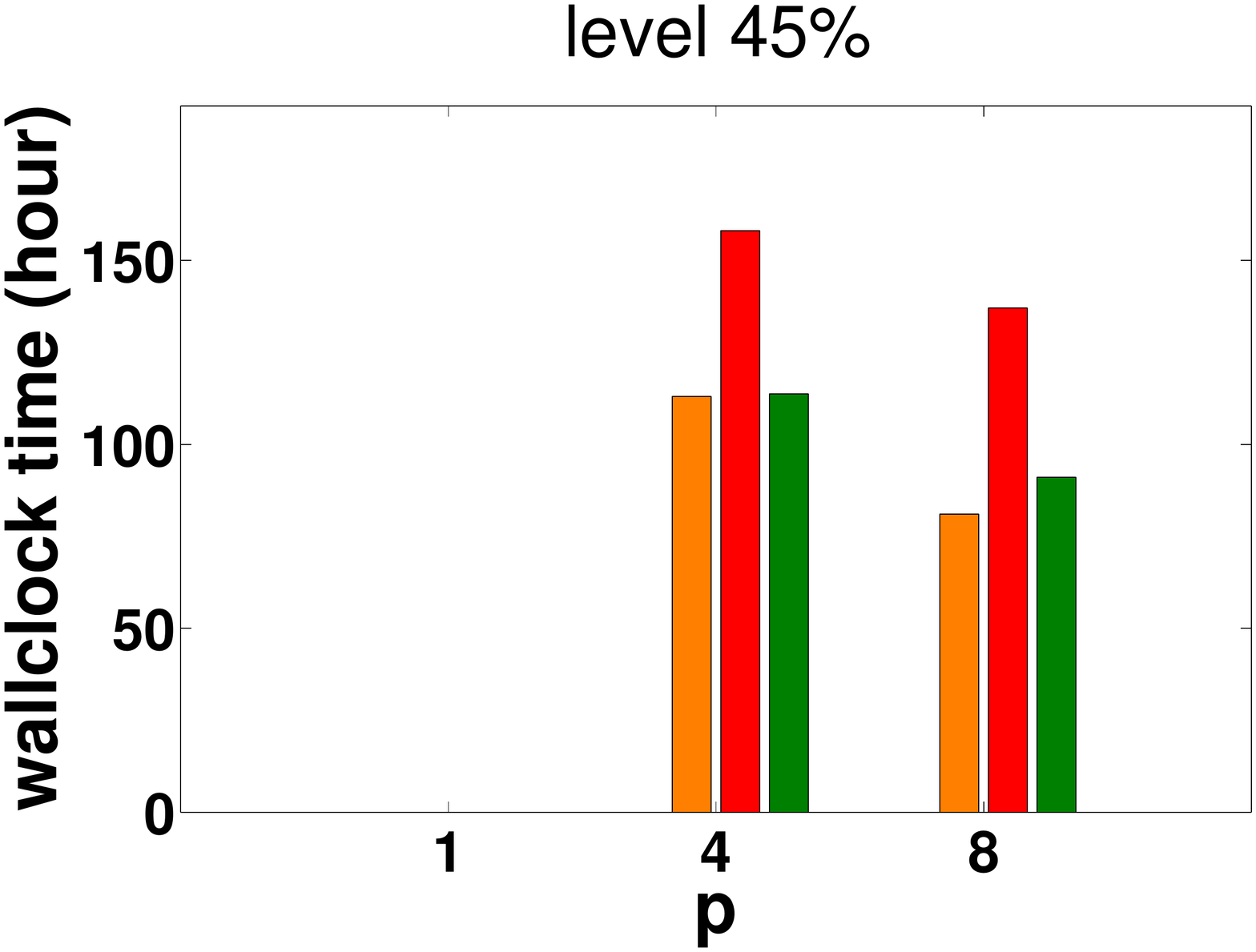}
\hspace{-0.28in}\includegraphics[trim=0cm 0cm 0cm 0cm,clip,angle=0,width = 1.56in]{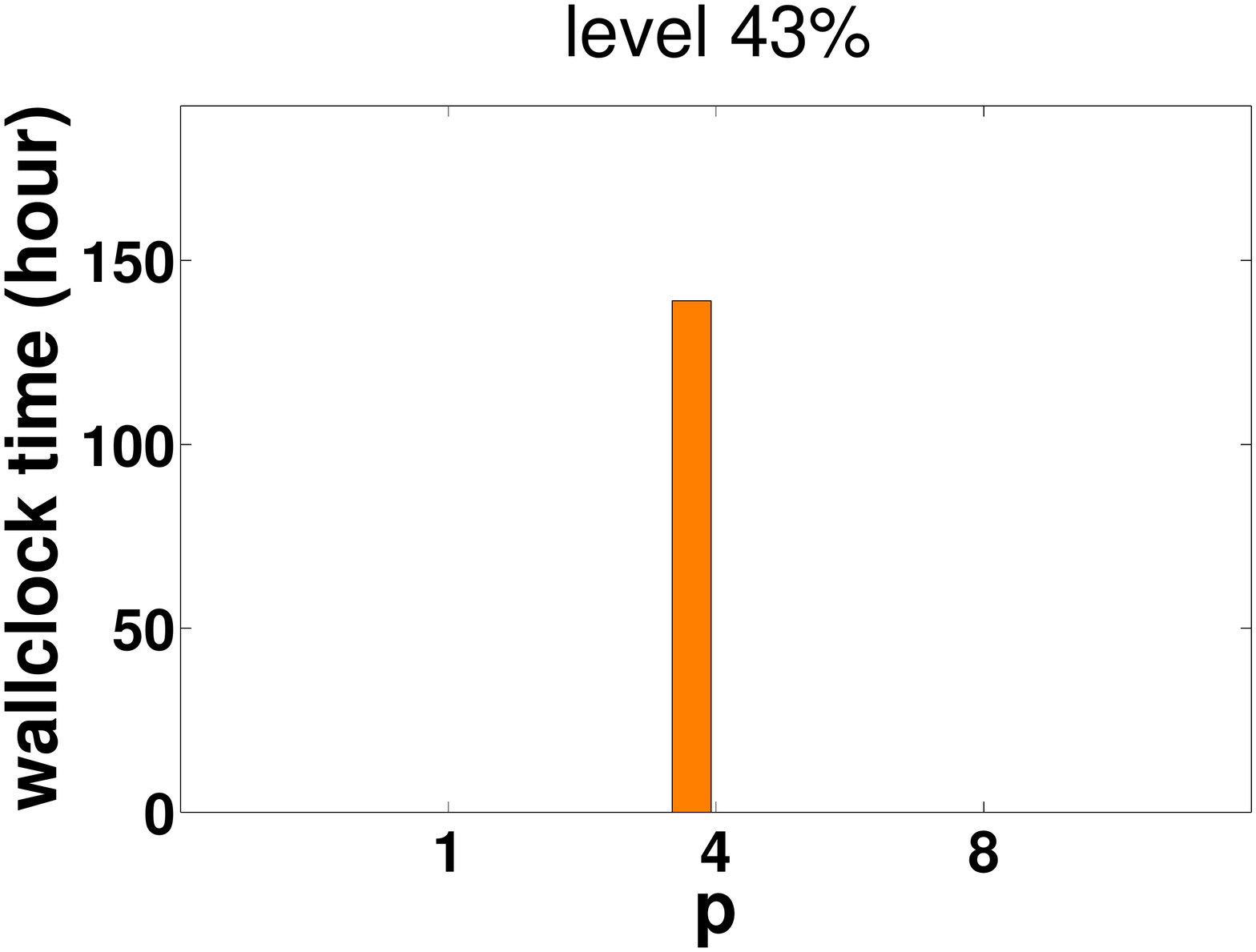}
\vspace{-0.3in}
\caption{The wall clock time needed to achieve the same level of the test error $thr$ as a function of the number of local workers $p$ on the \textit{ImageNet} dataset. From left to right:  $thr = \{49\%,47\%,45\%,43\%\}$. Missing bars denote that the method never achieved specified level of test error.}.
\label{fig:ImageNetspeedup}
\vspace{-0.1in}
\end{figure}

\end{document}